\documentclass[twoside,11pt]{article}

\usepackage{jmlr2e}
\usepackage{amsmath}
\usepackage{amssymb}

\usepackage{thmtools,thm-restate}
\usepackage{color}
\usepackage{rotating}

\newcommand{\Naturals}{\mathbb{N}}

\newtheorem{define}{Definition}
\newtheorem{case}{Case}
\newtheorem{subcase}{Case}
\numberwithin{subcase}{case}

\def\##1{\relax\ifmmode\mathchoice      
{\mbox{\boldmath$\displaystyle#1$}}
{\mbox{\boldmath$\textstyle#1$}}
{\mbox{\boldmath$\scriptstyle#1$}}
{\mbox{\boldmath$\scriptscriptstyle#1$}}\else
\hbox{\boldmath$\textstyle#1$}\fi}

\def\blackslug
     {\hbox{\hskip 1pt\vrule width 5pt height 8pt depth 1.5pt\hskip 1pt}}
\def\qedd{\hfill\quad\blackslug\lower 8.5pt\null\par}

\def\bar#1{\overline{#1}}

\newcommand{\EE}{\mathbb{E}}

\renewcommand{\Re}{\mathbb{R}}

\newcommand{\SF}{{\cal F}}
\newcommand{\SX}{{\cal X}}

\newcommand{\SY}{{\cal Y}}

\newcommand{\ST}{{\cal T}}

\newcommand{\SC}{{\cal C}}

\newcommand{\SO}{{\cal O}}

\def\M3N{M$^3$N}

\def\AUC{{\mathrm{AuRC}}}
\def\SELE{\Delta_{\rm sele}}
\def\SELEcvx{\psi_{\rm sele}}

\def\beq{\begin{equation}}
\def\eeq{\end{equation}}
\def\equ#1{(\ref{#1})}

\def\veps{\varepsilon}

\def\argmax{\mathop{\rm argmax}}                                                
 
\def\argmin{\mathop{\rm argmin}}

\def\sgn{\mathop{\rm sgn}}

\def\lz{\langle}
\def\pz{\rangle}

\def\sgn{\mathop{\rm sgn}}

\def\leftbb{\mathopen{\rlap{$[$}\hskip1.3pt[}}
\def\rightbb{\mathclose{\rlap{$]$}\hskip1.3pt]}}

\newcommand{\Prob}{\mathbb{P}}

\jmlrheading{A}{B}{C}{1/2021}{E}{F}{G}

\ShortHeadings{Optimal strategies for reject option classifiers}{Franc, Prusa and Voracek}
\firstpageno{1}

\begin{document}

\title{Optimal strategies for reject option classifiers}

\author{\name Vojtech Franc \email xfrancv@fel.cvut.cz 
       \AND
       \name Daniel Prusa \email prusa@fel.cvut.cz
       \AND
       \name Vaclav Voracek \email voracva1@fel.cvut.cz \\
       \addr Department of Cybernetics, Faculty of Electrical Engineering\\
       Czech Technical University in Prague, Czech Republic
       }

\editor{Editors}

\maketitle

\begin{abstract}
In classification with a reject option, the classifier is allowed in uncertain cases to abstain from prediction. 
The classical cost-based model of a reject option classifier requires the cost of rejection to be defined explicitly. An alternative bounded-improvement model, avoiding the notion of the reject cost, seeks for a classifier with a guaranteed selective risk and maximal cover. We coin a symmetric definition, the bounded-coverage model, which seeks for a classifier with minimal selective risk and guaranteed coverage. We prove that despite their different formulations the three rejection models lead to the same prediction strategy: a Bayes classifier endowed with a randomized Bayes selection function. We define a notion of a proper uncertainty score as a scalar summary of prediction uncertainty sufficient to construct the randomized Bayes selection function. We propose two algorithms to learn the proper uncertainty score from examples for an arbitrary black-box classifier. We prove that both algorithms provide Fisher consistent estimates of the proper uncertainty score and  we demonstrate their efficiency on different prediction problems including classification, ordinal regression and structured output classification. 
\end{abstract}


\section{Introduction}

In safety critical applications of classification models, prediction errors may lead to serious losses. In such cases estimating when the model makes an error can be as important as its average performance. These two objectives are taken into account in classification with a reject option when the classifier is allowed in uncertain cases to abstain from prediction. 

The cost-based model of a classification strategy with the reject option was proposed by Chow in his pioneering work~\cite{Chow-RejectOpt-TIT1970}. The goal is to minimize the expected loss equal to the cost of misclassification, when the classifier predicts, and to the reject cost, when the classifier abstains from prediction. An optimal strategy leads to the Bayes classifier abstaining from prediction when the conditional expected risk exceeds the reject cost. 
The known form of the optimal strategy allows to construct the classifier by plugging in an estimate of the class posterior probabilities to the formula for the conditional risk. 
Besides the plug-in rule, the reject option classifiers can be learned by empirical risk minimization based approaches like e.g.  modifications of Support Vector Machines~\citep{Grandvalet-SVMwithRejectOpt-NIPS2008}, Boosting~\citep{Cortes-BoostWitAbst-NIPS2016}, or Prototype-based classifiers~\citep{Villmann-RejectProto-AISC2016} to name a few.


The cost-based model requires the reject cost to be defined explicitly which is difficult in some applications e.g. when the misclassifiation costs have different physical units than the reject cost. An alternative bounded-improvement model coined in~\cite{Pietraszek-AbstainROC-ICML2005} avoids explicit definition of the reject cost. The rejection strategy is evaluated by two antagonistic quantities: i) a selective risk defined as the expected misclassification cost on accepted predictions and ii) a coverage which corresponds to the probability that the prediction is accepted. An optimal strategy for the bounded-improvement model is the one which maximizes the coverage under the condition that the selective risk does not exceed a target value. In contrast to the cost-based model, it has not been formally shown what is the optimal prediction strategies when the underlying model is known. 
A solution has been proposed only for special instances of the task. \cite{Pietraszek-AbstainROC-ICML2005} coined a method based on ROC analysis applicable when a score proportional to posterior probabilities is known and the task is to find only the optimal thresholds. \cite{ElYaniv-SelectClass-JMLR10} proposed an algorithm learning the optimal strategy in the noise-free setting, i.e. when a perfect strategy with zero selective risk exists. ~\cite{Geifman-SelectClass-NIPS2017} shows how to equip a trained classifier with a reject option provided an uncertainty measure is known and the task is to find only a rejection threshold optimal under the bounded-improvement model.

A large number of other works address the problem of uncertainty prediction, including recent papers related to deep learning like e.g.~\cite{Laksh-UncertDeep-NIPS2017,Jiang-TrustOrNot-NIPS2018,Corbiere-Failure-NeurIPS2019}. They seek for an uncertainty score~\footnote{Some works use term confidence score which is inverse to the uncertainty score utilized in this paper. } defined informally as a real valued summary of an input observation that is predictive of the classification error. These works do not formulate the problem to be solved explicitly as a rejection model. However, most of these works asses performance of their uncertainty scores using evaluation metrics for the rejection models, namely, using the Risk-Coverage (RC) curve and the Area under the RC curve (AuRC).

This article unifies and extends existing formulations of an optimal reject option classifier and proposes theoretically grounded algorithms to learn the classifiers from examples. The main contributions are as follows:

\begin{enumerate}
   \item We provide necessary and sufficient conditions for an optimal prediction strategy of the bounded-improvement model when the underlying distribution is known. We show that an optimal solution is the Bayes classifier endowed with a rejection strategy, which we call {\em randomized Bayes selection function}. The randomized Bayes selection function is constructed from the conditional expected risk and two parameters: a decision threshold and an acceptance probability. The strategy rejects prediction when the conditional risk is above the threshold, accepts prediction when it is below the threshold and randomizes with the acceptance probability otherwise. We provide an explicit relation between the decision threshold, the acceptance probability and the target risk. 
   \item We formulate a bounded-coverage model whose definition is symmetric to the bounded-improvement model. The optimal prediction strategy of the bounded-coverage model minimizes the selective risk under the condition that the coverage is not below a target value. We provide necessary and sufficient conditions for an optimal strategy and we show that the conditions are satisfied by the Bayes classifier endowed with the randomized Bayes selection function. We provide an explicit relation between the decision threshold, the acceptance probability and the target coverage.
   %
   %
   \item We define a notion of a {\em proper uncertainty score} as a function which preserves ordering of the inputs induced by the conditional expected risk. A proper uncertainty score is sufficient for construction of the randomized Bayes selection function. We propose two generic algorithms to learn the proper uncertainty score from examples for an arbitrary black-box classifier. The first is based on regression of the classifier loss. The second is based on minimization of a newly proposed loss function which we call SELEctive classifier learning (SELE) loss. We show that SELE loss is a tight approximation of the AuRC and at the same time amenable to optimization. We prove that both proposed algorithms provide Fisher consistent estimate of the proper uncertainty score.
   As a proof of concept we apply the proposed algorithms to learn proper uncertainty scores for different prediction problems including classification, ordinal regression and structured output classification. We demonstrate that the algorithm based on the SELE loss minimization learns uncertainty scores which consistently outperform common baselines and work on par with the state-of-the-art methods that are, unlike our algorithm, applicable only to particular prediction models.
\end{enumerate}

Besides the proposed algorithms applicable to learning uncertainty score for an arbitrary classification model, our contributions may have the following uses. Firstly, our analysis shows that despite their different objectives the cost-based, the bounded-improvement and the bounded-coverage rejection models are equivalent in the sense that they lead to the same prediction strategy. Secondly, the explicit characterization of optimal strategies provides a recipe how to construct plug-in rules which has been so far possible only for the cost-based model. That is, any method estimating the class posterior distribution can be turned into an algorithm learning the reject option classifier that solves the bounded-improvement and the bounded-coverage model, respectively. Thirdly, there is a tight connection between the proposed bounded-coverage model and the RC curve. The RC curve represents quality of all solutions of the bounded-coverage model that can be constructed from a pair of a classifier and an uncertainty score. The AuRC is then an expected quality of the reject option classifier constructed from the pair when the target coverage is selected uniformly at random. This connection sheds light on many published methods which do not explicitly define the target objective but use the RC curve and the AuRC as evaluation metrics. 

This article is an extension of our previous work published in~\cite{Franc-SELE-ICML2019}. The major extensions involve introduction of the bounded-coverage model and its analysis, analysis of the learning algorithms including the proof of Fisher consistency, and most of the experiments. 


The paper is organized as follows. Section~\ref{sec:rejectOption} introduces the three rejection models and provides characterization of their optimal solutions. Algorithms to learn a proper uncertainty score from examples are discussed in Section~\ref{sec:UncertaintyLearning}. Survey of related literature is given in Section~\ref{sec:relatedWorks}.
Experimental evaluation of the proposed learning algorithms is provided in Section~\ref{sec:experiments}. Section~\ref{sec:conclusions} concludes the paper. Proofs of all theorems are deferred to Appendix.

\section{Reject Option Models and Their Optimal Strategies} \label{sec:rejectOption}

Let $\SX$ be a set of input observations and $\SY$ a finite set of labels. Let us assume that inputs and labels are generated by a random process with p.d.f. $p(x,y)$ defined 
over $\SX\times\SY$. A goal in the non-reject setting is to find a {\em classifier} $h\colon\SX\rightarrow\SY$ with a small {\em expected risk}
\[
    R(h) = \int_{\SX} \sum_{y\in\SY}p(x,y)\ell(y,h(x)) \, dx\:,
\]
where $\ell\colon\SY\times\SY\rightarrow\Re_+$ is a {\em loss} penalizing the predictions. 

The expected risk can be reduced by abstaining from prediction in uncertain cases. To this end, we use a {\em selective classifier}~\footnote{The {\em classifier with a reject option} is usually represented by a single function $h'\colon\SX\rightarrow\SY\cup \{\mbox{reject}\}$. We use the decomposition $h'(x)=(h,c)(x)$, and the terminology {\em selective classifier} from~\cite{ElYaniv-SelectClass-JMLR10} because we analyze $h$ and $c$ separately. } $(h,c)$ composed of a classifier $h\colon\SX\rightarrow\SY$ and a {\em selection function} $c\colon \SX\rightarrow [0,1]$. When applying the selective classifier to input $x\in\SX$ it outputs
\[
    (h,c)(x) = \left \{ \begin{array}{rcl}
        h(x) & \mbox{with probability} & c(x)\,, \\
        \mbox{reject} & \mbox{with probability} & 1- c(x)\:.
    \end{array}
     \right .
\]
In the sequel we introduce three models of an optimal selective classifier: the cost-based, the bounded-improvement and the bounded-coverage model. We characterize optimal strategies of the three models provided the underlying distribution $p(x,y)$ is known. 


\subsection{Cost-based model} 

Besides the label loss $\ell\colon\SY\times\SY\rightarrow\Re_+$, let us define a reject loss $\veps \in\Re_+$ incurred when a classifiers rejects to predict. The selective classifier $(h,c)$ is evaluated in terms of the expected risk
\[
  R_B(h,c) \! = \!\! \int_{\SX} \! \sum_{y\in\SY} p(x,y)\,\big (\ell(y,h(x)) c(x)
   + (1-c(x))\veps\big ) dx\:.
\]

\begin{restatable}{problem}{costBasedModel}{\bf (Cost-based model)}\label{prob:CostBased} The optimal selective classifier $(h_B,c_B)$ is a solution to the minimization problem
\begin{equation}
   \min_{h,c} R_B(h,c)  \,,
\end{equation}
where we assume that both minimizers exist. 

\end{restatable}
\noindent
The well-known optimal strategy $(h_B,c_B)$ solving Problem~\ref{prob:CostBased} reads
\begin{eqnarray}
  \label{equ:bayesPredictor}
   h_B(x) \in \argmin_{\hat{y}\in\SY} \sum_{y\in\SY} p(y\mid x) \,
   \ell(y,\hat{y}) \:,  \\
  \label{equ:BayesOptimalRejection}
   c_B(x) = \left \{ \begin{array}{rcl}
       1 & \mbox{if} & r^*(x) < \veps \:, \\
       \tau& \mbox{if} & r^*(x) = \veps  \:, \\
       0 & \mbox{if} & r^*(x) > \veps \:,\       
     \end{array} \right .
\end{eqnarray}
where 
\[
   r^*(x) = \min_{\hat{y}\in\SY} \sum_{y\in\SY} p(y\mid x) \, \ell(y,\hat{y})
\] 
is the minimal conditional expected risk associated to input $x$, and 
$\tau$ is any real number from the interval $[0,1]$. In the boundary cases when $r^*(x)=\veps$ one can arbitrarily reject or return the best label $h_B(x)$ without affecting the value of the risk $R_B(h,c)$. In turn there always exist a deterministic optimal strategy solving the cost-based model. In the sequel we denote the classifier $h_B$ alone as the {\em Bayes classifier}.

\subsection{Bounded-improvement model} 
One can characterize the selective classifier by two antagonistic quantities: i) the {\em coverage}
\[
    \phi(c) = \int_{\SX} p(x)\, c(x)\, dx
\]
corresponding to the probability that the prediction is accepted~\footnote{For a function $f\colon\SX\rightarrow\Re$ and $a\in\Re \cup \{\infty\}$, we define
$\SX_{f(x)\le a}\!=\!\{x\in \SX \! \mid \! f(x) \le a\}, \, \SX_{f(x)< a}\!= \{x\in \SX \! \mid \!f(x) < a\}, 
\SX_{f(x)=a}\!\!=\{x\in \SX \! \mid \!f(x) = a\}, \, \SX_{f(x)> a} = \{x\in \SX \! \mid \!f(x) > a\}, 
\SX_{f(x)\ge a}\!\!=\{x\in \SX \! \mid \!f(x) \ge a\}. 
$} 
and ii) the {\em selective risk}
\[
   R_S(h,c) = \frac{\int\limits_{\SX}\sum\limits_{y\in\SY}p(x,y)\,\ell(y,h(x))\,c(x)\,dx}{\phi(c)} \:,
\]
defined for non-zero $\phi(c)$ as the expected classification loss on the accepted predictions.
\begin{restatable}{problem}{boundedImprovement}\label{task:boundedImprovement}{\bf (Bounded-improvement model)}  Given a {\em target risk} $\lambda>0$,  the optimal selective classifier $(h_I,c_I)$ is a solution to the problem
\begin{equation}
\label{equ:selectClassifTask}
   \max_{h,c} \phi(c) \,\quad \mbox{s.t.}\quad
    R_S(h,c) \leq \lambda  \,,
\end{equation}
where we assume that both maximizers exist.
\end{restatable}

\begin{restatable}{theorem}{optimalPredictor}
\label{thm:optimalPredictor}
Let $(h,c)$ be an optimal solution to~(\ref{equ:selectClassifTask}). Then, $(h_B,c)$, where $h_B$ is the Bayes classifier~\equ{equ:bayesPredictor}, is also optimal to~(\ref{equ:selectClassifTask}).
\end{restatable}

According to Theorem~\ref{thm:optimalPredictor} the Bayes classifier $h_B$ is also optimal for the task~\equ{equ:selectClassifTask} defining the bounded-improvement model which is not surprising. Note however that the Bayes classifier is not a unique solution to~\equ{equ:selectClassifTask} because the predictions on the reject region $\SX_{c(x)=0}$ do not count to the selective risk and hence they can be arbitrary. 

Theorem~\ref{thm:optimalPredictor} allows to solve the bounded-improvement task~\equ{equ:selectClassifTask} in two consecutive steps: First, set $h_I$ to be the Bayes classifier $h_B$. Second, when $h_I$ is fixed, the optimal selection function $c_I$ is obtained by solving the task~\equ{equ:selectClassifTask} only with respect to $c$ which boils down to
\begin{restatable}{problem}{boundedImprovSelection}\label{task:boundedImprovSelection}{\bf (Bounded-improvement model for known $h(x)$)} \label{task:BIMknownH} Given a classifier $h(x)$, the optimal selection function $c^*(x)$ is a solution to
\begin{equation}\label{equ:OptSelectContinuousTask}
   \max_{c\in [0,1]^\SX} \phi(c) \,\quad \mbox{s.t.}\quad
    R_S(h,c) \leq \lambda  \,.
\end{equation}
\end{restatable}
\noindent
Note that Problem~\ref{task:BIMknownH} is meaningful even if $h$ is not the Bayes classifier $h_B$. In practice we can seek for an optimal selection function $c^*$ for any fixed $h$ which is usually our best approximation of $h_B$ learned from data. We will show that the key concept to characterize an optimal selection function of Problem~\ref{task:boundedImprovSelection} is the conditional expected risk of $h$ defined as
\begin{equation}
   \label{equ:condRisk}    
    r(x) = \sum_{y\in\SY}p(y\mid x) \,\ell(y,h(x))\:.
\end{equation}

\begin{restatable}{theorem}{boundedImprovSolution}\label{thm:selectContinuousTask}
A selection function $c^*:\SX\to [0,1]$ is an optimal solution to Problem~\ref{task:BIMknownH}
if and only if it holds
\begin{align}
    \int_{\SX_{\bar{r}(x)<b}}p(x)c^*(x)dx &=\int_{\SX_{\bar{r}(x)<b}}p(x)dx,\label{equ:cond-1} \\
    \int_{\SX_{\bar{r}(x)=b}}p(x)c^*(x)dx &=\left \{  
    \begin{array}{ccl}
         -\frac{\rho(\SX_{\bar{r}(x)<b})}{b} & \mbox{if} & b > 0 \,,\\
         \int_{\SX_{\bar{r}(x)=0}}p(x)dx & \mbox{if} & b=0 \,,\\
      \end{array}
   \right.\label{equ:cond-2} \\
    \int_{\SX_{\bar{r}(x)>b}}p(x)c^*(x)dx &=0\,,\label{equ:cond-3}
\end{align}
where $\bar{r}(x) = r(x)-\lambda$ measures how much the conditional risk $r(x)$ of the classifier $h(x)$ exceeds the target $\lambda$,
\begin{equation}\label{equ:rho_def}
     \rho(\SX')=\int_{\SX'}p(x)\bar{r}(x)\,dx
\end{equation}
is the expectation of $\bar{r}(x)$ restricted to inputs in $\SX'$, and
\begin{equation}
    \label{equ:formulaForB}
    b=\sup\, \{a \mid \rho(\SX_{\bar{r}(x)\le a})\le 0\} \ge 0\:.
\end{equation}
\end{restatable}

Theorem~\ref{thm:selectContinuousTask} defines behaviour of an optimal selection function $c^*(x)$ on a partition of the input space $\SX$ into three regions $\SX_{\bar{r}(x)<b}$, $\SX_{\bar{r}(x)=b}$ and $\SX_{\bar{r}(x)>b}$. In each region the expected value of $c^*(x)$ is constrained to a particular constant the value of which depends on parameters of the problem. A particular selection function satisfying the optimality condition is given by the following theorem.

\begin{restatable}{theorem}{ThmOptSelectingFce}\label{thm:optSelectingFun}
Let $r\colon\SX\rightarrow\Re$ be the conditional risk~\equ{equ:condRisk} of a classifier $h\colon\SX\rightarrow\SY$, $\gamma=b+\lambda$ the rejection threshold given by the target risk $\lambda$ and a constant $b$ computed by~\equ{equ:formulaForB}. Then the selection function 
\begin{equation}
  \label{equ:optSelFun2}
   c^*(x) = \left \{ 
      \begin{array}{rcl}
         1 & \mbox{if} & r(x) < \gamma \,,\\
         \tau & \mbox{if} & r(x) = \gamma\,, \\
         0 & \mbox{if} & r(x) < \gamma \,,\\
      \end{array}
   \right .
   \end{equation}
where $\tau$ is the acceptance probability given by
   \begin{equation}
    \label{equ:rejectProbab}
   \tau=\left \{\begin{array}{ccl}
       1 & \mbox{if} & \rho(\SX_{r(x)=\gamma}) = 0 \,,\\
       -\frac{\rho(\SX_{r(x)<\gamma)})}{\rho(\SX_{r(x)=\gamma)})}\, & \mbox{if} & \rho(\SX_{r(x)=\gamma}) > 0 \,,\\
     \end{array}
   \right .
\end{equation}
satisfies the optimality condition of Theorem~\ref{thm:selectContinuousTask}, and hence it is a solution to Problem~\ref{task:boundedImprovSelection}.
\end{restatable}

The selection function~\equ{equ:optSelFun2} is defined by the conditional risk $r(x)$, the decision threshold $\gamma$ and the acceptance probability $\tau$. 
The prediction is always accepted when $r(x)<\gamma$ and always rejected when $r(x)> \gamma$. In the boundary cases, when $r(x) = \gamma$, the strategy randomizes and the prediction is accepted with probability $\tau$. The decision threshold is given by $\gamma=b+\lambda$ where $\lambda$ is the target risk in the definition of Problem~\ref{task:boundedImprovSelection} and $b$ is given by~\equ{equ:formulaForB}. Solving~\equ{equ:formulaForB} is hard and it requires knowledge of $p(x,y)$. When the probability mass of the set of boundary cases $\SX_{r(x)=\gamma}$ is zero, which usually happens in case of continuous $p(x)$, the acceptance probability is $\tau=1$ and the boundary cases are always accepted, i.e. no randomization is needed.

\subsection{Bounded-coverage model}

In this section we introduce {\em bounded-coverage model} the definition of which is symmetric to the definition of the bounded-improvement model. Although the problem seems equally useful in practice we are unaware of its formal definition in literature. 

\begin{restatable}{problem}{boundedCoverage}\label{task:boundedCoverage}{\bf (Bounded-coverage model)} Given a {\em target coverage} $\omega>0$, the optimal selective classifier $(h_C,c_C)$ is a solution to the problem
\begin{equation}
\label{equ:boundedCoverageModel}
   \min_{h,c} R_S(h,c) \,\quad \mbox{s.t.}\quad
    \phi(c) \geq \omega  \,,
\end{equation}
where we assume that both minimizers exist.
\end{restatable}

\begin{restatable}{theorem}{optimalClassifierForBndCov}
\label{thm:optClsForBoundedCovModel}
Let $(h,c)$ be an optimal solution to~\equ{equ:boundedCoverageModel}. Then, $(h_B,c)$, where $h_B$ is the optimal Bayes classifier~\equ{equ:bayesPredictor}, is also optimal to~\equ{equ:boundedCoverageModel}.
\end{restatable}

Theorem~\ref{thm:optClsForBoundedCovModel} ensures that the Bayes classifier $h_B$ is an optimal solution to~\equ{equ:boundedCoverageModel} defining the bounded-coverage model. Note that the solution is not unique as the predictions on $\SX_{c(x)=0}$ do not count to the selective risk hence they can be arbitrary. After fixing the classifier $h=h_B$ the search for an optimal selection function leads to:

\begin{restatable}{problem}{boundedCoverageKnownCls}\label{task:boundedCoverageKnownCls}{\bf (Bounded-coverage model for known $h(x)$)} Given a classifier $h(x)$ and a target coverage $0< \omega \leq 1$, the optimal selection function $c^*(x)$ is a solution to the problem
\begin{equation}
\label{equ:boundedCoverageModelKnownCls}
   \min_{c\in [0,1]^\SX} R_S(h,c) \,\quad \mbox{s.t.}\quad
    \phi(c) \geq \omega  \,,
\end{equation}
where we assume that the minimizer exists.
\end{restatable}






\begin{restatable}{theorem}{boundedCoverageSolution}\label{thm:selectContinuousTask2}
A selection function $c^*:\SX\to [0,1]$ is an optimal solution to Problem~\ref{task:boundedCoverageKnownCls}
if and only if it holds
\begin{align}
    \int_{\SX_{{r}(x)<\beta}}p(x)c^*(x)dx &=\int_{\SX_{{r}(x)<\beta}}p(x)dx,\label{equ:task2-cond-1} \\
    \int_{\SX_{{r}(x)=\beta}}p(x)c^*(x)dx &= \omega - \int_{\SX_{{r}(x)<\beta}}p(x)dx, \label{equ:task2-cond-2} \\
    \int_{\SX_{{r}(x)>\beta}}p(x)c^*(x)dx &=0\,,\label{equ:task2-cond-3}
\end{align}
where
\begin{equation}
    \label{equ:task2-formulaForB}
    \beta=\inf\, \left\{a \mid \int_{\SX_{{r}(x)<a}}p(x)dx \ge \omega \right\}\:.
\end{equation}
\end{restatable}

Theorem~\ref{thm:selectContinuousTask2} defines necessary and sufficient conditions on an optimal solution to Problem~\ref{task:boundedCoverageKnownCls}. A particular selection function satisfying the optimality conditions is given by the following theorem.

\begin{restatable}{theorem}{ThmOptSelectingFceTwo}\label{thm:optSelectingFun2}
Let $r\colon\SX\rightarrow\Re$ be the conditional risk~\equ{equ:condRisk} of a classifier $h\colon\SX\rightarrow\SY$, $1\ge \omega > 0$ be a target coverage and $\beta$ be the constant computed by~\equ{equ:task2-formulaForB}. Then the selection function 
\begin{equation}
  \label{equ:optSelFun}
   c^*(x) = \left \{ 
      \begin{array}{rcl}
         1 & \mbox{if} & r(x) < \beta \,,\\
         \kappa & \mbox{if} & r(x) = \beta\,, \\
         0 & \mbox{if} & r(x) > \beta \,,\\
      \end{array}
   \right .
   \end{equation}
where $\kappa$ is the acceptance probability given by
   \begin{equation}
    \label{equ:rejectProbab2}
   \kappa=\left \{\begin{array}{cl}
       0 & \mbox{if } \int_{\SX_{{r}(x)=\beta}}p(x)dx = 0 \,,\\
       \frac{\omega - \int_{\SX_{{r}(x)<\beta}}p(x)dx} {\int_{\SX_{{r}(x)=\beta}}p(x)dx}\, & \mbox{otherwise} \,,\\
     \end{array}
   \right .
\end{equation}
satisfies the optimality condition of Theorem~\ref{thm:optSelectingFun2}, and hence it is a solution of Problem~\ref{task:boundedCoverageKnownCls}.
\end{restatable}

The selection function~\equ{equ:optSelFun} is determined by the conditional risk $r(x)$, the decision threshold $\beta$ and the acceptance probability $\kappa$. Both computations of the decision threshold $\beta$, defined by~\equ{equ:task2-formulaForB}, and the acceptance probability $\kappa$, defined by~\equ{equ:rejectProbab2}, involve integration of $p(x)$. When the probability mass of the set of boundary cases $\SX_{r(x)=\beta}$ is zero, the acceptance probability is $\kappa=0$ and the boundary cases are always rejected without any randomization.


\subsection{Summary}

We have shown that the three rejection models, namely, the cost-based model (c.f. Problem~\ref{prob:CostBased}), the bounded-improvement model (c.f. Problem~\ref{task:boundedImprovement}) and the bounded-coverage model (c.f. Problem~\ref{task:boundedCoverage}), share the same class of optimal prediction strategies. An optimal selective classifier $(h,c)$ can be always constructed from the Bayes classifier $h=h_B$ given by~\equ{equ:BayesOptimalRejection} and a selection function 
\begin{equation}
   \label{equ:optSelectionFun}
   c_R(x) = \left \{
   \begin{array}{ccc}
        1 & \mbox{if} & r(x) < \alpha\:, \\
        \nu & \mbox{if} & r(x) = \alpha\:, \\
        0 & \mbox{if} & r(x) > \alpha\:, \\
   \end{array}
   \right .
\end{equation}
where $r(x)$ is the conditional expected risk of $h(x)$ given by~\equ{equ:condRisk}, $\alpha\in\Re$ is a decision threshold and $\nu\in[0,1]$ is an acceptance probability. We denote $c_R$ defined by~\equ{equ:optSelectionFun} as the {\em randomized Bayes selection function}. Note that the randomized Bayes selection function $c_R$ is also an optimal solution of the rejection models defined for an arbitrary (i.e. non-Bayes) classifier $h$, that is, an optimal solution of Problem~\ref{task:boundedImprovSelection} and Problem~\ref{task:boundedCoverageKnownCls}.

The constants $(\nu,\alpha)$ are defined for each rejection model differently and their value depends on parameters of the model (i.e. reject cost $\veps$, target risk $\lambda$ or target coverage $\omega$), the conditional risk $r(x)$ and the underlying distribution $p(x,y)$. For example, in case of the cost-based model the acceptance threshold $\alpha$ equals to the reject cost $\veps$ and the acceptance probability $\tau$ can be arbitrary. In case of the bounded-improvement and the bounded-coverage model the constants $(\nu,\alpha)$ are defined implicitly via optimization problems and integral equations. In practice $(\nu,\alpha)$ can be tuned on data. For example, \cite{Geifman-SelectClass-NIPS2017} show how to find $\alpha$ from a finite sample such that it is optimal for the bounded-improvement model in PAC sense. 

The key component of randomized Bayes selection function $c_R$ is ranking of the inputs $\SX$ according to $r(x)$. 
We introduce notion of a {\em proper uncertainty score} which is less informative than the conditional risk $r(x)$, yet it is sufficient to construct $c_R$.

\begin{define} \label{def:properUncertScore}Let $h\colon\SX\rightarrow\SY$ be a classifier and $r(x)=\sum_{y\in\SY} p(y\mid x)\, \ell( y, h(x))$ its conditional expected risk. We say that function $s\colon\SX\rightarrow\Re$ is a {\em proper uncertainty score} of $h$ iff $\forall(x,x')\in\SX\times\SX\colon r(x) < r(x') \Rightarrow s(x) < s(x')$. 
\end{define}
\noindent
By definition the proper uncertainty score $s(x)$ preserves ordering of the inputs $\SX$ induced by the conditional risk $r(x)$. Therefore replacing $r(x)$ by $s(x)$ in function~\equ{equ:optSelectionFun}, and changing the decision threshold $\alpha$ appropriately, leads to the same optimal selection function.

\section{Learning the uncertainty function}\label{sec:UncertaintyLearning}


Assume we want to construct a selective classifier $(h,c)$ solving any of the three rejection models described in Section~\ref{sec:rejectOption}. We have shown that regardless the rejection model, an optimal $h$ is the Bayes classifier $h_B$ given by~\equ{equ:bayesPredictor} and an optimal $c$ is the randomized Bayes selection function $c_R$ given by~\equ{equ:optSelectionFun}. In this section we consider the scenario when $h\colon\SX\rightarrow\SY$ has been already trained and we want to endow it with $c_R$. The key component of $c_R$ is a proper uncertainty score $s\colon\SX\rightarrow\Re$ satisfying Definition~\ref{def:properUncertScore}. In this section we address problem of learning a proper uncertainty score from examples $\ST_n=\{ (x_i,y_i)\in\SX\times\SY\mid i=1,\ldots,n\}$ assumed to be generated from $n$ i.i.d. random variables with distribution $p(x,y)$. Before describing the algorithms, in Section~\ref{sec:AuRC} we introduce the notion of Risk-Coverage (RC) curve and Area under Risk-Coverage curve (AuRC). In line with the literature we use the AuRC as a metric to evaluate performance of the learned uncertainty scores. We also show a connection between the RC curve, the AuRC and the bounded-coverage model. In Section~\ref{sec:pluginRisk} we describe a {\em plug-in condition risk rule} and point out that a frequently used Maximum Class Probability rule (MCP) is its special instance. In Section~\ref{sec:REG} we outline a learning approach based on a {\em loss regression}. In Section~\ref{sec:SELE} we introduce a proxy of the AuRC which we call a loss {\em for SELEctive classifier learning} (SELE). We prove that both proposed methods learn the Fisher consistent estimator of the proper uncertainty score. 


\subsection{Area under Risk Coverage curve}
\label{sec:AuRC}


Majority of existing methods that learn selective classifiers output a classifier $h\colon\SX\rightarrow\SY$ and a deterministic selection function $c\colon\SX\rightarrow[0,1]$ defined as~\footnote{Note that the deterministic~\equ{equ:determinSelect} and the randomized Bayes selection function $c_R$ coincide if the acceptance probability is $\nu=1$, which is a usual case when $p(x)$ is continuous.}
\begin{equation}
   \label{equ:determinSelect}
   c(x)=\leftbb s(x) \leq \theta \rightbb \:,
\end{equation}
where $s\colon\SX\rightarrow\Re$ is an uncertainty score and $\theta\in\Re$ a decision threshold. Performance of the pair $(h,s)$ is evaluated by the RC curve obtained after computing the empirical selective risk and the coverage for all settings of the threshold $\theta$. 
Namely, the computation is as follows. Let us order the examples $\ST_n=\{ (x_i,y_i)\in\SX\times\SY\mid i=1,\ldots,n\}$ according to $s(x)$ so that $s(x_{\pi(1)})\leq s(x_{\pi(2)})\leq\cdots\leq s(x_{\pi(n)})$, where $\pi\colon\{1,\ldots,n\}\rightarrow\{1,\ldots,n\}$ is a permutation defining the order~\footnote{To break ties we use the index of the input in case the scores are the same. }. Let ${L}(i,s)=\sum_{j=1}^i\ell(y_{\pi_j},h(x_{\pi_j}))$ be a sum of losses incurred by the classifier $h(x)$ on the examples with uncertainty not higher than the $i$-th highest uncertainty on the examples $\ST_n$. The Risk-Coverage curve ${\SC}=\{(\frac{1}{i}{L}(i,s),\frac{i}{n})\mid i=1,\ldots,n\}$ is 
a set of 2-dimensional points, where the pair ($\frac{1}{i}{L}(i,s)$,$\frac{i}{n}$) corresponds to the empirical estimate of the selective risk $R_S(h,c)$ and the coverage $\phi(c)$ of a selective classifier $(h,c)$ with the deterministic selective function \equ{equ:determinSelect} and decision threshold $\theta=s(x_{\pi_i})$. The area under the RC curve ${\SC}$ is then
\begin{equation}
  \AUC(s,\ST_n) = \frac{1}{n}\sum_{i=1}^n \frac{1}{i} {L}(i,s) = \frac{1}{n} \sum_{i=1}^n \frac{1}{i} \sum_{j=1}^{i} \ell( y_{\pi_j},h(x_{\pi_j})) \,.
\end{equation}
%
%
The value of $\AUC(s,\ST_n)$ can be interpreted as an arithmetic mean of the empirical selective risks corresponding to coverage equidistantly spread over the interval $[0,1]$ with step $\frac{1}{n}$. 

There is a tight connection between RC curve, AuRC and the bounded-coverage model. The RC curve $\SC$ represents quality of all admissible solutions of the bounded-coverage model that can be constructed from the pair $(h,s)$ when using the sample $\ST_n$ for evaluation. The value of $\AUC(s,\ST_n)$ is an estimate of the expected quality of the selective classifier constructed from the pair $(h,s)$ when the target coverage is selected uniformly at random. 

\subsection{Plug-in conditional risk rule} 
\label{sec:pluginRisk}

Prediction models, like e.g. Logistic Regression or Neural Networks learned by cross-entropy loss, use the training set $\ST_n$ to learn an estimate $\hat{p}(y\mid x)$ of the class posterior distribution $p(y\mid x)$. The estimate is then used to construct a plug-in Bayes classifier $\hat{h}(x)\in\argmin_{\hat{y}\in\SY}\sum_{y\in\SY} \hat{p}(y\mid x)\ell(y,\hat{y})$.
Similarly, using $\hat{p}(y\mid x)$ instead of $p(y\mid x)$ in~\equ{equ:condRisk} yields the plug-in rule for the conditional risk of classifier $h$ defined as
\[
   \hat{r}(x) = \sum_{y\in\SY} \hat{p}(y\mid x) \, \ell(y,h(x))\:.
\]
Provided $p(y\mid x)=\hat{p}(y\mid x)$,$\forall x\in\SX, \forall y\in\SY$, the plug-in conditional risk $\hat{r}(x)$ is by definition a proper uncertainty score and it can be used to construct the randomized Bayes selection function $c_R$ which is an optimal rejection strategy for all the three rejection models.

\begin{example}[Maximum Class Probability rule]\label{example:MCP} In case of 0/1-loss $\ell(y,y') = \leftbb y\neq y'\rightbb$ the plug-in Bayes classifier decides based on the maximum posterior probability $\hat{h}(x)\in\argmax_{y\in\SY} \hat{p}(y\mid x)$ and the plug-in conditional risk rule is
  \[
 \hat{r}(x)=\sum_{y\in\SY} \hat{p}(y\mid x) \,  \ell(y,\hat{h}(x)) = 1-\max_{y\in\SY} \hat{p}(y\mid x)\:.
 \]
%
\end{example}
%

\subsection{Loss regression}
\label{sec:REG}

A straightforward approach to learn the uncertainty score is to pose it as a regression problem. The regression function gets an input $x\in\SX$ and outputs an estimate of the classification loss $\ell(y,h(x))$. Formally, 
given a hypothesis space $\SF\subset \{s\colon\SX\rightarrow\Re\}$, classifier $h(x)$ and training set $\ST_n$, the {\em loss regression score} $s\colon\SX\rightarrow\Re$ is a solution to $\min_{s\in\SF} F_{\rm reg}(s)$ where
\[
   F_{\rm reg}(s) =\frac{1}{n}\sum_{i=1}^n \Big (\ell(y_i,h(x_i))-s(x_i) \Big )^2 \:.
\]
It is easy to show that the loss regression score is Fisher consistent estimate of the proper uncertainty score. This amounts to defining the expectation $F_{\rm reg}(x)$ with respect to i.i.d. generated training set $\ST_n$, i.e., 
\begin{align}
  \nonumber
  E_{\rm reg}(s)  = & \;\EE_{\ST_n \sim p(x,y)} F_{\rm reg}(s)\\
  \label{equ:regScoreExp}
  =&\int_{\SX^n}\sum_{\#y\in\SY^n} \prod_{i=1}^n p(x_i,y_i) \bigg[ \frac{1}{n}\sum_{i=1}^n \Big(\ell(y_i,h(x_i))-s(x_i)\Big )^2  \bigg ] dx_1\cdots dx_n \\
  \nonumber
  = &  \int\limits_{\SX} \sum_{y\in\SY} p(x,y)\,\Big (\ell(y,h(x))-s(x)\Big)^2 dx \:,
\end{align}
and showing that its minimizer is the conditional risk $r(x)$ which is by definition a proper uncertainty score. This is ensured by the following theorem.

\begin{restatable}{theorem}{regEstimatorSolution}\label{thm:regEstimSolution} The conditional risk $r(x)$ defined by~\equ{equ:condRisk} is an optimal solution to $\min\limits_{s:\SX \to \Re} E_{\rm reg}(s)$.
\end{restatable}

\subsection{Minimization of SELE loss}
\label{sec:SELE}

In this section we define a computationally manageable proxy of AuRC which we call SElective classifier LEarning (SELE) loss. The SELE loss $\SELE\colon \Re^n \times \SX^n\times\SY^n\rightarrow\Re_+$ is defined as~\footnote{We assume that the training set $\ST_n$ has at least two examples, i.e. $n\geq 2$.}
\begin{equation}
 \label{equ:sampleSele}
   \SELE(s, \ST_n) =\displaystyle\frac{1}{n^2}\sum_{i=1}^n\sum_{j=1}^n \ell( y_i,h(x_i)) \leftbb s(x_i)\leq s(x_j)\rightbb \:.
\end{equation}
%
%
In contrast to $\AUC$ the computation of $\SELE$ does not require sorting the examples, i.e. we got rid of the permutations that make the evaluation hard. $\SELE$ is still hard to optimize directly due to the step function in its definition. After replacing the step function $\leftbb\cdot \rightbb$ by a logistic function we obtain its proxy $\SELEcvx\colon\Re^n\times\SX^n\times\SY^n\rightarrow\Re_+$ defined as
\begin{equation}
    \label{equ:sampleSeleCvx}
   \SELEcvx(s, \ST_n) =\displaystyle\frac{1}{n^2}\sum_{i=1}^n\sum_{j=1}^n \ell( y_i,h(x_i))\log \big ( 1+\exp(s(x_j)-s(x_i)\big ) \:.
\end{equation}
The function $\SELEcvx(s,\ST_n)$ is smooth and convex w.r.t. the argument $s$ and hence it is amenable to optimization. Minimization of $\SELEcvx$ is in the core of the proposed learning algorithm which works as follows. 

Given a hypothesis space $\SF\subset \{s\colon\SX\rightarrow\Re\}$, a classifier $h(x)$ and training set $\ST_n$, the {\em SELE score} $s\colon\SX\rightarrow\Re$ is a solution to the problem $\min_{s\in\SF} \SELEcvx(s,\ST_n)$. 
We justify the proposed algorithm empirically in Section~\ref{sec:experiments}. The theoretical justification is based on the following three arguments:
\begin{enumerate}
    \item The value of $\SELE$ is a tight approximation of $\AUC$.
    In case when value of $s(x)$ is different for each input in $\ST_n$, i.e. $s(x_i)\neq s(x_j)$, $\forall i\neq j$, then it holds that
    \[
    \SELE(s,\ST_n)\leq \AUC(s,\ST_n) \leq 2\cdot \SELE(s,\ST) \:.
    \]
    The first inequality follows from 
    \[
   \SELE(s,\ST_n) = \frac{1}{n} \sum_{i=1}^n \frac{1}{n} \hat{L}(i,s) \leq \frac{1}{n}\sum_{i=1}^n \frac{1}{i} \hat{L}(i,s) = \AUC(s,\ST_n)
   \:.
   \]
   The second inequality is ensured by Theorem~\ref{theorem:SeleUpperBoundsAuRC}.
    \item The uncertainty score estimator defined as $\hat{s}\in\argmin_{s\in [0,1]^\SX} \SELE(s,\ST_n)$ is {\em Fisher consistent}. Namely, Theorem~\ref{thm:scoringFunction} ensures that a population minimizer of $\SELE$ is a proper uncertainty score.
    \item The Fisher consistency is preserved even for the smooth proxy $\SELEcvx$ the minimization of which is in the core of the proposed algorithm. Namely, Theorem~\ref{thm:seleProxySolution} ensures the population minimizer of $\SELEcvx$ is a proper uncertainty score.
\end{enumerate}

\begin{restatable}{theorem}{SeleUpperBoundsAuRC}\label{theorem:SeleUpperBoundsAuRC} The inequality $\AUC(s, \ST_n) \le 2 \cdot \SELE(s, \ST_n)$ holds true for any $s\colon\SX\rightarrow\Re$ and $\ST_n=\{ (x_i,y_i)\in\SX\times\SY\mid i=1,\ldots,n\}$.
\end{restatable}

To show the Fisher consistency of $\SELE$ we define its expectation with respect to i.i.d. generated examples $\ST_n$, i.e., 
\begin{align}
  \nonumber
  E_{\rm sele}(s)  = &\int_{\SX^n}\sum_{\#y\in\SY^n} \prod_{i=1}^n p(x_i,y_i) \SELE( s, \ST_n) dx_1\cdots dx_n \\
  \nonumber
    = & \frac{1}{n^2} \int_{\SX^n} \prod_{i=1}^n p(x_i) \sum_{i=1}^n \sum_{j=1}^n r(x_i) \, \leftbb s(x_i) \leq s(x_j) \rightbb \, dx_1\cdots dx_n\\
    \nonumber
    = & \frac{1}{n^2}\sum_{i\neq j} \int_{\SX}\int_{\SX} p(x_i)p(x_j)r(x_i)\leftbb s(x_i) \leq s(x_j) \rightbb dx_i\, dx_j\\
    \nonumber
    &+\frac{1}{n^2}\sum_{i=1}^n \int_{\SX} \int_{\SX} p(x_i)p(x_i)r(x_i)\leftbb s(x_i)\leq s(x_i) \rightbb dx_i\, dx_i \\
    \label{equ:SeleExpectation}
    = & \frac{n^2-n}{n^2}\int_{\SX} p(x)\,r(x) \left( \int_{\SX} p(z) \,\leftbb s(x)\leq s(z) \rightbb dz \right) dx + 
    \frac{1}{n}\int_{\SX} \int_{\SX} p^2(x) r(x) dx\, dx \:.
\end{align}
Minimizers of $E_{\rm sele}$ are characterized by the following theorem~\footnote{$\int_{\SX}\!\int_{\substack{z\neq x \\ s^*(z)=s^*(x)}}\!\!\! f(x,\!z)dz\,dx$ stands for $\int_{\SX}\!\int_{\SX'} \! f(x,\!z)dz\,dx$ where $\SX'=\{z\!\in\!\SX \mid\! z\!\neq\! x \,\land\, s^*(z)\!=\!s^*(x)\}$, etc.}.


\begin{restatable}{theorem}{minSeleSolution}\label{thm:scoringFunction}
    A function $s^*:\SX \to \Re$ is an optimal solution to $\min_{s:\SX \to \Re} E_{\rm sele}(s)$ iff
    \begin{align}
	    \int_{\SX}&\int_{\substack{z\neq x \\ s^*(z)=s^*(x)}}\!\!\!\!\! \max\{r(x),r(z)\} p(x)p(z)dz\, dx = 0\, \text {, and} \label{equ:scoringFunction-cond1} \\
	    \int_{\SX} &\int_{\substack{ r(z)< r(x) \\ s^*(z)> s^*(x)}}\!\! \left(r(x)-r(z)\right) p(x)p(z)dz\, dx = 0\,. \label{equ:scoringFunction-cond2}
	\end{align}
\end{restatable}
%

\noindent
The conditions~(\ref{equ:scoringFunction-cond1}) and (\ref{equ:scoringFunction-cond2}) imply that the conditional 
expectations $\EE_{x,z\sim p(x)}[\max\{r(x),r(z)\}\mid z\neq x \wedge s^*(x)=s^*(z)]$ and $\EE_{x,z\sim p(x)}[r(x)-r(z)\mid r(z) < r(x) \wedge s^*(z) > s^*(x)]$ are both zero. If combined it further implies that a subset of input space $\SX'=\{(x,z)\in\SX\times\SX\mid r(z) < r(x) \wedge s^*(z) > s^*(x)\}$, on which the order is violated, has probability measure zero. In other words the optimal $s^*(x)$ preserves the ordering induced by $r(x)$ {\em almost surely}~\footnote{This means that the condition can be violated at most on a subset of $\SX$ with probability measure zero. }.


\begin{corollary}\label{cor:fromRtoS}
Any function $s:\SX \to \Re$ fulfilling
\begin{alignat}{1}
   \label{equ:uniqueValue}
    \forall (x,x') &\in \SX \times \SX : x \neq x' \Rightarrow s(x) \neq s(x'),\, \text{ and }\\
    \forall (x,x') &\in \SX \times \SX : r(x) < r(x') \Rightarrow s(x) < s(x')
\end{alignat}
satisfies the optimality conditions of Theorem~\ref{thm:scoringFunction}.
\end{corollary}
Note that \equ{equ:uniqueValue} requires the minimizer of $\SELE$ to assign a unique value to each input $x\in\SX$ which is not necessary for the score to be proper. Hence the minimizers of $\SELE$ form a subset of all proper uncertainty scores.



To show the Fisher consistency of the smooth proxy $\SELEcvx$ we define its expectation with respect to i.i.d. generated examples $\ST_n$, i.e.,
\begin{align}
  \nonumber
  E_{\rm proxy}(s)  = & \int_{\SX^n}\sum_{\#y\in\SY^n} \prod_{i=1}^n p(x_i,y_i) \, \SELEcvx( s, \ST_n) \,dx_1\cdots dx_n \\
  \nonumber
    = &  \frac{n^2-n}{n^2}\int_{\SX} p(x)\,r(x) \left( \int_{\SX} p(z) \, \log\big( 1 + \exp(s(z)-s(x)) \big ) dz \right) dx \\ 
        \label{equ:objProxyExp} & + 
    \frac{\log(2)}{n}\int_{\SX} p(x)\, r(x) dx \:.
\end{align}
We omitted the derivation as it is similar to~\equ{equ:SeleExpectation}. The key property of the minimizers of $E_{\rm proxy}$ is stated in the following theorem. 

\begin{restatable}{theorem}{seleProxySolution}\label{thm:seleProxySolution} Let $s^*\colon\SX\rightarrow\Re$ be an optimal solution to $\min_{s\colon \SX\rightarrow\Re} E_{\rm proxy}(s)$. Then, the condition
\[
    \forall (x,x') \in \SX \times \SX : r(x) < r(x') \Rightarrow s^*(x) < s^*(x')
\]
is satisfied almost surely.
\end{restatable}

\section{Related Works}\label{sec:relatedWorks}

The cost-based rejection model was proposed in~\cite{Chow-RejectOpt-TIT1970} who also provides the optimal strategy in case the distribution $p(x,y)$ is known, analyzes the error-reject trade-off, and proves basic properties of the error-rate and the reject-rate, e.g. that both functions are monotone with respect to the reject cost. The original paper considers the risk with 0/1-loss only. The model with arbitrary classification costs was analyzed e.g. in~\cite{Torella-OptRejectRule-2000,SchlesingerHlavac-10Lectures02}. 

The bounded-improvement model was coined in~\cite{Pietraszek-AbstainROC-ICML2005}. He assumes that a classifier score proportional to the posterior probabilities is known and the task is to find only an optimal decision threshold, which is done numerically based on ROC analysis. The original formulation assumes two classes and 0/1-loss. In this article we consider a straightforward generalization of the bounded-improvement model, see Problem~\ref{task:boundedImprovement}, which allows arbitrary number of classes, arbitrary loss and puts no constraint on the class of optimal strategies. We have proved necessary and sufficient conditions on an optimal strategy in case $p(x,y)$ is known. We showed that a particular optimal solution is composed of the Bayes classifier and the randomized Bayes selection function. In addition, we coined the bounded-coverage model, see Problem~\ref{task:boundedCoverage}, the definition of which is symmetric to the bounded-improvement model. Although the bounded-coverage model seems equally useful in practice we are unaware of its formal definition in a literature. 


There exist another formulations of rejection models for two-class classifiers. For example, in \cite{Hanczar-ClsReject-BIO08} the objective is to maximize the coverage under the constraints that each class has an error rate below a specific threshold. Hence it can be seen as a generalization of the bounded-improvement model. The objective of a rejection model proposed in~\cite{Lei-ClassConf-Biom14} is to maximize the total coverage under the constraint that each class has coverage above a specific threshold. None of the two papers analyzes optimal strategies of the corresponding models. 

A common approach to construct the selective classifiers for the cost-based model is based on the plug-in rule, which involves learning the class posterior distribution from examples and plugging the distribution to the formula defining the Bayes-optimal strategy~\equ{equ:bayesPredictor}. In the case of 0/1-loss the plug-in rule rejects based on the maximal class posterior which is denoted as {\em Maximal Class Probability} (MCP) rule (see Example~\ref{example:MCP}). The MCP rule is probably most frequently used uncertainty score in the literature. The statistically consistency of the plug-in rejection rule is discussed in~\cite{Herbei-ClassRejectOpt-CJS2006}. \cite{Fumera-MultiReject-2000} investigate how errors in estimation of the posterior distribution affect effectiveness of the plug-in rule and they also try to improve its performance by using class-specific thresholds. Other methods trying to improve the plug-in rule by tuning multiple thresholds were proposed in~\cite{Kummert-LocalReject-ICANN2016,Fisher-LocalRejectClassif-NC2016}. In our work we have derived the optimal strategies for the bounded-improvement model and the newly proposed bounded-coverage model. Our results thus provides a recipe to construct the plug-in rules also for these two rejection models. 


There exist many modifications of standard prediction models to learn reject option classifiers for the cost-based model. For example, extensions of the Support Vector Machines to learn a reject option classifier have been studied extensively in~\cite{Grandvalet-SVMwithRejectOpt-NIPS2008,Bartlett-RejectHinge-JMLR2008,Yuan-ClassifRejectOpt-JMLR2010}. These works are limited to two-class problems and $0/1$-loss. Learning leads to minimization of a convex surrogate of the cost-based model's objective function. Under some conditions the algorithms are statistically consistent. A boosting algorithm for learning a two-class classifier with reject option is proposed in~\cite{Cortes-BoostWitAbst-NIPS2016}. The algorithm minimizes a convex surrogate for the cost-based model and show that the surrogate is calibrated with Bayes solution. Learning prototype-based classifier with rejection option has been addressed in~\cite{Villmann-RejectProto-AISC2016}. All these methods require the reject cost to be fixed at the time of learning, and hence changing the cost requires re-training. In contrast, we propose algorithms to learn the proper uncertainty score on top of a pre-trained classifier so that the risk-coverage trade-off can be set by tuning the reject threshold without re-training.

For many prediction models it is easy to devise an ordinal uncertainty score from outputs of the learned (non-reject) classifier. Such strategies are often heuristically based but work reasonably in practice. For example, \cite{Lecun-OcrNN-NIPS90} proposed a reject strategy for a Neural Network classifier based on thresholding either the output of the maximally activated unit of the last layer or a difference between the maximal and runner upper output units. 
Other heuristically based strategies for neural networks were evaluated in~\cite{Zaragoza-ConfInNN-IPMU98,Fisher-EfficientReject-Neuro2015}.
In case of Support Vector Machine classifiers~\citep{Vapnik-StatLearning-98} the trained linear score, proportional to the distance between the input and the decision hyper-plane, is directly used as the uncertainty score~\citep{Fumera-SvmRejectOpt-2002}. We denote this approach as the {\em margin score} and use it as a baseline in our experiments.



Learning of a selective classifier optimal for the bounded-improvement model was discussed in~\cite{ElYaniv-SelectClass-JMLR10}. Their method requires a noisy-free scenario, i.e. they maximize the coverage under the constraint that the selective risk is zero. They provide a characterization of the lower and upper bound of the risk-coverage curves in PAC setting. \cite{Geifman-SelectClass-NIPS2017} assume a selective classifier based on thresholding an uncertainty score and show how to find a decision threshold for the bounded-coverage model which is optimal in PAC sense. They do not address the problem of learning the uncertainty score. We complement their work by showing that the thresholding based selective classifier is an optimal solution when the uncertainty function is proper, and we propose algorithms to learn the proper uncertainty score from examples.

Recent works address uncertainty prediction in context of deep learning~\citep{Laksh-UncertDeep-NIPS2017,Jiang-TrustOrNot-NIPS2018,Corbiere-Failure-NeurIPS2019}. These works do not formulate the problem to be solved explicitly as a rejection model. However, they evaluate their uncertainty scores empirically in terms of the Risk-Coverage curve and the Area under the RC curve which we have shown to be connected with the bounded-coverage model (see Section~\ref{sec:AuRC}).
\cite{Laksh-UncertDeep-NIPS2017} construct the MCP rule from a posterior distribution modeled as an ensemble of neural networks trained from multiple random initialization. They use adversarial examples to smooth the posterior estimate. \cite{Jiang-TrustOrNot-NIPS2018} propose a Trust Score as the ratio between the distance from the test sample to the samples of the nearest class with a different label and the distance to the samples with the same labels as the predicted class. 
\cite{Corbiere-Failure-NeurIPS2019} propose a True Class Probability (TCP) score as a measure of prediction confidence. The TCP predicts the value of  $p(y^*\mid x)$ where $y^*$ is the ground truth label. They learn a NN, so called ConfNet, by minimizing L2-loss between the ConfNet output and $\hat{p}(y_i\mid x)$ on training examples, where $\hat{p}(y\mid x)$ is the soft-max distribution trained by standard cross-entropy loss. 
They show that the TCP empirically outperforms the MCP score and the Trust Score~\citep{Jiang-TrustOrNot-NIPS2018} in terms of the AuRC metric. Both \cite{Jiang-TrustOrNot-NIPS2018,Corbiere-Failure-NeurIPS2019} consider the two-stage approach to learn the uncertainty score similarly to our paper. We show empirically that our proposed SELE score, besides having a theoretical backing and being applicable for a generic classification problem, outperforms the the state-of-the-art TCP score. 



\section{Experiments}
\label{sec:experiments}

In Section~\ref{sec:UncertaintyLearning}, we outlined two risk minimization based methods to learn the uncertainty score $s(x)$ for a pre-trained predictor $h(x)$, namely, the algorithm based on i) {\em loss regression} (Section~\ref{sec:REG}) and ii) {\em minimization of SELE loss} (Section~\ref{sec:SELE}). We have shown that both methods are Fisher consistent, i.e. they are guaranteed to find the proper score in the idealized setting when the distribution $p(x,y)$ is known (estimation error is zero), the hypothesis space $\SF$ contains the proper score (approximation error is zero) and the loss minimizer can be found exactly (optimization error is zero). In this section we evaluate these methods experimentally on real data when all assumptions are presumably violated. We design the experiments so that the optimization and the estimation error are small by using a large number of training examples and linear rules making the loss minimization a convex problem. We compare against the recently proposed True Class Probability (TCP) score~\citep{Corbiere-Failure-NeurIPS2019} which is learned from examples like the proposed methods. Unlike the proposed methods, the TCP requires the prediction model $h(x)$ to provide an estimate of the posterior $p(y\mid x)$, hence it is not applicable to fully discriminative models like e.g. SVMs. We emphasize that the experiments are meant to be a proof of concept rather than an exhaustive comparison to all existing methods. On the other hand, we are not aware of any other generic method (i.e. being not connected to a particular prediction model) we could compare against. 

To demonstrate that the proposed methods are generic we consider three different categories of prediction problems: classification, ordinal regression and structured output classification. For each prediction problem we use several benchmark datasets and frequently used prediction models like the Logistic Regression (LR), three variants of Support Vector Machines (SVMs) and Gradient Boosted Trees. For each prediction model there exists an uncertainty score that is commonly used in practice, like e.g. Maximal Class Probability (MCP) for logistic regression or distance to the decision hyper-plane (a.k.a. margin score) for SVMs. We use these uncertainty scores as additional baselines in our experiments. 

\subsection{Compared methods for uncertainty score learning}
\label{sec:CompetingMethods}

In this section we describe three algorithms that use a training set $\ST_n=\{(x_i,y_i)\in\SX\times\SY\mid i=1,\ldots,n\}$ to learn an uncertainty score $s(x)$ for a pre-trained classifier $h(x)$. We consider linear scores $s_{\#\theta}(x)=\lz\#\theta,\#\psi(x)\pz$, where $\#\theta\in\Re^m$ are parameters to be learned and $\#\psi\colon\SX\rightarrow\Re^m$ is a fixed mapping that will be defined for each prediction model separately in the following sections. 
All evaluated methods are instances of regularized risk minimization framework. In all cases learning leads to an unconstrained minimization of a convex objective  $F(\#\theta)=\frac{C}{2}\|\#\theta\|^2+\hat{R}(\#\theta,\ST_n)$, where $C>0$ is a regularization constant and $\hat{R}(\#\theta,\ST_n)$ is an empirical risk defined by each method differently. The optimal value of $C$ is selected from $\{0,1,10,100,1000\}$ based on the minimal value of the AuRC evaluated on a validation set.

\subsubsection{Regression score} \label{sec:REGrule} The parameters $\#\theta\in\Re^m$ are learned by minimizing a convex function 
\[
   F_{\rm REG}(\#\theta) = \frac{C}{2}\|\#\theta\|^2 + \frac{1}{n}\sum_{i=1}^n \big ( \ell(y_i,h(x_i))- s_{\#\theta}(x_i) \big )^2 \:.
\]
Minimization of $F_{\rm REG}(\#\theta)$ is an instance of the ridge regression that can be solved efficiently e.g. by Singular Value Decomposition (SVD).

\subsubsection{SELE score} \label{sec:SELErule} Evaluation of the proposed SELE loss $\SELEcvx(s, \ST_n)$ as defined by~\equ{equ:sampleSeleCvx} requires $\SO(n^2)$ operations. To decrease the complexity we approximate its value by splitting the examples into chunks and computing the average loss over the chunks. Namely, the parameters $\#\theta\in\Re^m$ are learned by minimizing a convex function  
\[ 
    F_{\rm SELE}(\#\theta) = \frac{C}{2}\|\#\theta\|^2 + \frac{1}{P}\sum_{i=1}^P \SELEcvx(s, \ST^k_n)\,,
\]
where $C>0$ is a regularization constant and $\ST^1_n\cup \ST^2_n \cup \cdots \cup\ST^P_n$ is a randomly generated partition of the training set $\ST_n$ into $P$ approximately equally sized batches. In all experiments we used $P={\rm round}(n/500)$, i.e. the chunks contain around 500 examples. We minimize $F_{\rm SELE}(\#\theta)$ by the Bundle Method for Risk Minimization (BMRM) algorithm~\citep{Teo-BMRM-JMLR0} which is set to find a solution whose objective is at most 1$\%$ off the optimum~\footnote{We use  $(F_{\rm primal}-F_{\rm dual})/F_{\rm primal}\leq 0.01$ as the stopping condition of the BMRM algorithm.}.
The total computation time of the BMRM algorithm is in order of units of minutes for all dataset using a contemporary PC.

\subsubsection{True Class Probability score}\label{sec:TCPscore} 
The TCP~\citep{Corbiere-Failure-NeurIPS2019} was originally designed for getting uncertainty score on top of a Convolution Neural Network (CNN) trained with cross-entropy loss. The setting we consider here can be seen as the original method applied to a single layer CNN. Namely, let $\hat{p}(y\mid x)$ be an estimate of the posterior distribution, in our experiments provided by the Logistic Regression (a.k.a. single layer NN). The parameters $\#\theta\in\Re^m$ are learned by minimizing a convex function  
\[
   F_{\rm TCP}(\#\theta) = \frac{C}{2}\|\#\theta\|^2 + \frac{1}{n}\sum_{i=1}^n ( \hat{p}(y_i\mid x_i) - s_{\#\theta}(x_i))^2 \:.
\]
Minimization of $F_{\rm REG}(\#\theta)$ is an instance of the ridge regression which we solve by SVD.

\subsection{Benchmark problems} 


\subsubsection{Classification} 
\label{sec:ClassificationModels}
Given real valued features $\#x\in\SX\subseteq\Re^d$, the task is to predict a hidden state $y\in\SY=\{1,\ldots,Y\}$ so that the  expectation of $0/1$-loss $\ell(y,y')=100\,\leftbb y\neq y'\rightbb$ is as small as possible~\footnote{Due to the factor 100 the reported errors correspond to the percentage of misclassified examples.}. We selected 11 classification problems from the UCI repository~\citep{Dua-UCI-2017} and libSVM datasets~\citep{Chang-libSVM-2011}. The datasets are summarized in Table~\ref{tab:dataset_sumarry}. We chose the datasets with sufficiently large number of examples relative to the number of features, as we need to learn both the classifier and the uncertainty score and simultaneously keep the estimation error low. Each dataset was randomly split 5 times into 5 subsets, Trn1/Val1/Trn2/Val2/Tst, in ratio $30/10/30/10/20$ (up to CODRNA with ratio $25/5/20/20/30$ and COVTYPE with ratio $28/20/2/20/30$). The subsets Trn1/Val1 were used for learning and tuning the best regularization constant of the classifier $h(x)$. The subsets Trn2/Val2 were used for learning and tuning the regularization constant of the uncertainty score $s(x)$ as described in Section~\ref{sec:CompetingMethods}. All features were normalized to have zero mean and unit variance. The normalization coefficients were estimated using only the Trn1 and Trn2 subsets, respectively. The Tst subset was used solely to evaluate the test performance. 

We used two prediction models: Logistic Regression (LR)~\citep{Hastie-Elements-09} and Support Vector Machines (SVM)~\citep{Vapnik-StatLearning-98}. 

\paragraph{Logistic Regression} learns parameters $\#\theta_{\rm LR}=((\#w_y,b_y)\in\Re^d\times\Re\mid y\in\SY)$ of the posterior probabilities $\hat{p}_{\#\theta}(y\mid \#x)\approx \exp( \lz \#w_y, \#x \pz+b_y)$ by maximizing the regularized log-likelihood $F_{\rm LR}(\#\theta)=\frac{C}{2}\|\#\theta\|^2+\frac{1}{m}\sum_{i=1}^m \log\big( \hat{p}_{\#\theta}(y_i\,|\, \#x_i)\big)$. The optimal $C$ was selected from $\{1,10,100,1000\}$ based on the validation classification error. After learning $\#\theta_{\rm LR}$ we used the plug-in Bayes classifier $h(\#x)=\argmax_{y\in\SY}\hat{p}_{\#\theta_{\rm LR}}(y\mid \#x)$. As a baseline uncertainty score we use the {\em plug-in} class conditional risk $\hat{r}(\#x)=1-\hat{p}_{\#\theta_{\rm LR}}(h(\#x)\,|\, \#x)$. In accordance with the literature we refer to this baseline as the {\em Maximal Class Probability (MCP) rule}. As shown in Section~\ref{sec:pluginRisk}, the MCP score is the proper uncertainty score provided the estimate $\hat{p}(y\mid\#x)$ matches the true posterior $p(y\mid\#x)$. 

\paragraph{Support Vector Machines} learn parameters $\#\theta_{\rm SVM}=((\#w_y,b_y)\in\Re^d\times\Re\mid y\in\SY)$ of the linear classifier $h(\#x)=\argmax_{y\in\SY}( \lz \#w_y,\#x\pz+b_y)$ by minimizing $F_{\rm SVM}(\#\theta)=\frac{C}{2}\|\#\theta\|^2 +\frac{1}{m}\sum_{i=1}^m \max_{y\in\SY}\big(\leftbb y\neq y_i\rightbb+\lz \#w_y-\#w_{y_i},\#x_i\pz \big)$. The optimal $C$ was selected in the same way as in case of the LR. As the baseline uncertainty measure we use $s(\#x)=\max_{y\in\SY}\lz \#w_y,\#x\pz+b_y$. In the binary case $|\SY|=2$, the setting was $\#\theta_{\rm SVM}=(\#w,b)$, $h(\#x)=\sgn(\lz \#w,\#x\pz+b)$, $F_{\rm SVM}(\#\theta)=\frac{C}{2}\|\#\theta\|^2+\frac{1}{m}\sum_{i=1}^m\max\{ 0,1-y_i(\lz\#w,\#x_i\pz+b)\}$ and $s(\#x)=|\lz\#w,\#x\pz+b|$. In both cases, the value of $s(\#x)$ is proportional to a distance between the input $\#x$ and the decision boundary. We denote this baseline as the {\em margin score}. 

Given the pre-trained LR or SVM classifier $h(x)$, we apply the  methods from Section~\ref{sec:CompetingMethods} to learn the uncertainty score
\begin{equation}
 \label{equ:linearUncertainRule1}
    s_{\#\theta}(\#x) = \lz \#w_{h(\#x)}, \#x \pz +b_y \:,
\end{equation}
where $\#\theta=((\#w_y,b_y)\in\Re^d\times\Re\mid y\in\SY)$ are the parameters to be learned. It is seen that the rule~\equ{equ:linearUncertainRule1} can be re-written as $s_{\#\theta}(x)=\lz \#\theta,\#\psi(x)\pz$, where $\#\psi\colon\SX\rightarrow\Re^d$ is an appropriately defined feature map. The form of the score~\equ{equ:linearUncertainRule1} can be justified by noting that its special instance is the margin score which is obtained after substituting $\#\theta_{\rm SVM}$ for $\#\theta$.

\subsubsection{Ordinal regression} \label{sec:SVOR} 
The task is to predict a hidden state from $\SY=\{1,\ldots,Y\}$ based on real valued features $\SX\subseteq\Re^d$. Unlike the classification problem, the hidden states $\SY$ are assumed to be ordered and the goal is to minimize the expectation of the Mean Absolute Error $\ell(y,y')=|y-y'|$. We selected 11 regression problems from UCI repository~\citep{Dua-UCI-2017}. The datasets are summarized in Table~\ref{tab:dataset_sumarry}. The real-valued hidden states were discretized into $Y$ bins which are constructed to get uniform class prior. Each dataset was randomly split 5 times into 5 subsets, Trn1/Val1/Trn2/Val2/Tst, in ratio $30/10/30/10/20$. We used the same normalization and evaluation protocol as described for the classification benchmarks.

As a prediction model we used a variant of the Support Vector Machine algorithm developed for ordinal regression~\citep{Chu-SVOR-05}~\footnote{\citep{Chu-SVOR-05} introduced two variants of SVOR algorithm. We use so called SVOR with {\em implicit constraints} which is designed for minimization of the MAE loss.}.

\paragraph{Support Vector Ordinal Regression (SVOR)} learns parameters $\#\theta_{\rm SVOR} = (\#w\in\Re^d, (b_1,\ldots,b_{Y-1})\in\Re^{Y-1})$ of the ordinal linear classifier $h(\#x)= 1+\sum_{y=1}^{Y-1}\leftbb \lz \#x,\#w\pz > b_y \rightbb$ by minimizing $F_{\rm SVOR}(\#\theta) = \frac{C}{2}\|\#\theta\|^2+\frac{1}{m}\sum_{i=1}^m \big ( \sum_{y=1}^{y_i-1} \max(0,1-\lz\#w,\#x_i\pz + b_{y}) +\sum_{y=y_i}^{Y-1}\max(0,1+\lz\#x_i,\#w\pz - b_y )\big )$. The optimal C was selected from $\{1,10,100,1000\}$ based on the validation MAE. The ordinal classifier can be thought of as a standard linear classifier composed of parallel decision hyper-planes. Similarly to the standard SVM, we use $s(\#x)=\min_{y\in\{1,\ldots,Y-1\}} | \lz \#x,\#w\pz - b_y|$ as a baseline uncertainty score. The value of $s(\#x)$ is proportional to the distance of $\#x$ to the closest hyper-plane hence we also denote it as the {\em margin score}. When learning the uncertainty from examples we use the parametrization~\equ{equ:linearUncertainRule1}.

\subsubsection{Structured Output Classification}\label{sec:DLIB}
Given an RGB image $\#x\in\SX=\{0,\ldots,255\}^{W\times H\times 3}$ capturing a human face, the task is to predict a pixel positions of $68$ landmarks  $\#y=(\#l_1,\ldots,\#l_{68})\in \SY=(\{1,\ldots,W\}\times\{1,\dots,H\})^{68}$ corresponding to contours of eyes, mouth, nose, etc. We use the 300-W dataset and the associated evaluation protocol which was created by organizers of landmark detection challenge~\citep{Sagonas-300W-ICCVW2013}. The 300-W dataset contains 5,807 faces each  annotated with 68 landmarks. The faces are split into 3,484 training, 1,161 validation and 1,162 test examples. The prediction accuracy is measured in terms of normalized average localization error $\ell(\#y,\hat{\#y}) = \frac{100}{{\rm iod}(\#y)}\sum_{i=1}^{68}\| \#l_i-\hat{\#l}_i\|$, where ${\rm iod}(\#y)$ is the inter-ocular distance computed from the ground-truth landmark positions $\#y$.

As the structured classifier $h(x)$ we use the landmark detector from DLIB package~\citep{dlib09}. The detector predicts the landmark positions based on HOG descriptors~\citep{Dalal-HOG-2005} of the input image using an ensemble of regression trees that are trained by gradient boosting~\citep{Kazemi-FaceAlign-CVPR14}. The DLIB landmark detector has been widely used by developers due to its robustness and exceptional speed even on a low-end hardware. The detector does not provide any measure of prediction uncertainty and it is unclear how to derive it from outputs of the regression trees. A commonly used uncertainty score for face recognition related prediction problems is a score function of a face detector. The face detector score is an output of a binary classifier trained to distinguish face from non-face images. The value of the score is high for well looking prototypical faces and low for corrupted or ``difficult" faces. As a baseline we use the score of the DLIB face detector which is a linear SVM classifier on top of HOG descriptors extracted from the image. 

When learning the uncertainty score from examples the score is a linear regressor $s_{\#\theta}(x) = \lz \#\theta,\#\psi(x)\pz$ on top of a feature vector $\#\psi(x)\in\Re^{2,448}$ which is a concatenation of HOG descriptors extracted from the input facial image $x$ along landmark positions predicted by the landmark detector $h(x)$. The DLIB detector uses the same features to predict the landmark positions hence the extra computational time required to evaluate the uncertainty score is  neglectable.

\begin{table}[h!]
    \centering
    \begin{tabular}{lccc}
    \multicolumn{4}{c}{Classification problems}\\
    \hline
    dataset & examples & feat & cls\\
    \hline
     AVILA  &  20,867  & 10 &    12\\
    CODRNA  &  331,152 &  8 &   2\\
    COVTYPE &  581,012 &  54 &  7\\
     IJCNN  &  49,990  &  22 &  2 \\
     LETTER &  20,000  & 16  & 26 \\
     MARKETING &  45,211& 51    & 2\\
   PENDIGIT &  10,992  & 16 &   10\\
  PHISHING  &  11,055  &  68 &   2\\
 SATTELITE  &   6,435  &  36 &   6\\
SENSORLESS  &  58,509  & 48 &   11\\
   SHUTTLE  &  58,000  &   9 &   7 \\
    \end{tabular}
    \hspace{0.8cm}
    \begin{tabular}{lccc}
    \multicolumn{4}{c}{Ordinal regression problems}\\
    \hline
    dataset & examples & feat & cls\\
    \hline 
     ABALONE &    4,177 &   10 &    19\\
      BANK   &  8,192 & 32 &   10\\
 BIKESHARE   &  17,379 &   11    & 10\\
CALIFORNIA   &  20,640 &  8 &   10\\
      CCPP   &   9568 &  4 &   10\\
       CPU   &   8192 &  21 &  10\\
  FACEBOOK   &   50,993 & 53 &   10\\
       GPU   &   24,1600 &   14 &   10\\
     METRO   &   48,204 &    30 &   10\\
        MSD & 499,671 &   90    & 41\\
SUPERCOND &    21,263 &   81 &   10
    \end{tabular}
    \caption{Summary of 11 classification problems (left) and 11 ordinal regression problems (right) selected from UCI repository~\citep{Dua-UCI-2017} and libSVM datasets~\citep{Chang-libSVM-2011}. The table shows the total number of examples, the number of features and the number of classes. }
    \label{tab:dataset_sumarry}
\end{table}

\subsection{Results}

\subsubsection{Classification problems}

For both classification models, LR and SVM, we recorded the test risk of the classifier $h(x)$ and the AuRC computed from $h(x)$ and uncertainty score $s(x)$ produced by the corresponding method under evaluation. In case of LR, we compare the baseline MCP score (sec~\ref{sec:ClassificationModels}) and scores learned from examples including the state-of-the-art TCP score (sec~\ref{sec:TCPscore}) and the two proposed SELE score (sec~\ref{sec:SELErule}) and Regression (REG) score (sec~\ref{sec:REGrule}). In case of SVM, we compare the baseline Margin score (sec~\ref{sec:ClassificationModels}) against the proposed SELE and REG scores. Note that TCP score is not applicable for SVM classifier as it does not provide an estimate of the posterior probability $p(y\mid x)$. The results are summarized in Table~\ref{tab:ClassifResults}. 

For each dataset we rank the compared methods according to the AuRC~\footnote{The score with smallest AuRC is ranked 1, the second smallest 2 and so on.}. Following the methodology of~\cite{Demsar-JMLR06} we summarize performance of each method by its average rank and use the Friedman test and the post-hoc Nemenyi test to analyze significance of the results.
\begin{itemize}
    \item  We used the Friedman test to check whether the measured average ranks are significantly different from the mean rank. The null hypothesis states that the compared scores are equivalent so that their average ranks should be equal. In both cases the null hypothesis is rejected for p-value $0.05$, i.e., the {\em performance of compared methods is significantly different}. 
    \item We used post-hoc Nemenyi test for pair-wise comparison. For each pair of methods it checks whether their average ranks are significantly different. In case of LR, when we compare $K=4$ methods using $N=11$ datasets, the critical difference for p-value $0.10$ is $CD=1.26$.
    By comparing the average ranks we conclude that {\em SELE score performs significantly better than MCP score and REG score}. In case of SVM, when we compare $K=3$ methods using $N=11$ dataset, the critical difference for p-value $0.10$ is $CD=0.98$. We conclude that {\em SELE performs significantly better than REG score and Margin score}. The data is not sufficient to reach any conclusion about other pairwise comparisons.
    The result of the Nemenyi test is visualized in Figure~\ref{fig:PairWiseCompClassif}.
\end{itemize}

We further computed relative improvement gained by using the scores SELE, TCP and REG, that are all learned from examples, with respect to the baseline scores derived from the classifier output, i.e. MCP in case of LR and Margin score in case of SVM. The results are summarized in Figure~\ref{fig:RelImprovClassif}. It is seen that the MCP uncertainty computed from the estimated $p(y\mid x)$ constitutes much stronger baseline than the Margin score of fully discriminative SVM model. The relative improvement of scores learned on top the LR is only moderate in contrast to the SVM classifier where the improvements are more significant and consistent. It is also seen that on majority of datasest the performance of the learned scores is similar taking into account the statistical error of the AuRC estimate. It is also worth mentioning that results of SELE score have the lowest variance of the AuRC estimates as seen from error bars in Figure~\ref{fig:RelImprovClassif}.

\begin{table}[h!]
  \centering
    \begin{tabular}{rccccc}
    & \multicolumn{1}{c}{LR+MCP} 
    & \multicolumn{1}{c}{LR+SELE} 
    & \multicolumn{1}{c}{LR+REG}  
    & \multicolumn{1}{c}{LR+TCP} &LR   \\
    & AuRC  &AuRC  & AuRC  & AuRC & R@100 \\  
   \hline
    AVILA   &  27.18$\pm$0.55  & {\bf 25.79$\pm$0.44} & 26.62$\pm$0.74  & 26.85$\pm$0.78  &  43.71$\pm$0.42 \\
    CODRNA  & 0.88$\pm$0.05  &{\bf 0.65$\pm$0.03} &  0.82$\pm$0.06  & 0.78$\pm$0.04  &  4.81$\pm$0.08 \\
    COVTYPE & {\bf 16.49$\pm$0.06} &  17.58$\pm$0.07  &  17.62$\pm$0.09  & 17.19$\pm$0.07  &  27.56$\pm$0.17 \\
    IJCNN   & 1.26$\pm$0.04  &{\bf 1.00$\pm$0.03} & 1.16$\pm$0.08  &  1.14$\pm$0.06  &  7.54$\pm$0.15 \\
    LETTER  & 7.43$\pm$0.40  &  {\bf 6.42$\pm$0.34} &  7.44$\pm$0.59  &  6.71$\pm$0.42  &  23.32$\pm$0.60 \\
    MARKETING &  2.60$\pm$0.31  & {\bf 1.88$\pm$0.11}  & 1.97$\pm$0.12  & 1.90$\pm$0.11  &  9.88$\pm$0.29 \\
    PENDIGIT  & {\bf 0.69$\pm$0.04} & 1.55$\pm$0.19  &  1.97$\pm$0.55  &  1.47$\pm$0.39  & 5.29$\pm$0.40 \\
    PHISHING  & 0.76$\pm$0.10  &  {\bf 0.75$\pm$0.10} &  0.91$\pm$0.31  & 0.85$\pm$0.25  &  6.29$\pm$0.44 \\
    SATTELITE & 3.83$\pm$0.26  &{\bf 3.68$\pm$0.27} &  4.93$\pm$1.07  & 4.52$\pm$0.85  &   15.06$\pm$0.46 \\
    SENSORLESS &  2.03$\pm$0.11  & {\bf 1.82$\pm$0.08} &  2.69$\pm$0.09  & 2.37$\pm$0.22  &   8.23$\pm$0.45 \\
    SHUTTLE   & 0.59$\pm$0.09  &  {\bf 0.26$\pm$0.07} & 1.24$\pm$0.51  & 0.58$\pm$0.13  &  3.36$\pm$0.25 \\
    \hline
    average rank   &     2.73           &   1.36        &     3.55          &      2.36 \\ 
    \multicolumn{5}{c}{(a) Uncertainty scores on top of LR classifier.}\\
    \end{tabular}
   \centering
    \begin{tabular}{rcccc}
    \multicolumn{5}{c}{}\\
    & \multicolumn{1}{c}{SVM+MARGIN} 
    & \multicolumn{1}{c}{SVM+SELE} 
    & \multicolumn{1}{c}{SVM+REG} & SVM \\
    & AuRC &AuRC & AuRC & R@100 \\  
    \hline
         AVILA &  31.65$\pm$0.83 &{\bf 25.26$\pm$0.67} & 25.95$\pm$0.75  &  43.34$\pm$0.70 \\ 
        CODRNA &  0.89$\pm$0.05  &{\bf 0.65$\pm$0.03} &  0.82$\pm$0.05  &  4.78$\pm$0.08 \\
       COVTYPE & 25.71$\pm$0.81  & 17.79$\pm$0.21  &{\bf 17.77$\pm$0.14} &  27.41$\pm$0.11\\ 
         IJCNN &  1.40$\pm$0.04  & {\bf 1.01$\pm$0.04} & 1.18$\pm$0.08  &  7.56$\pm$0.16 \\
        LETTER & 10.20$\pm$0.22  & {\bf 6.05$\pm$0.65} & 7.15$\pm$0.65  &  22.06$\pm$0.69 \\
     MARKETING &  2.24$\pm$0.20  & {\bf 1.97$\pm$0.10} & 2.04$\pm$0.20  &  10.48$\pm$0.39 \\
      PENDIGIT & 2.79$\pm$0.40  & {\bf 1.57$\pm$0.21} & 2.16$\pm$0.43  &  4.88$\pm$0.57 \\
      PHISHING &  0.84$\pm$0.12  & {\bf 0.72$\pm$0.12} & 0.90$\pm$0.30  &  6.37$\pm$0.44 \\
     SATTELITE & 4.75$\pm$0.60  & {\bf 3.82$\pm$0.27} & 5.44$\pm$0.68  &  15.36$\pm$0.37 \\
    SENSORLESS &  3.68$\pm$0.20  & {\bf 1.56$\pm$0.08} & 2.46$\pm$0.29  &  6.92$\pm$0.17 \\
       SHUTTLE & 1.31$\pm$0.47  & {\bf 0.24$\pm$0.07} & 0.55$\pm$0.15  &   2.02$\pm$0.15 \\
       \hline
        average rank &  2.82      &  1.09         &      2.09              & \\
    \multicolumn{5}{c}{(b) Uncertainty scores on top of SVM classifier.}\\
    \end{tabular}
    \caption{Performance of the uncertainty scores on 11 classification problems. The scores are constructed on top of the LR classifier and the SVM classifier measured in terms of AuRC. For each score we show the mean and the standard deviation of the test AuRC computed over 5 random splits. We compare the performance of scores learned from examples (SEL, REG, TCP) and the baseline scores derived from the classifiers output (MCP and Margin score). The last column shows the risk of the base (non-selective) classifier. All the values correspond to percentage of misclassification. The best results for each dataset are shown in bold. The last row shows the average rank.}
    \label{tab:ClassifResults}
\end{table}

\begin{figure}[h!]
    \begin{tabular}{cc}
     \includegraphics[width=0.42\columnwidth]{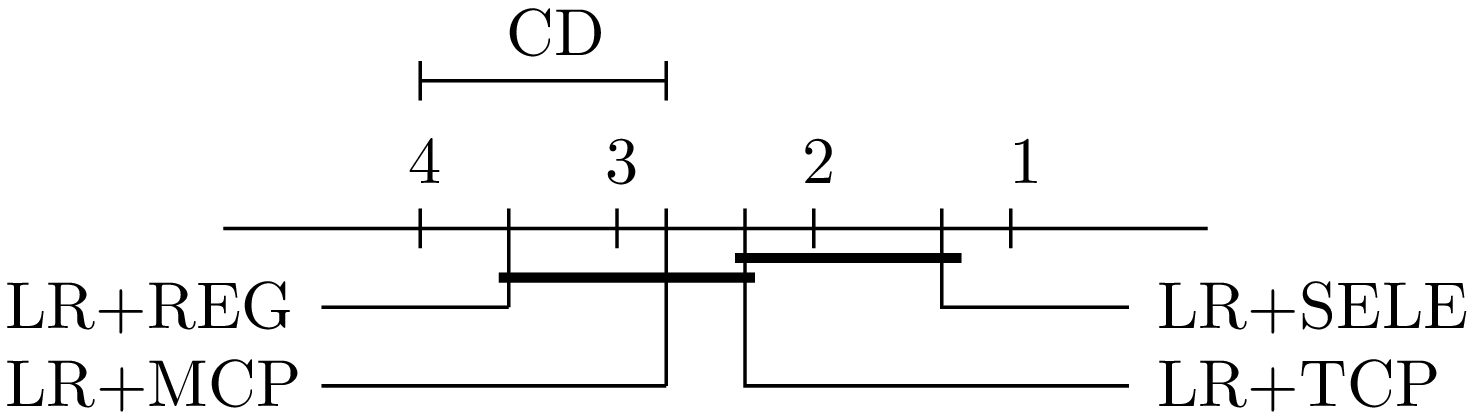} &
        \includegraphics[width=0.42\columnwidth]{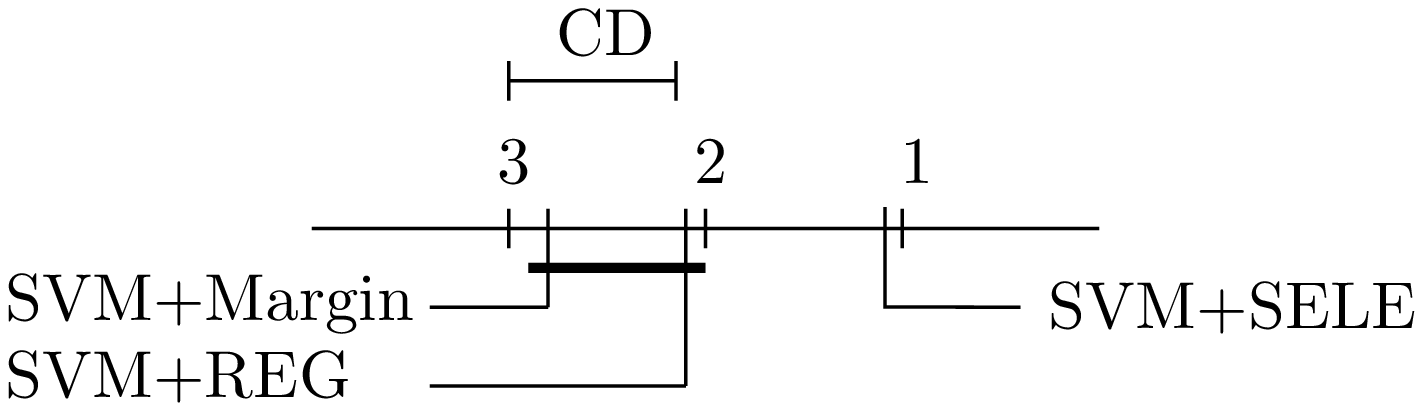} \\
      a) Average ranks for scores on top LR. & b) Average ranks for scores on top of SVM.
    \end{tabular}
    \caption{Comparison of all uncertainty scores against each other with the Nemenyi test. The test is computed separately for the LR classifier (4 scores compared) and the SVM classifier (3 scores compared). The figures show the average ranks for each score and the critical distance (CD). Groups of scores that are not significantly different at $p$-value $0.10$ are connected. 
    }
    \label{fig:PairWiseCompClassif}
\end{figure}

\begin{figure}[h!]            
            \begin{tabular}{cc}
 (a) Improvement over LR+MCP score. & (b) Improvement over SVM+Margin score.\\
        \includegraphics[width=0.45\columnwidth]{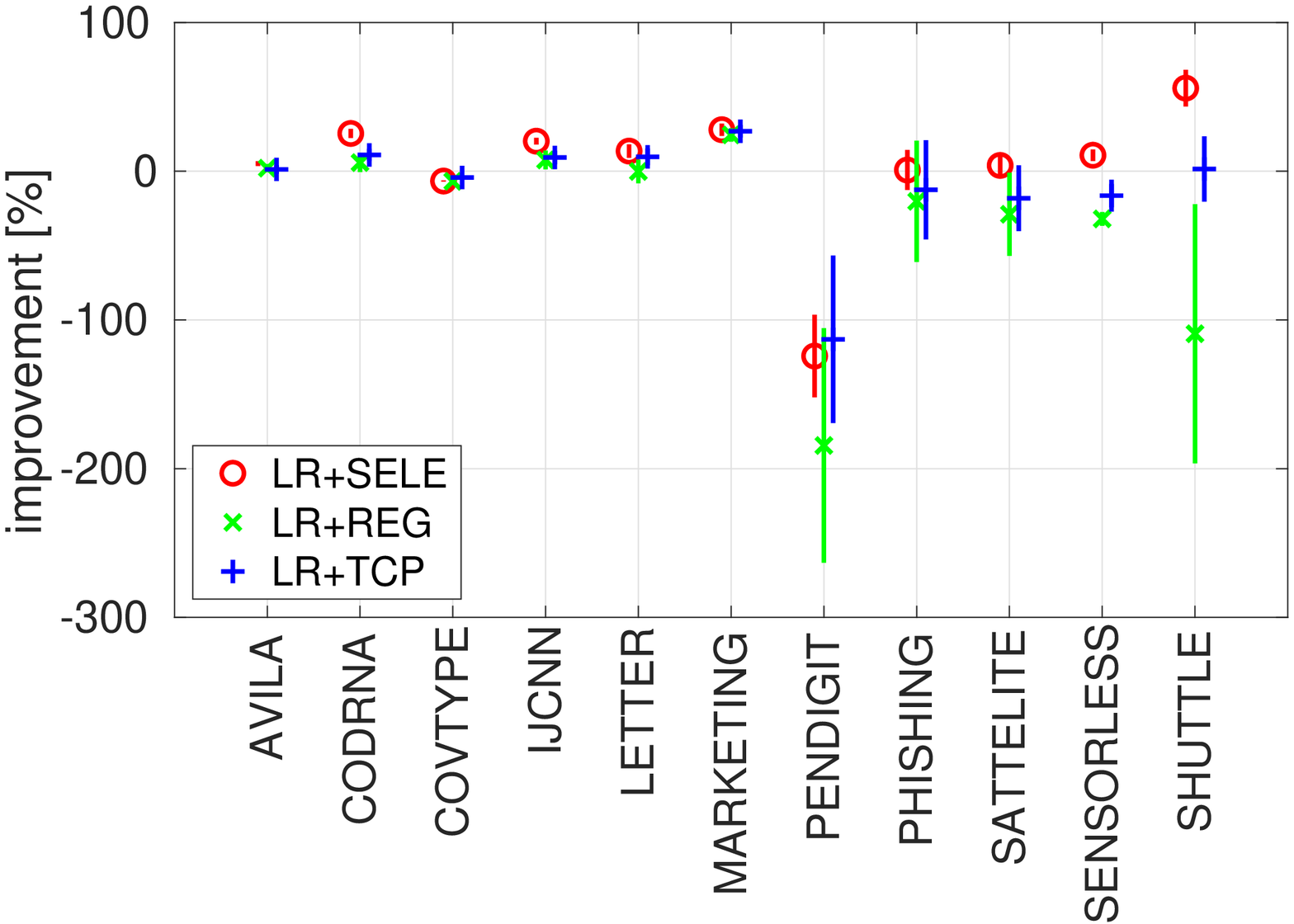} &
        \includegraphics[width=0.45\columnwidth]{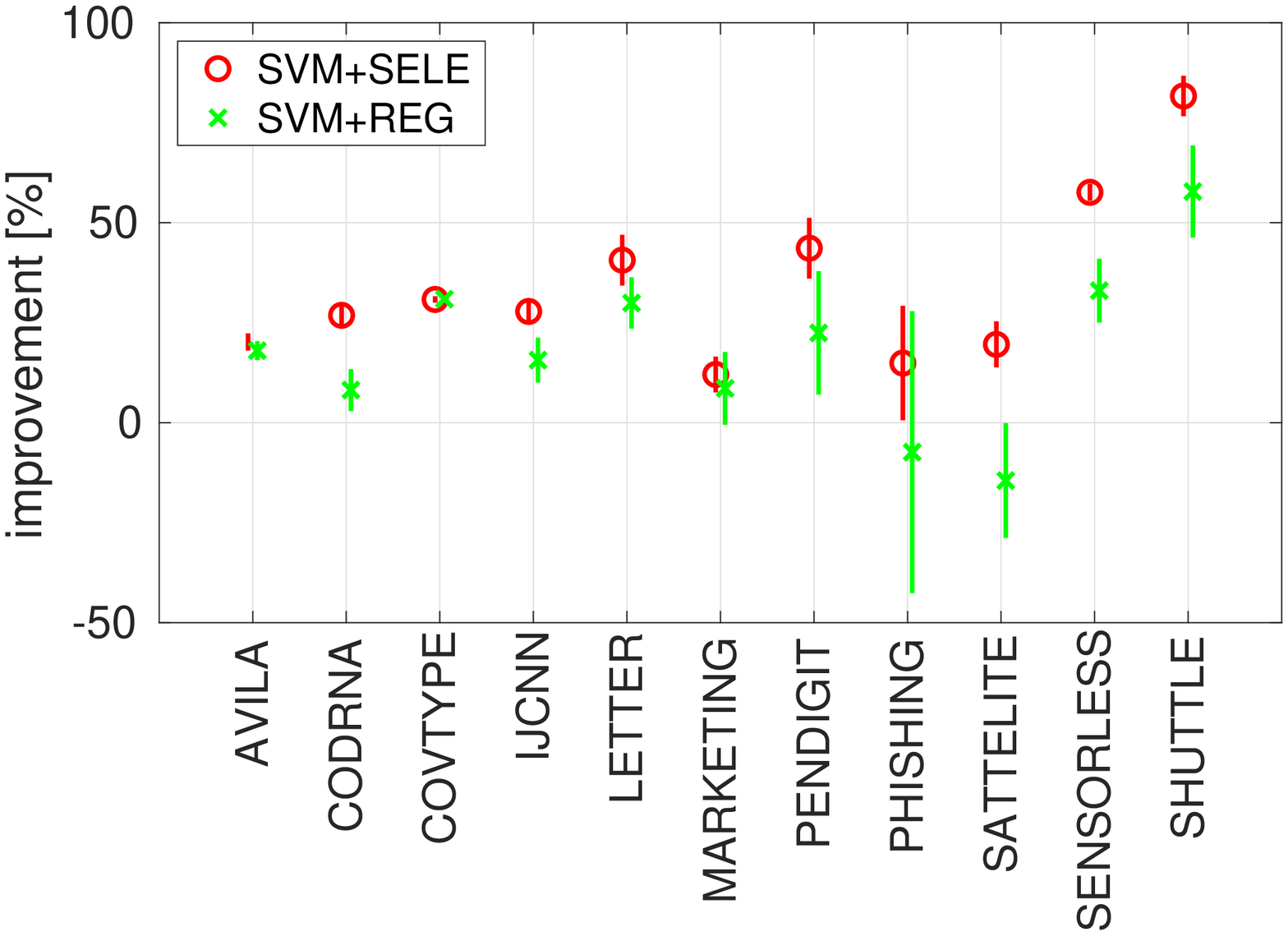}  \\
\end{tabular}
    \caption{Relative improvement gained by using the uncertainty scores (SELE, REG and TCP) that are learned from examples over the baseline scores (MCP for LR and Margin score for SVM) constructed from the classifier output. The relative improvement is computed as $100\times ({\rm AuRC}_{\rm baseline}-{\rm AuRC}_{\rm method})/{\rm AuRC}_{\rm baseline}$. 
    We show the mean and the standard deviation (error bar) of the relative improvement computed over the random 5 splits.
    }
    \label{fig:RelImprovClassif}
\end{figure}

\subsubsection{Ordinal regression}

In case of SVOR classifier we compared the baseline Margin score (sec~\ref{sec:SVOR}) and the scores learned from examples including SELE (sec~\ref{sec:SELErule}) and REG score (sec~\ref{sec:REGrule}). We used exactly the same evaluation protocol as for the classification task, however, instead of 0/1-loss the errors were evaluated by the MAE loss. The results are summarized in Table~\ref{tab:OrdRegResults}.

We again ranked the methods according to the AuRC and summarized their performance by the average rank:
\begin{itemize}
    \item  We applied the Friedman test checking whether the measured average ranks are significantly different from the mean rank. The null hypothesis, stating that the compared scores are equivalent, is rejected for p-value $0.05$, i.e., the {\em performance of compared methods is significantly different}. 
    \item We used post-hoc Nemenyi test to check for each pair whether their average ranks are significantly different. Considering $K=3$ compared methods using $N=11$ dataset yields the critical difference for p-value $0.10$ is $CD=0.98$. We conclude that {\em both SELE and REG scores are significantly better than the baseline Margin score}. The data is not sufficient to reach any conclusion about comparisons of SELE and REG. The result of the Nemenyi test is visualized in Figure~\ref{fig:OrRegResultsStatistics}(a).
\end{itemize}

The relative improvement gained by using SELE and REG scores learned from examples w.r.t. the baseline Margin score is shown in Figure~\ref{fig:OrRegResultsStatistics}(b). It is seen that the performance of the learned scores is similar and that they consistently outperform the baseline by a significant margin.

\begin{table}[h!]
  \centering
    \begin{tabular}{rcccc}
    \multicolumn{5}{c}{}\\
    & \multicolumn{1}{c}{SVOR+MARGIN} 
    & \multicolumn{1}{c}{SVOR+SELE} 
    & \multicolumn{1}{c}{SVOR+REG} & SVOR \\
    & AuRC &AuRC & AuRC & R@100 \\  
    \hline
  CALIFORNIA &  0.98$\pm$0.03  &  {\bf 0.82$\pm$0.02} &    0.84$\pm$0.02  & 1.18$\pm$0.01\\ 
       ABALONE &  1.48$\pm$0.10    &{\bf 1.19$\pm$0.09} &   1.21$\pm$0.05  &  1.54$\pm$0.02\\ 
          BANK &  1.07$\pm$0.04   & 0.99$\pm$0.04  &    {\bf 0.98$\pm$0.03} &  1.50$\pm$0.03\\ 
           CPU &  0.41$\pm$0.01    &{\bf 0.36$\pm$0.02} & 0.36$\pm$0.02  &    0.64$\pm$0.03\\ 
     BIKESHARE &  1.60$\pm$0.07    & {\bf 1.25$\pm$0.01} & 1.27$\pm$0.01  &    1.70$\pm$0.03\\ 
          CCPP &  0.46$\pm$0.02  & {\bf 0.41$\pm$0.02} &  0.42$\pm$0.02  &   0.58$\pm$0.02\\ 
      FACEBOOK &  0.51$\pm$0.01  & 0.37$\pm$0.01  & {\bf 0.36$\pm$0.01} &   1.11$\pm$0.01\\ 
           GPU&  1.43$\pm$0.02  &  {\bf 0.85$\pm$0.03} &  0.86$\pm$0.03  &   1.49$\pm$0.02\\ 
         METRO &  2.20$\pm$0.07  & {\bf 1.97$\pm$0.01} &   1.98$\pm$0.03  &   2.37$\pm$0.03\\ 
           MSD &  6.23$\pm$0.07  & 4.26$\pm$0.03  &    {\bf 4.25$\pm$0.03} &   6.22$\pm$0.03\\ 
  SUPERCONDUCT &  0.98$\pm$0.02  & {\bf 0.75$\pm$0.01} &  0.77$\pm$0.01  & 1.07$\pm$0.01\\ 
           \hline
    average rank & 3.00     &       1.27      &       1.73      & 
    \end{tabular} 
      \caption{Performance of the uncertainty scores on 11 ordinal regression problems. The scores are constructed on top of the SVOR classifier. For each score we show the mean and the standard deviation of the test AuRC computed over 5 random splits. We compare the performance of scores learned from examples (SEL, REG) and the baseline Margin score derived from the SVOR classifier output. The last column shows the risk of the base (non-selective) classifier. All the values correspond to the Mean Absolute Error (MAE). The best results for each dataset are shown in bold. The last row shows the average rank.}
    \label{tab:OrdRegResults}
\end{table}

\begin{figure}[h!]
    \begin{tabular}{cc}
        a) Average rank for scores on top of SVOR. &
        b) Relative improvement over \\
        &  SVOR+Margin score. \\
        \includegraphics[width=0.45\columnwidth]{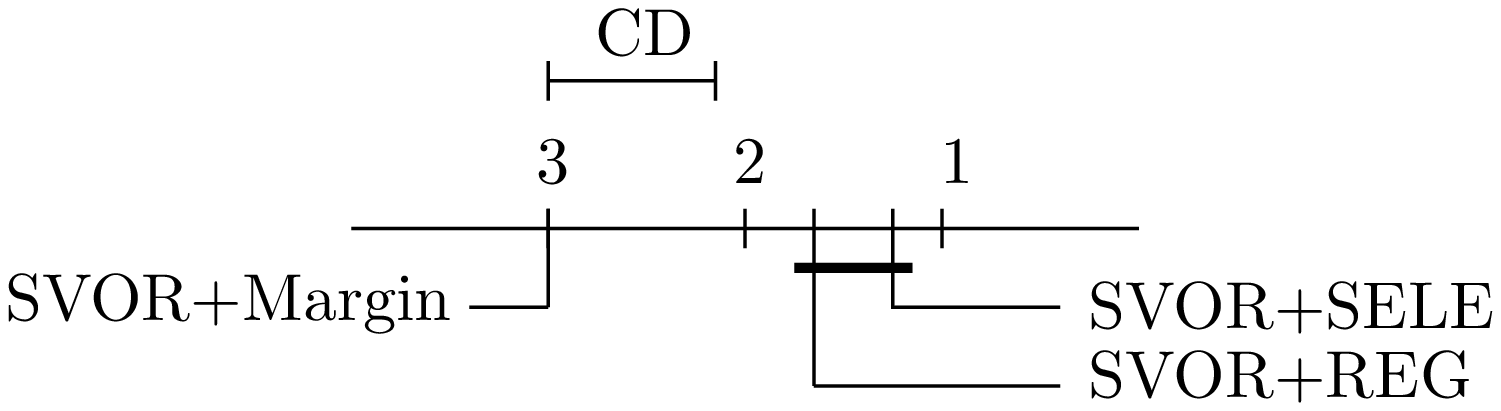} &
        \includegraphics[width=0.45\columnwidth]{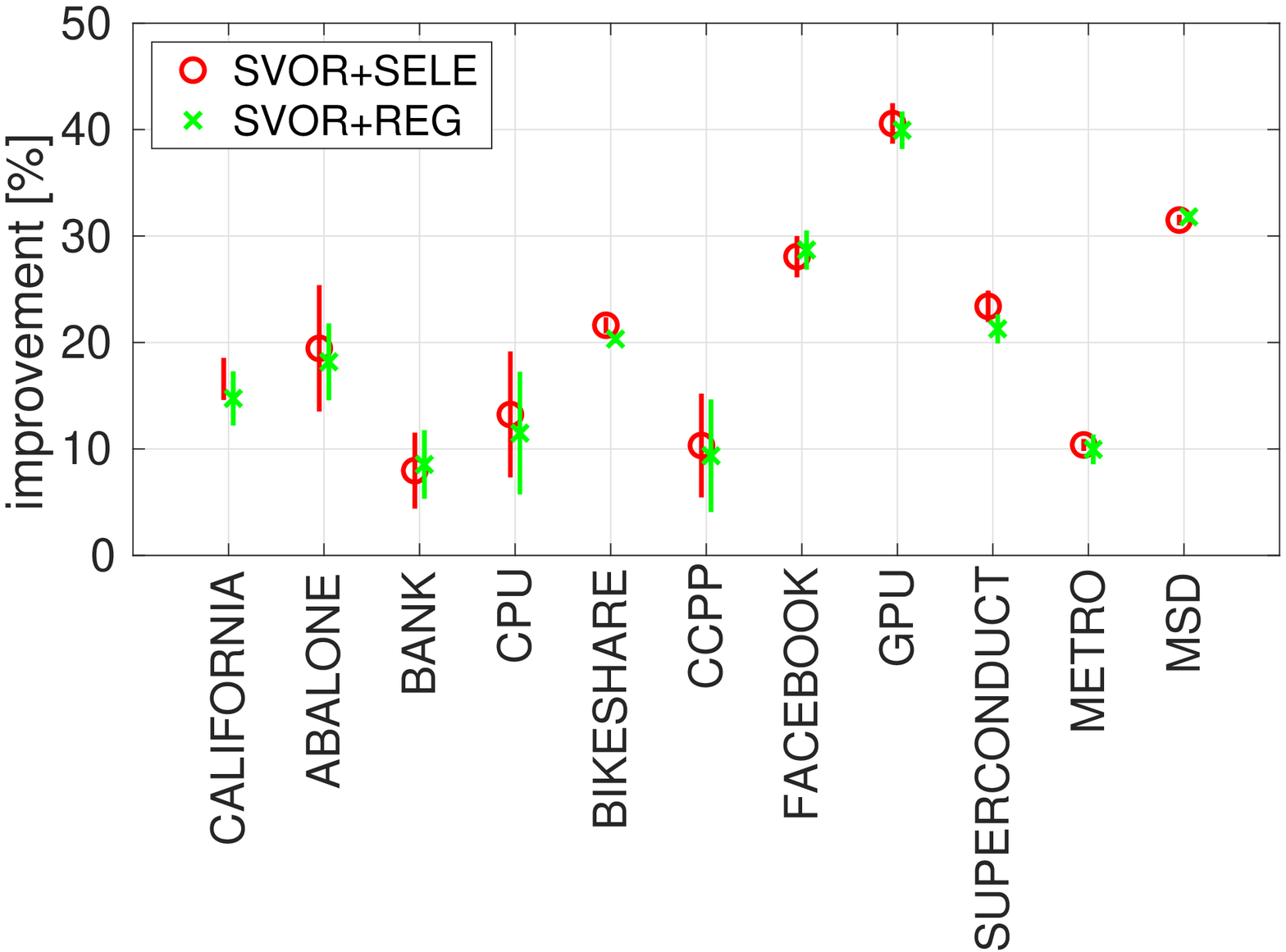} 
    \end{tabular}
    \caption{Statistics derived from results obtained on 11 ordinal regression problems. Figure (a) shows pair-wise comparison of uncertainty scores with the Nemenyi test. The figure shows the average ranks for each score and the critical distance (CD). Groups of scores that are not significantly different at $p$-value $0.10$ are connected. 
    Figure (b) shows relative improvement gained by using the SELE and REG uncertainty scores learned from examples over the baseline Margin score. 
    }
    \label{fig:OrRegResultsStatistics}
\end{figure}

\subsubsection{Structured Output Classification}
We trained SELE score (sec~\ref{sec:SELErule}) and REG score (sec~\ref{sec:REGrule}) on top of the DLIB detector and compared them with the baseline which uses the DLIB face detector score (sec~\ref{sec:DLIB}) as an uncertainty measure. The Risk-Coverage curves of the three methods and their corresponding AuRC are shown in Figure~\ref{fig:dlibResults}(a). Both the learned scores, SELE and REG, are significantly better than the baseline face detector score. The SELE is slightly better than REG score. The largest difference between the three scores are seen for low values of coverage where SELE most outperforms the other two methods. High selective risk for low values of coverage means that faces with very bad landmark predictions are assigned the lowest uncertainty scores. SELE score does not suffer from this problem. This can be seen in Figure~\ref{fig:DlibExamples} where we show examples of 10 test faces with the lowest uncertainty and the highest uncertainty predicted by SELE. 

Unlike the experiments in the previous section, the number of parameters to be learned ($m=2,448$) relative to the number of training examples ($n=3,484$) is much higher. To see whether the number of examples is sufficient we trained SELE and REG scores from increasingly bigger training set. Figure~\ref{fig:dlibResults}(b) shows the test AuRC as a function of the the number of training examples. It is seen that AuRC of the SELE is not yet saturated and would most likely converge to a significantly lower value relative to the REG score provided 300-W dataset had more training examples. 

\begin{figure}[h!]
    \centering
    \begin{tabular}{cc}
    (a) Risk-Coverage curve& (b) Learning curve\\
    \includegraphics[width=0.48\columnwidth]{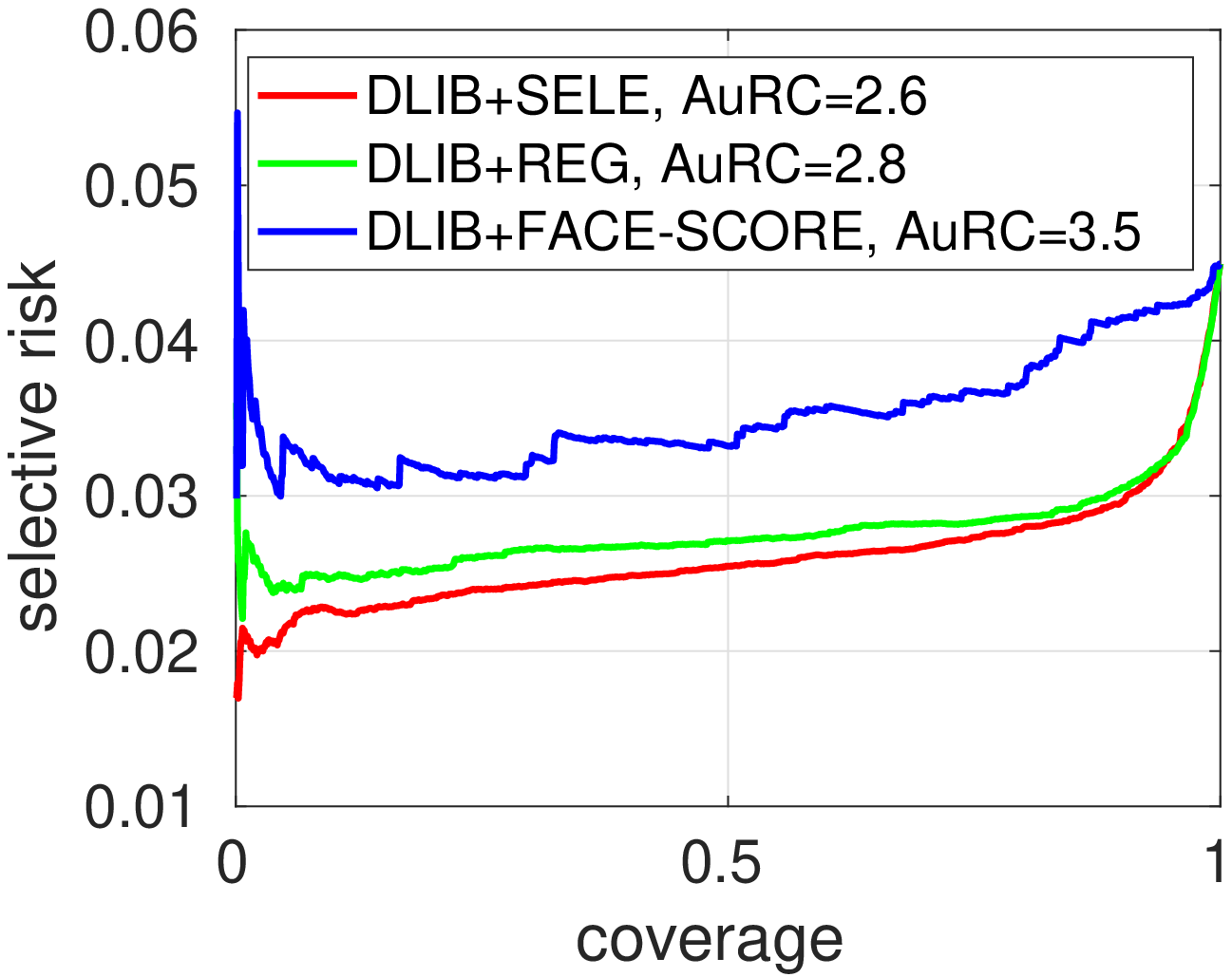}&
    \includegraphics[width=0.465\columnwidth]{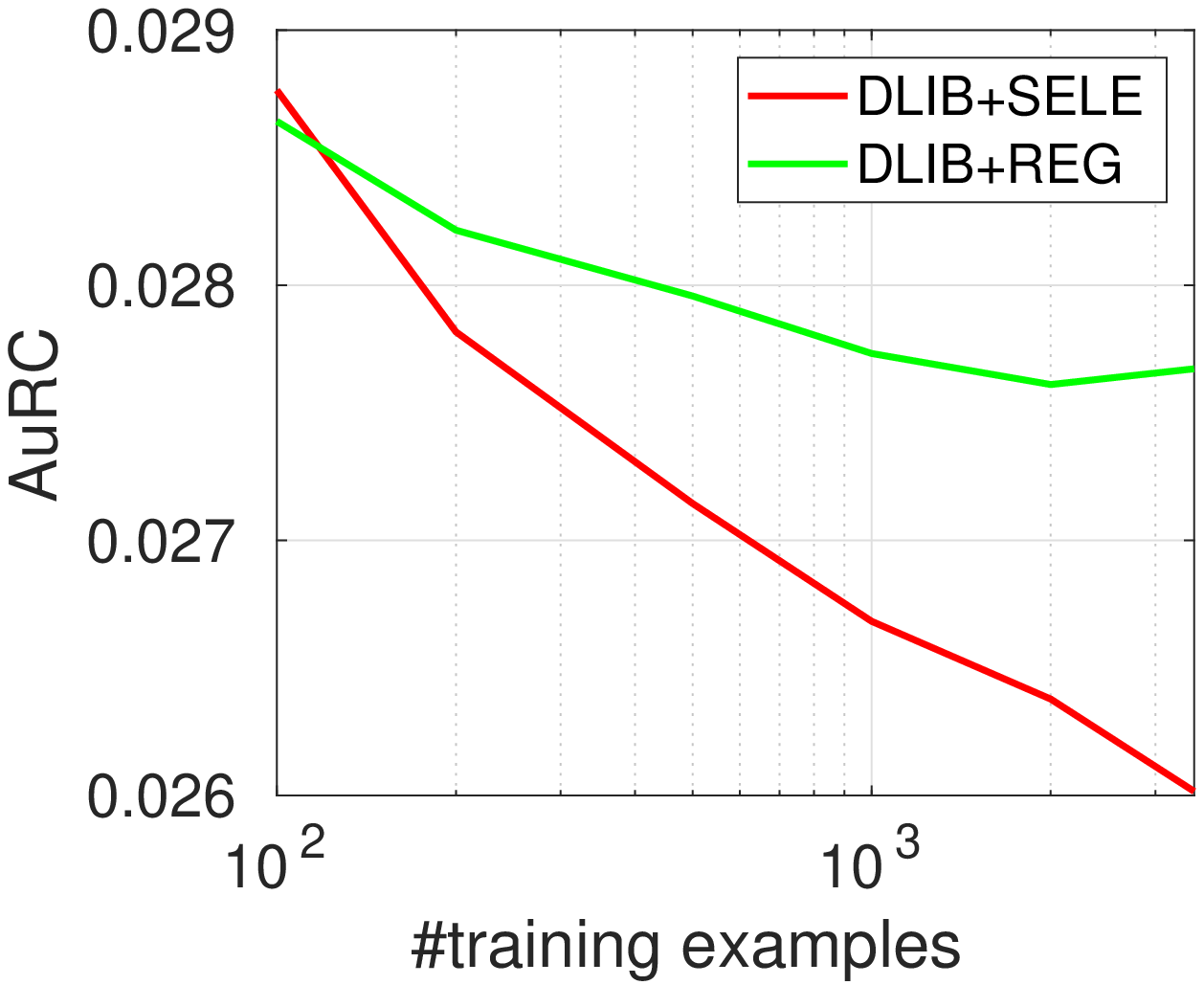}
    \end{tabular}
    \caption{Evaluation of uncertainty scores on top of DLIB landmark detector using the 300-W benchmark. Figure (a) shows the RC curve and AuRC of computed from the predictions of DLIB detector endowed with the three compared uncertainty scores: the proposed SELE and REG scores and the DLIB face detector score used as a baseline. Figure (b) shows the test AuRC for SELE and REG scores as a function of the number of training examples.}
    \label{fig:dlibResults}
\end{figure}

\begin{figure}[h!]
    \centering
    \begin{tabular}{ccccc}
        \includegraphics[width=0.22\columnwidth]{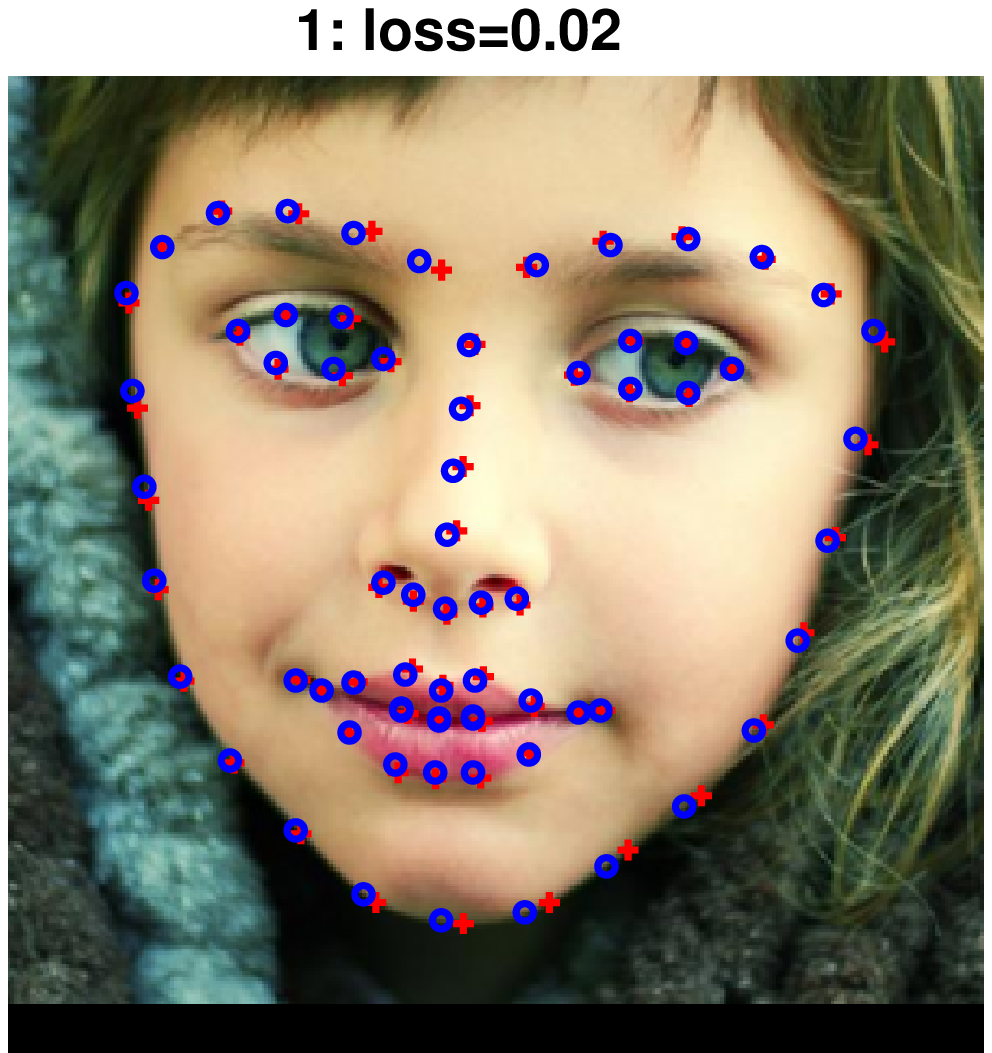}&
        \includegraphics[width=0.22\columnwidth]{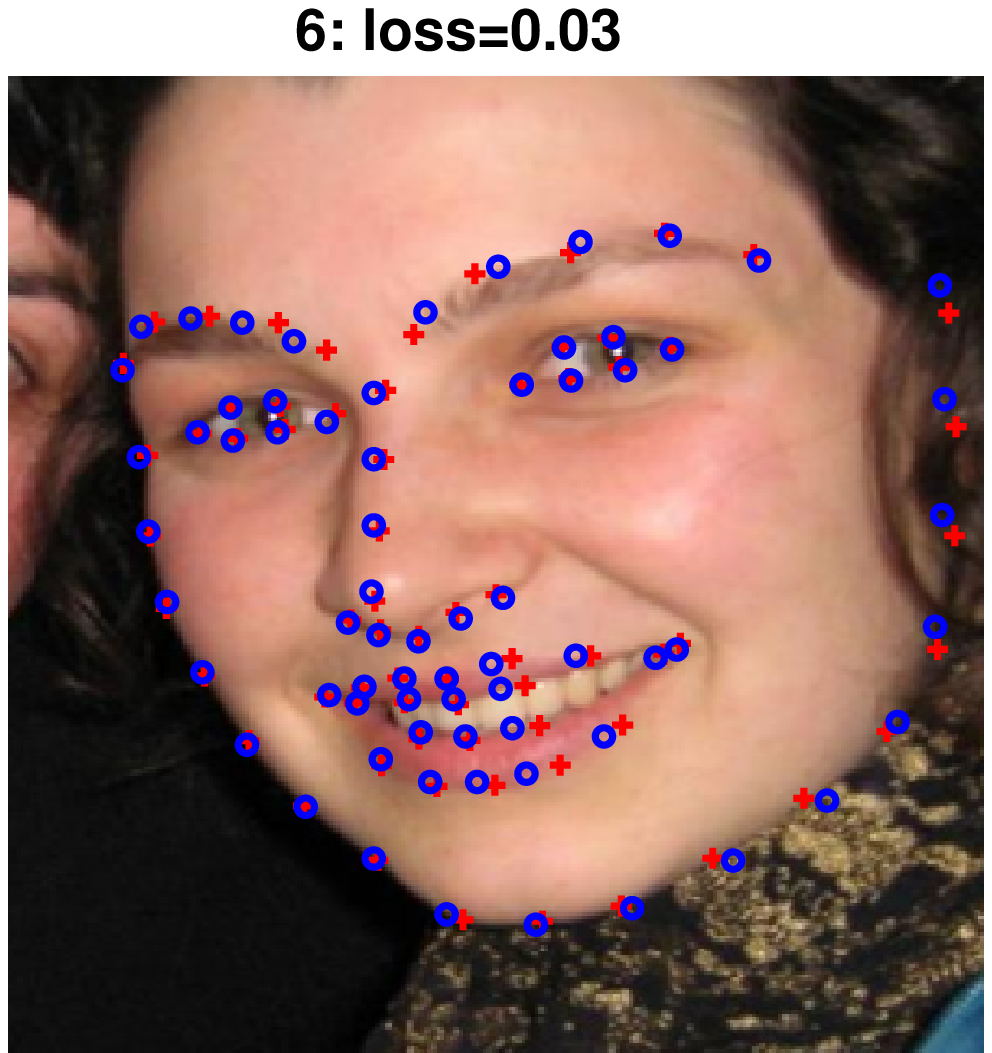}&
        \includegraphics[width=0.22\columnwidth]{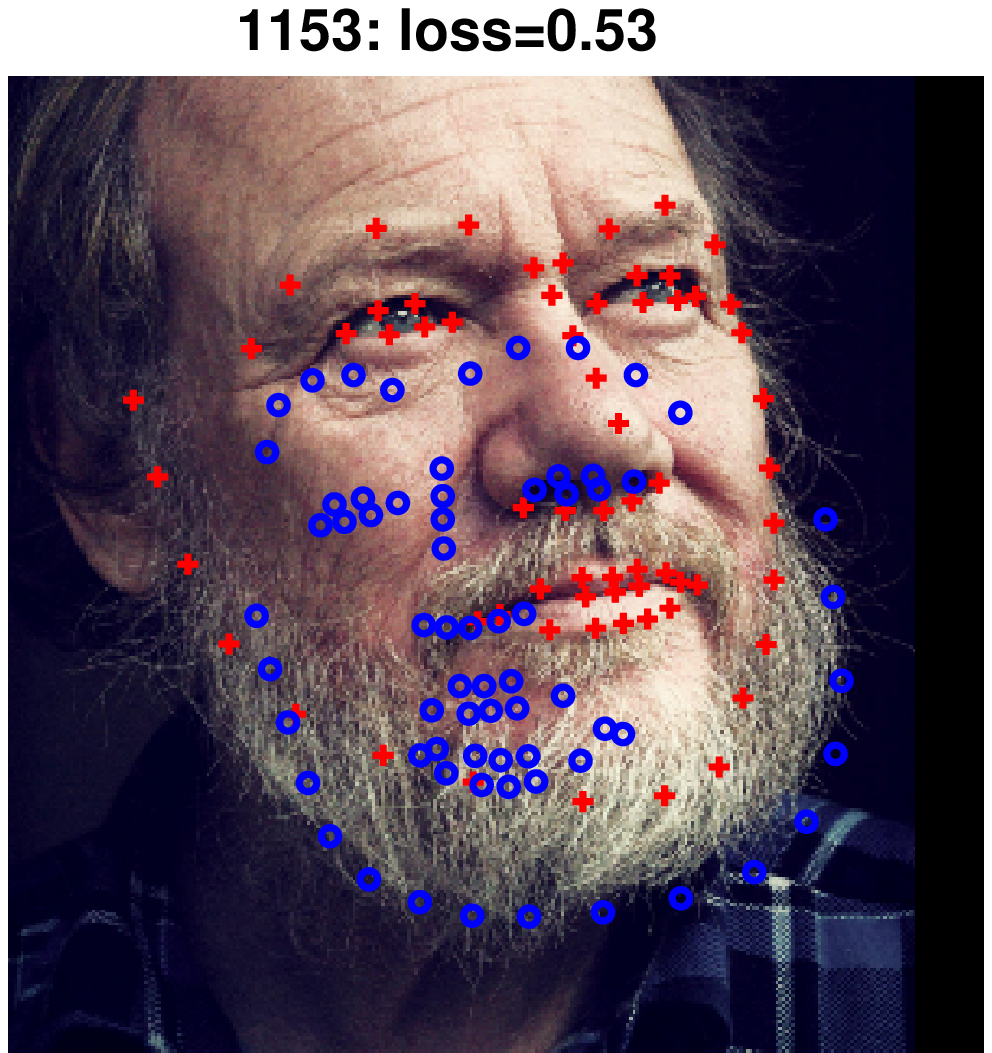}&
        \includegraphics[width=0.22\columnwidth]{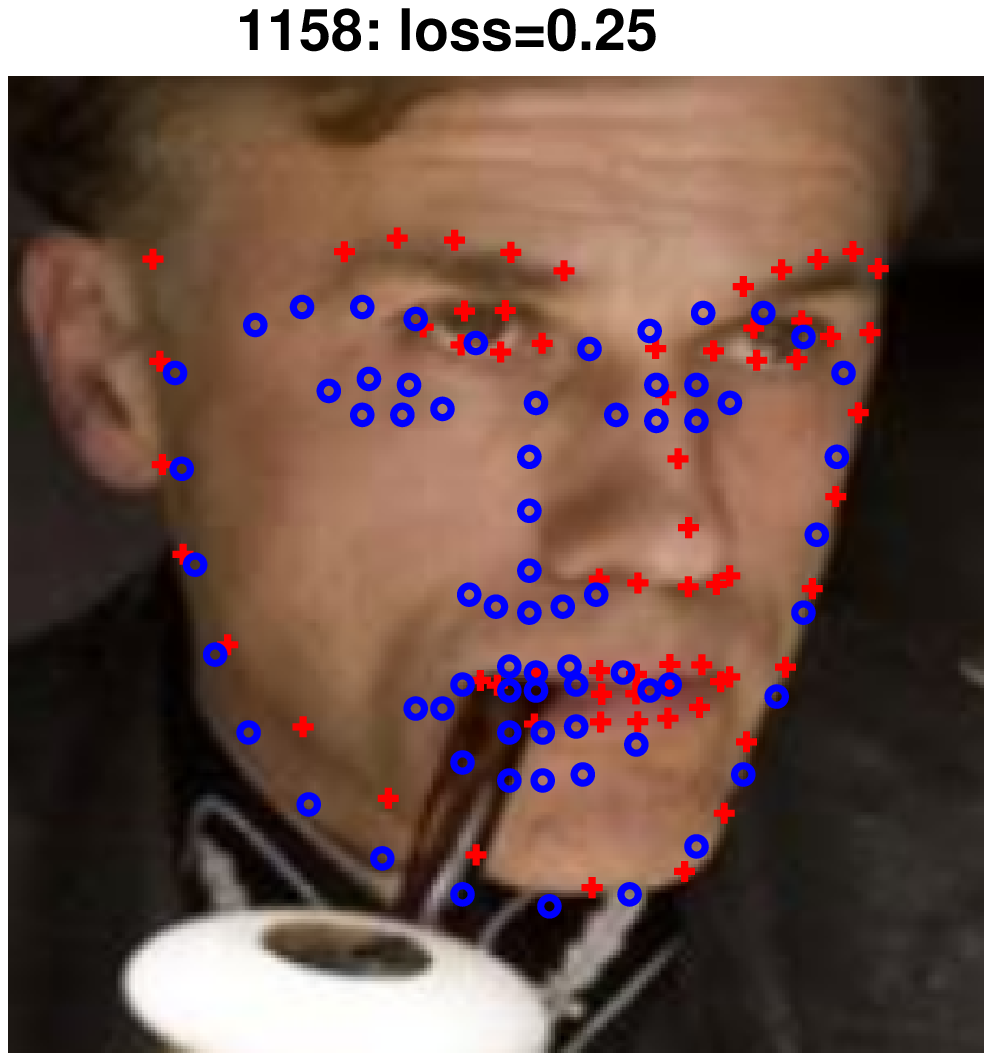}\\
        \includegraphics[width=0.22\columnwidth]{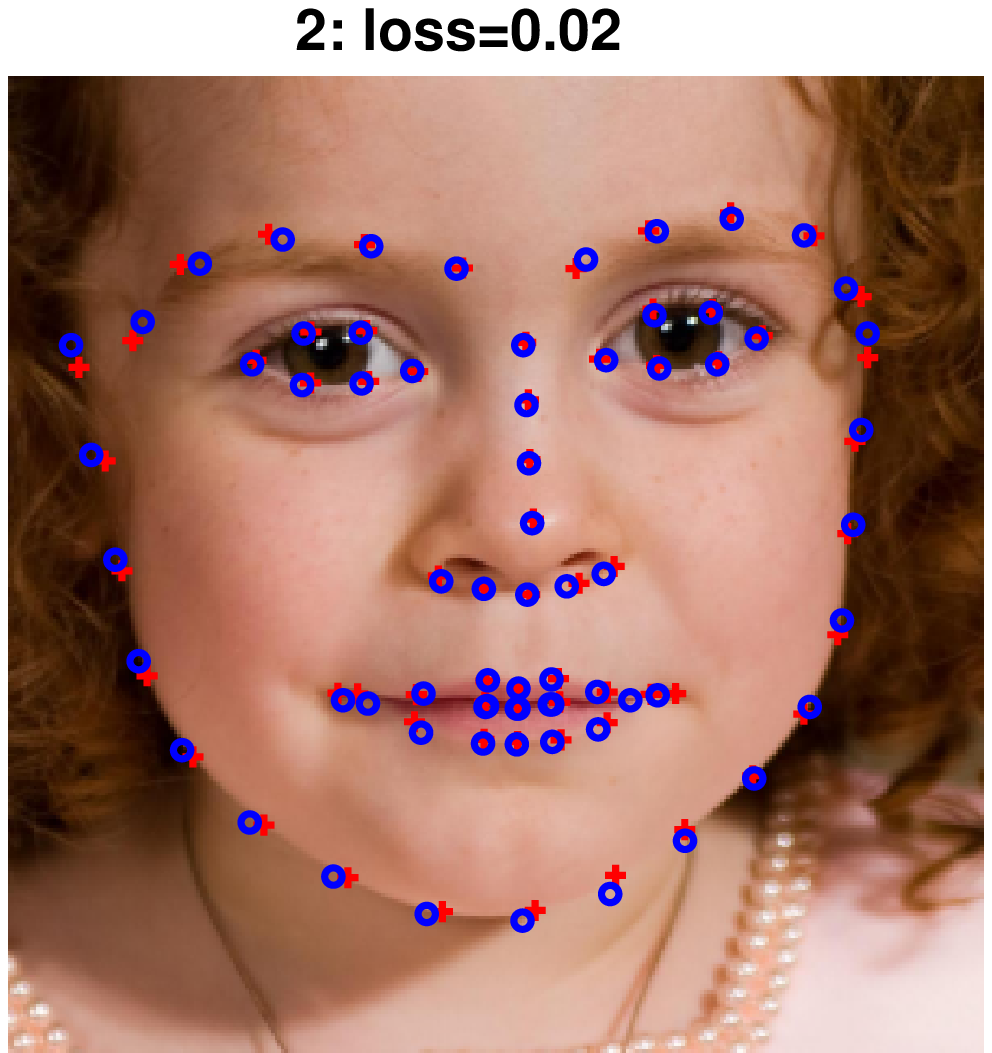}&
        \includegraphics[width=0.22\columnwidth]{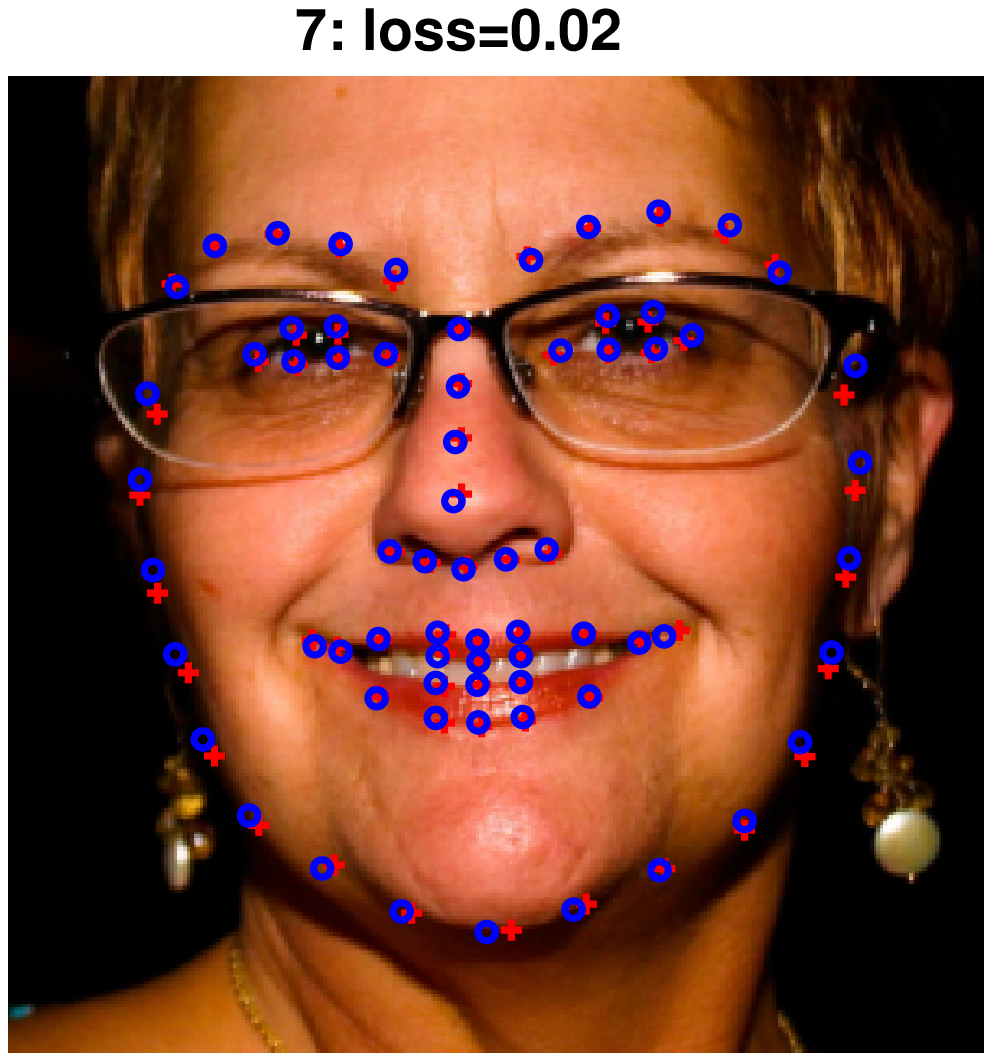}&
        \includegraphics[width=0.22\columnwidth]{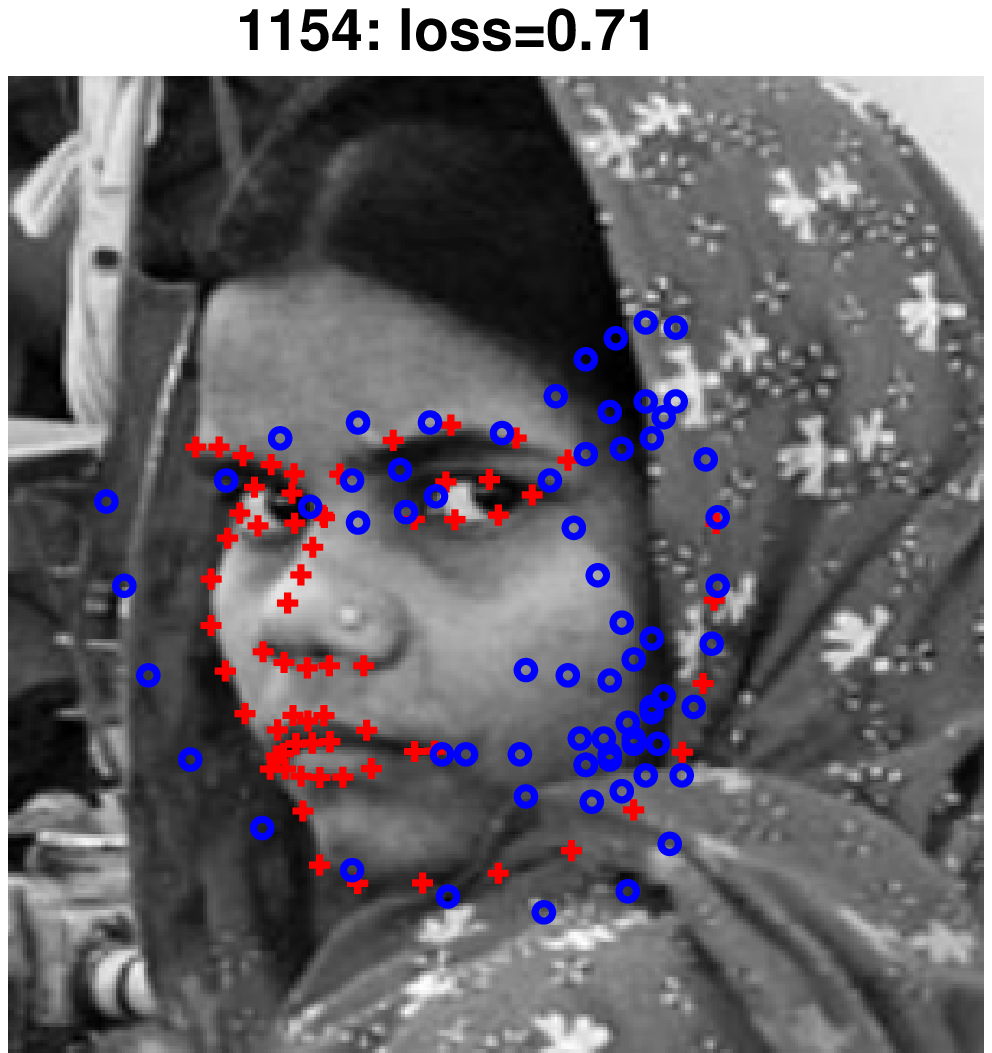}&
        \includegraphics[width=0.22\columnwidth]{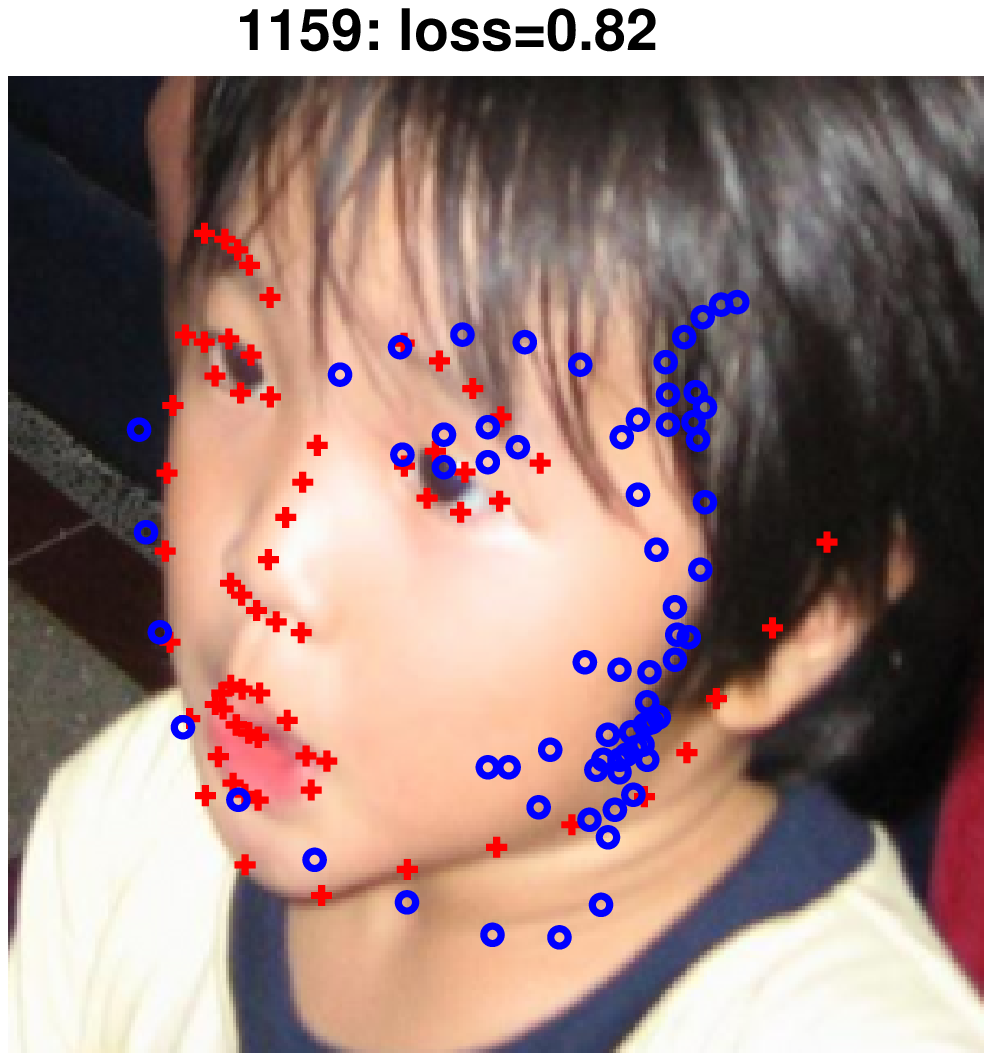}\\
        \includegraphics[width=0.22\columnwidth]{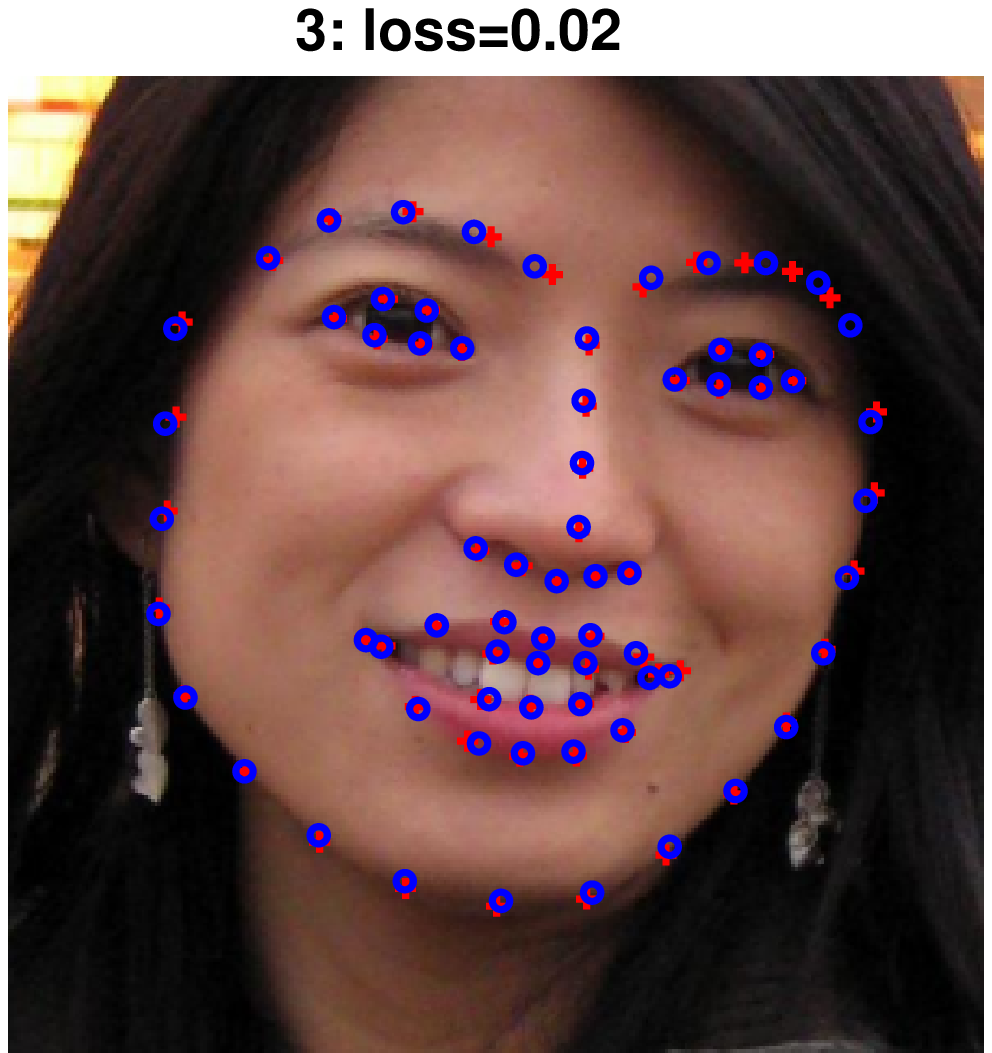}&
        \includegraphics[width=0.22\columnwidth]{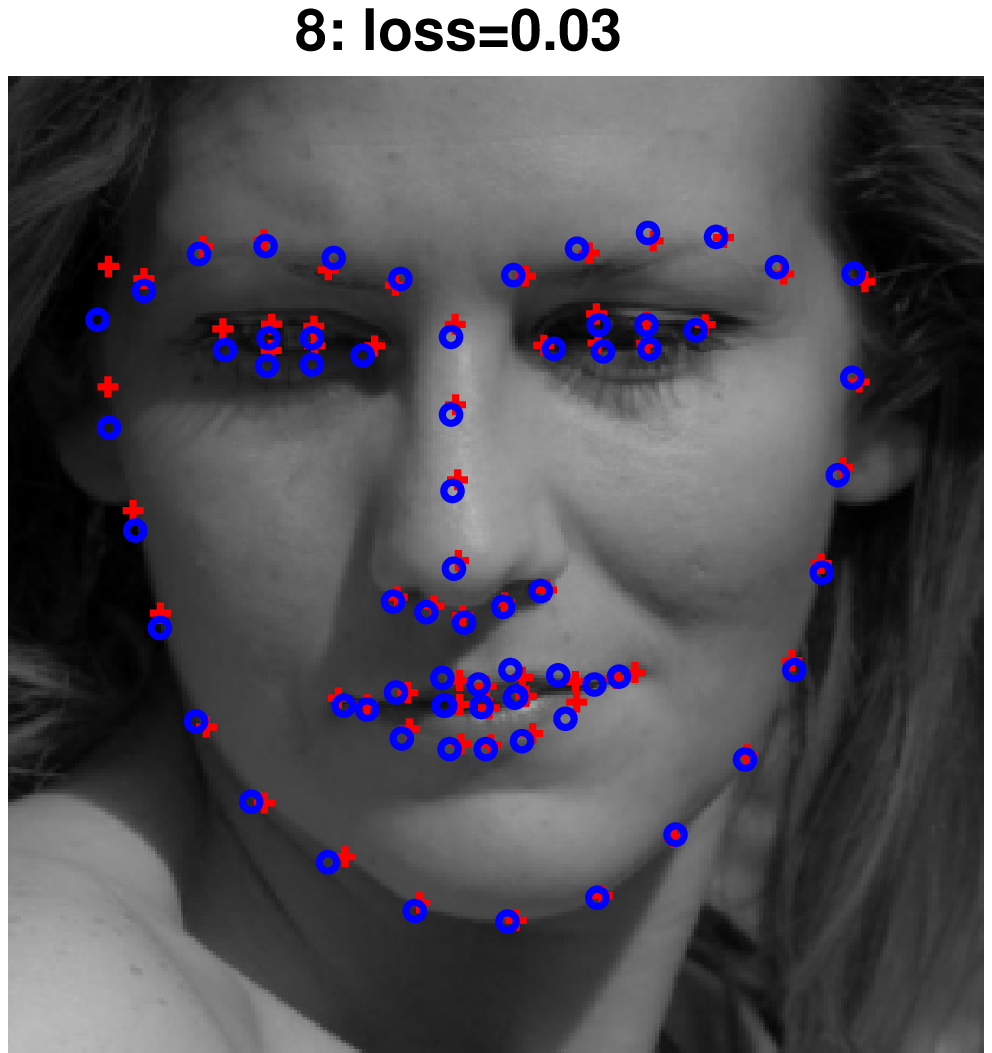}&
        \includegraphics[width=0.22\columnwidth]{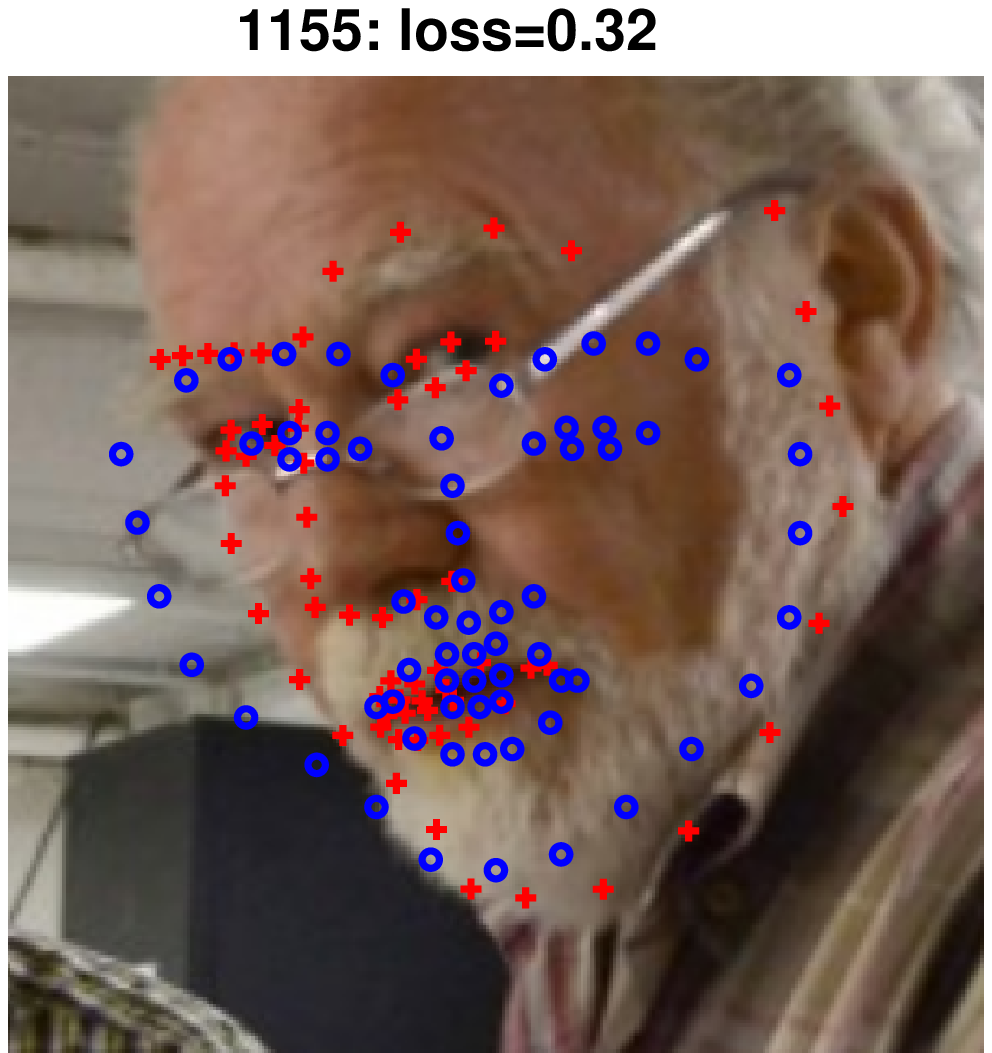}&
        \includegraphics[width=0.22\columnwidth]{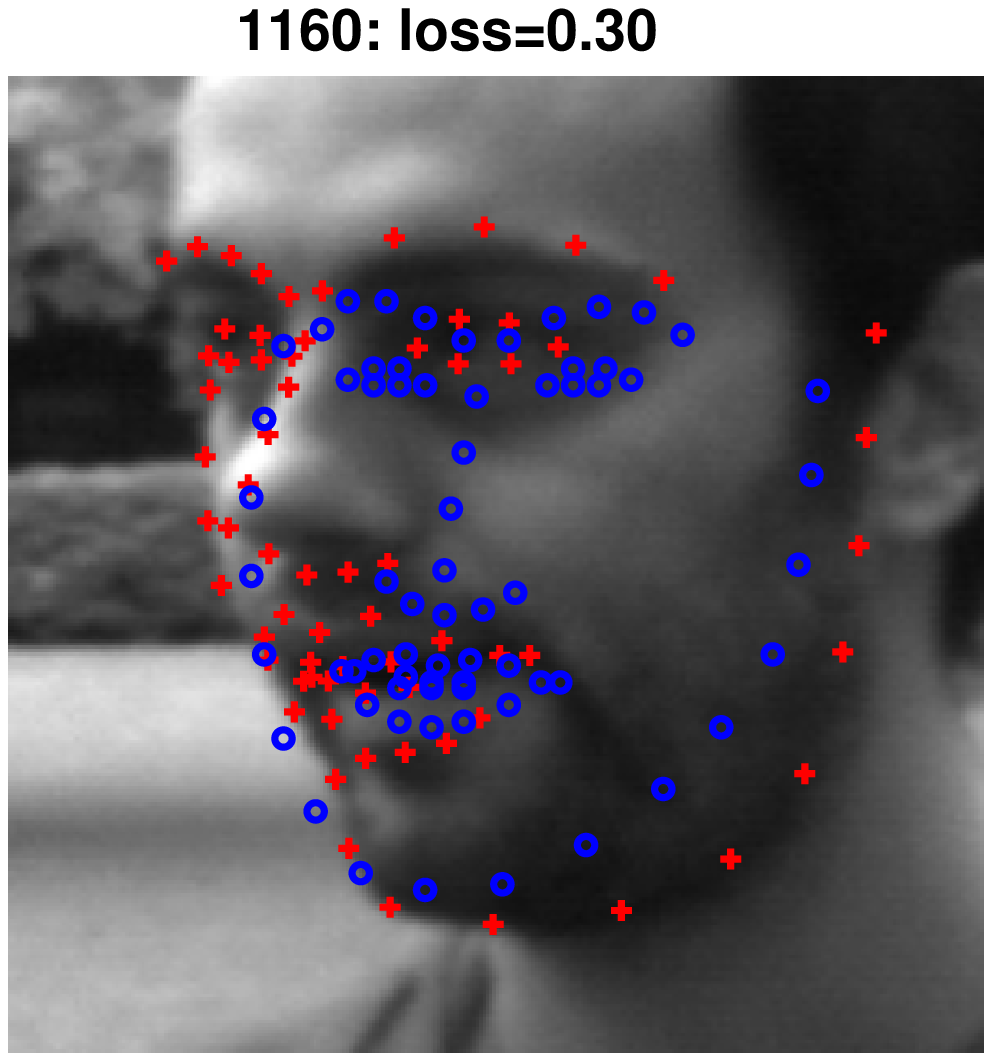}\\
        \includegraphics[width=0.22\columnwidth]{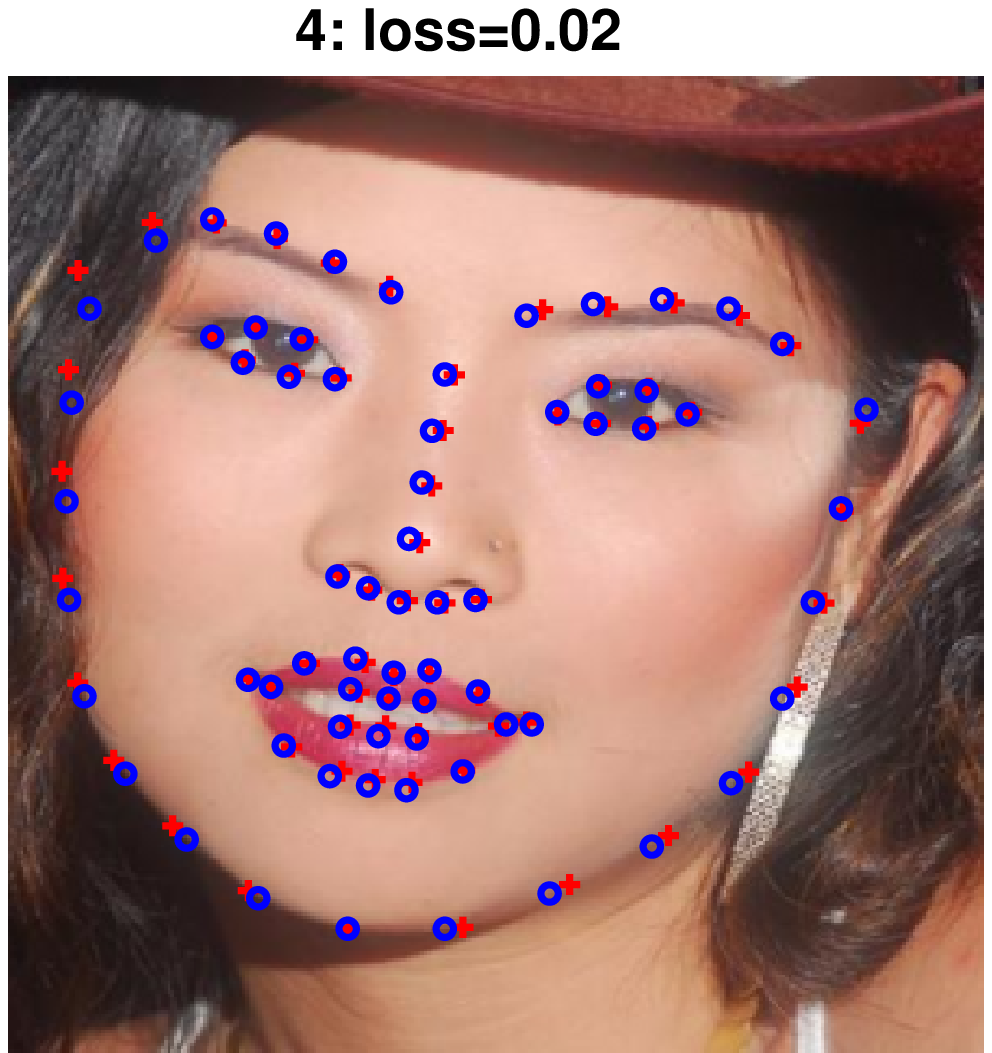}&
        \includegraphics[width=0.22\columnwidth]{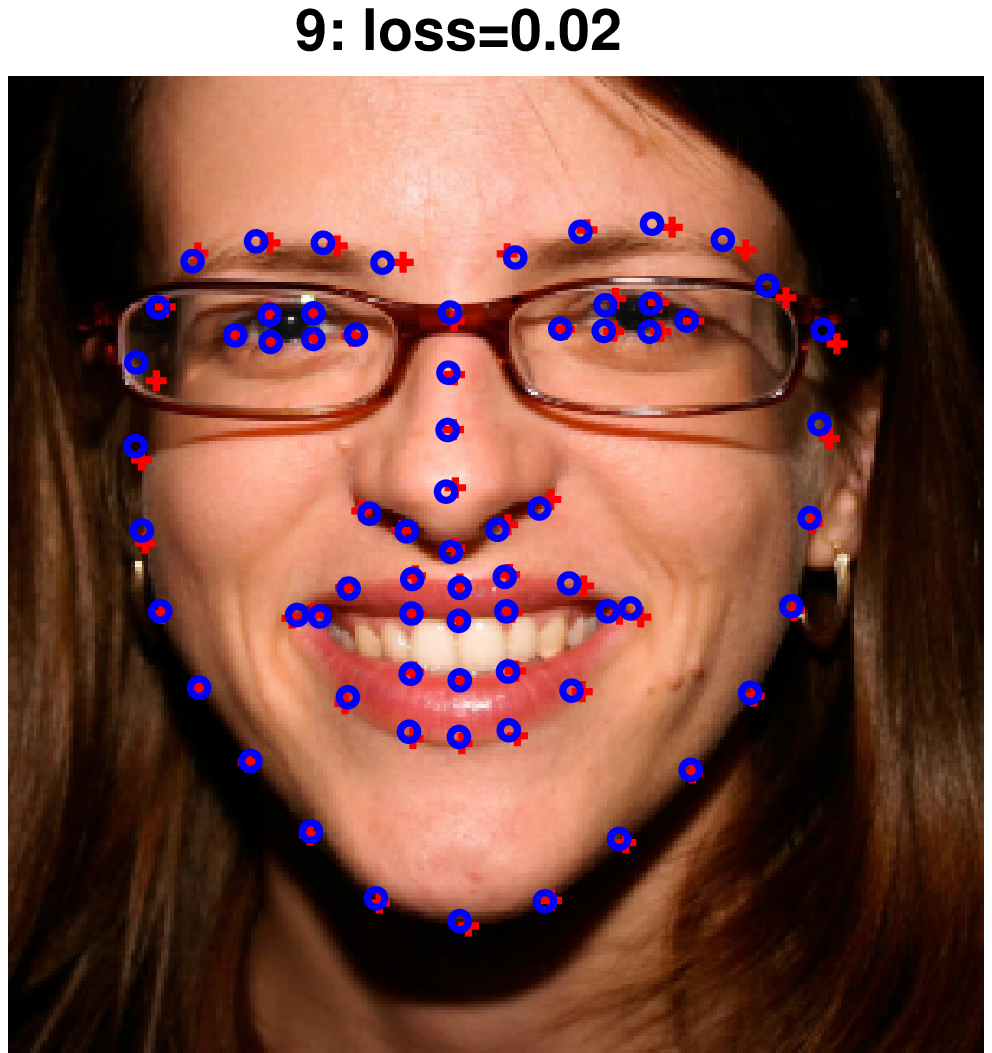}&
        \includegraphics[width=0.22\columnwidth]{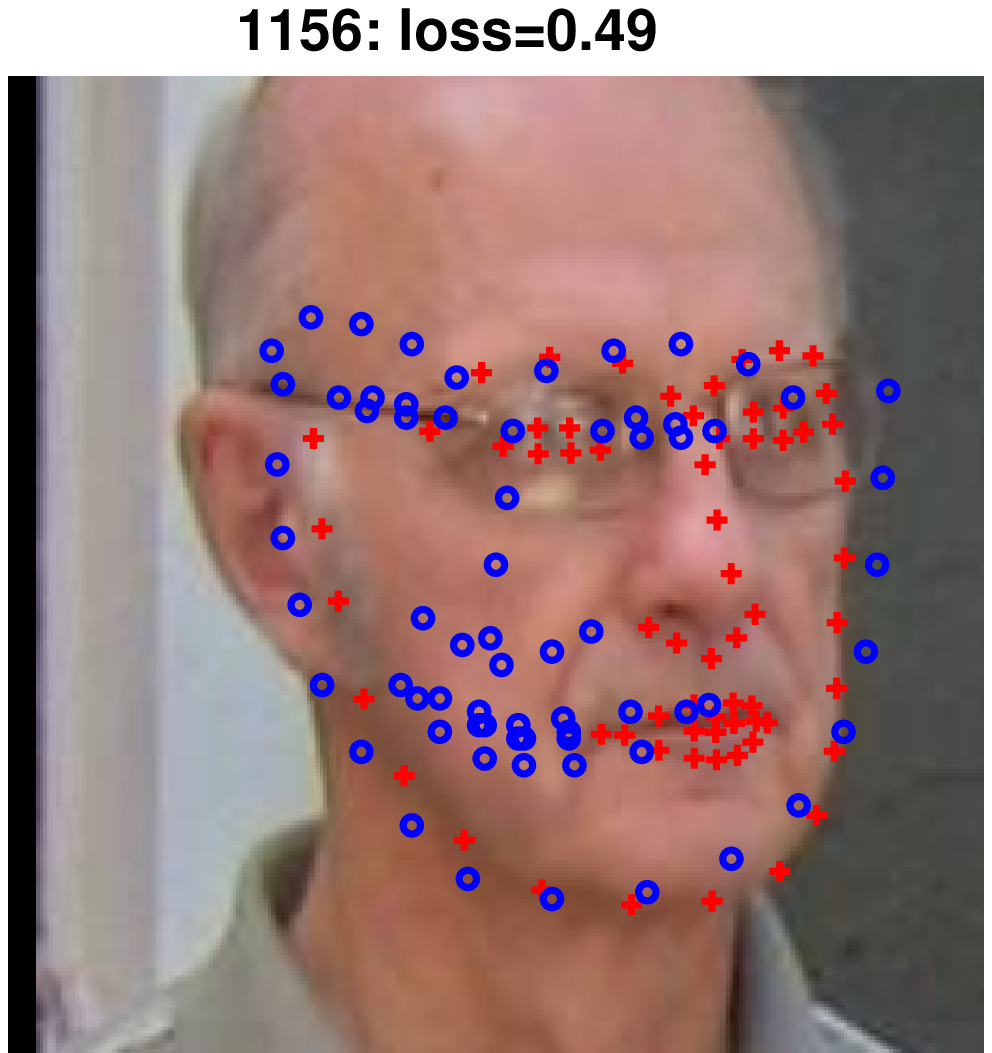}&
        \includegraphics[width=0.22\columnwidth]{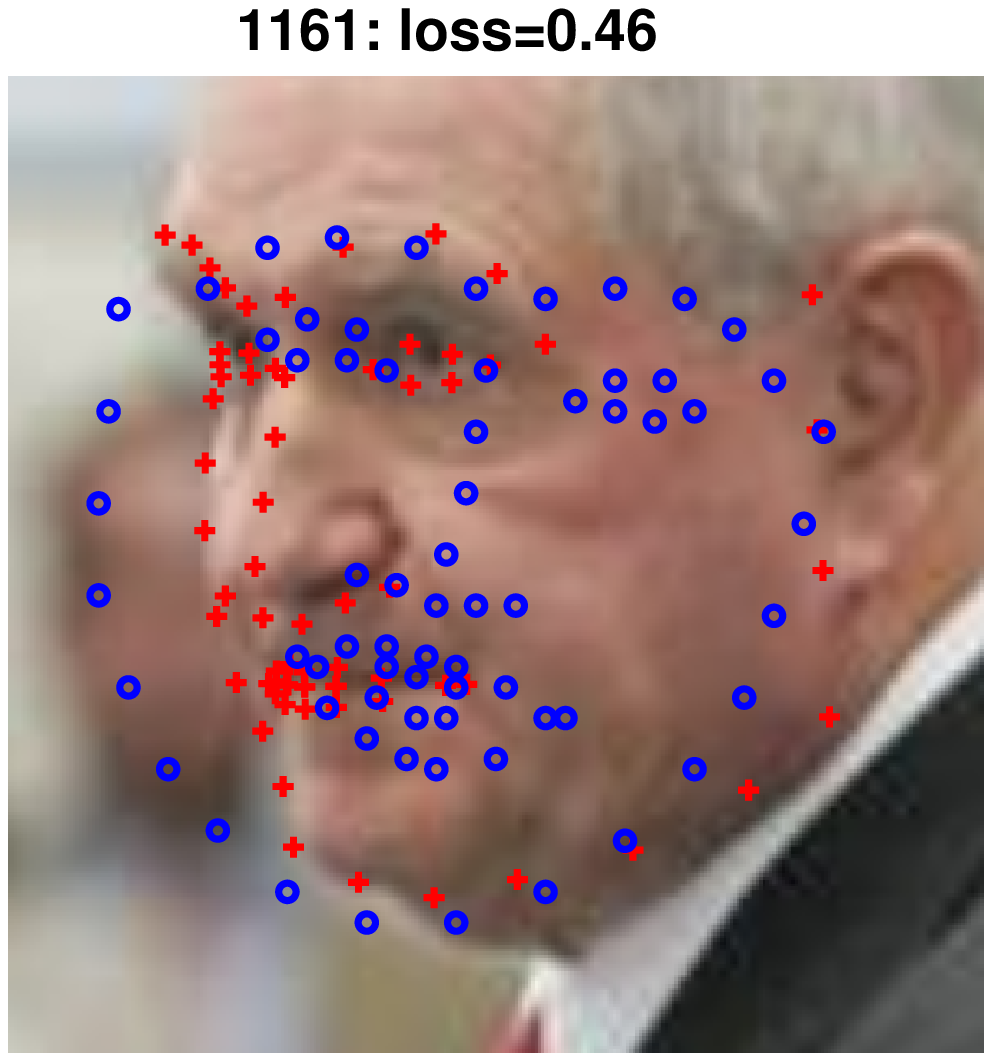}\\
        \includegraphics[width=0.22\columnwidth]{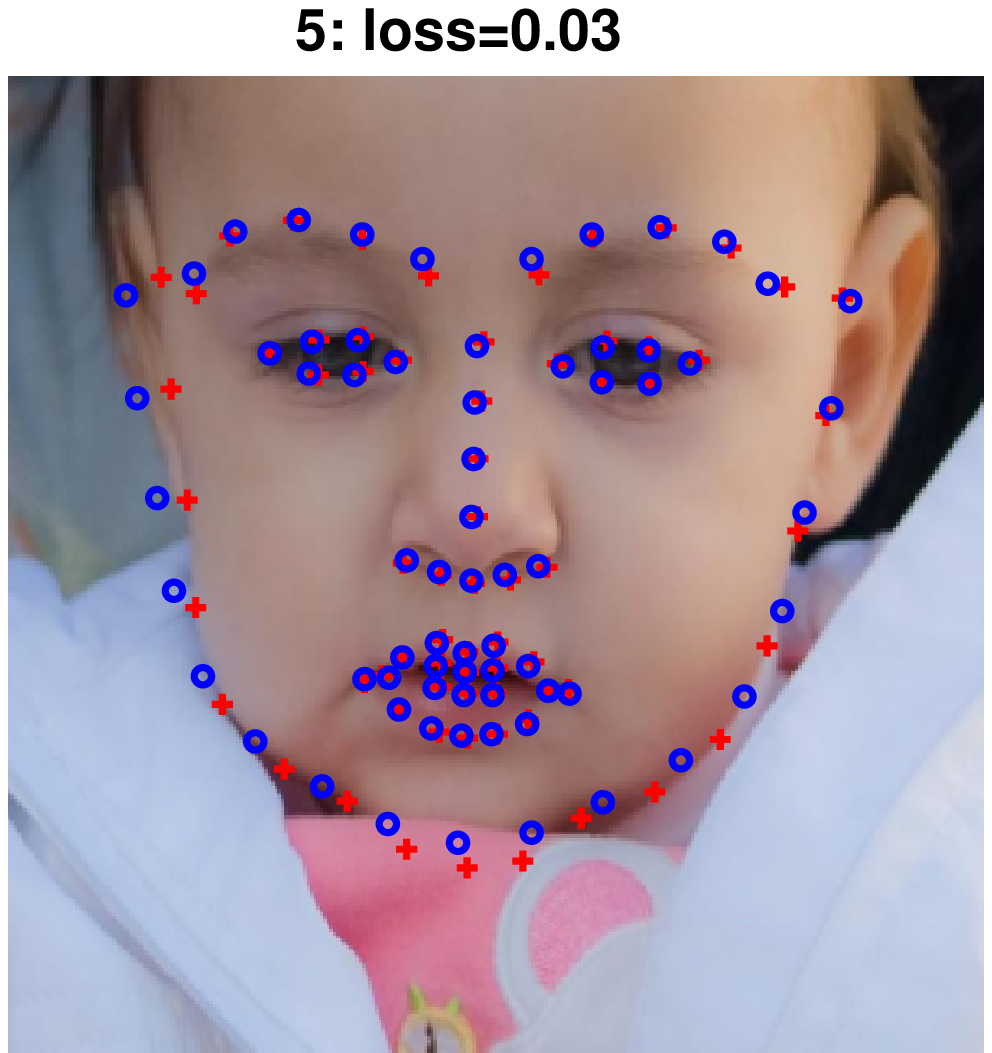}&
        \includegraphics[width=0.22\columnwidth]{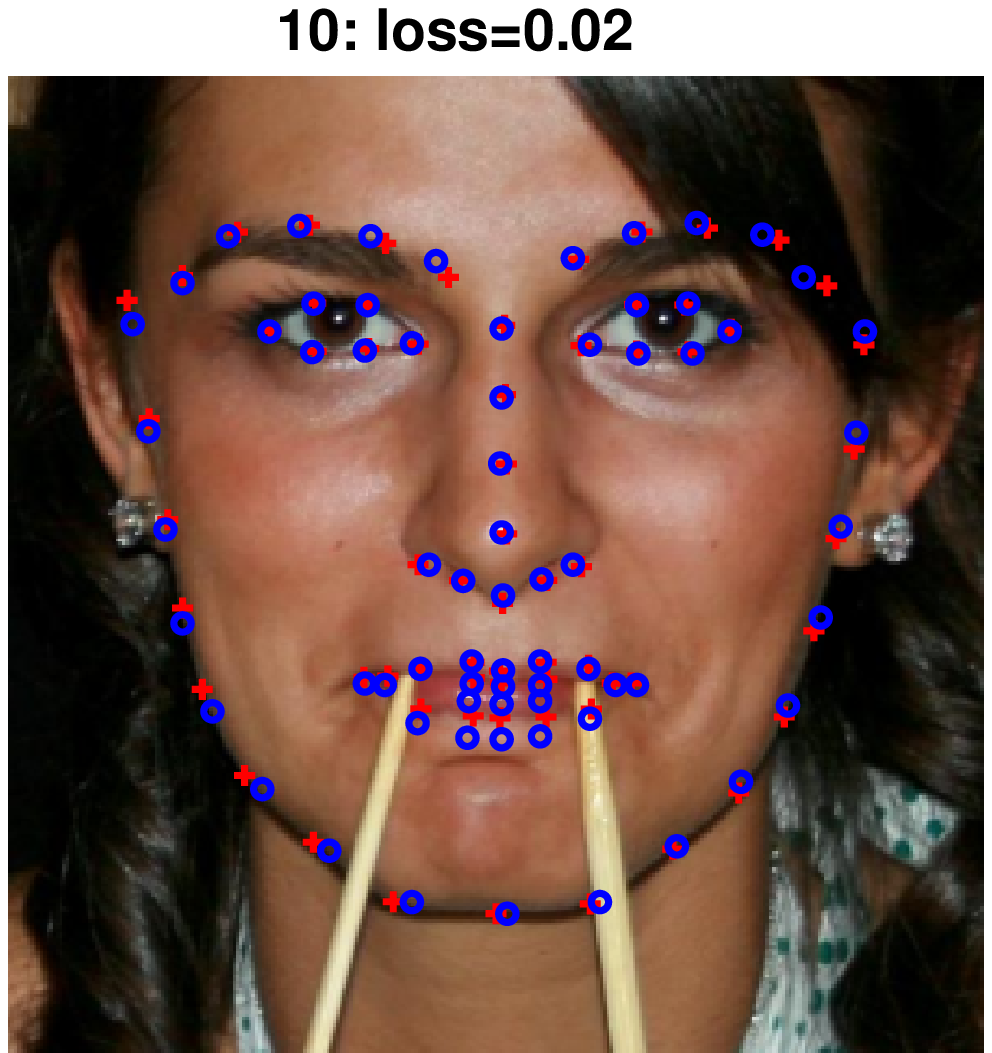}&
        \includegraphics[width=0.22\columnwidth]{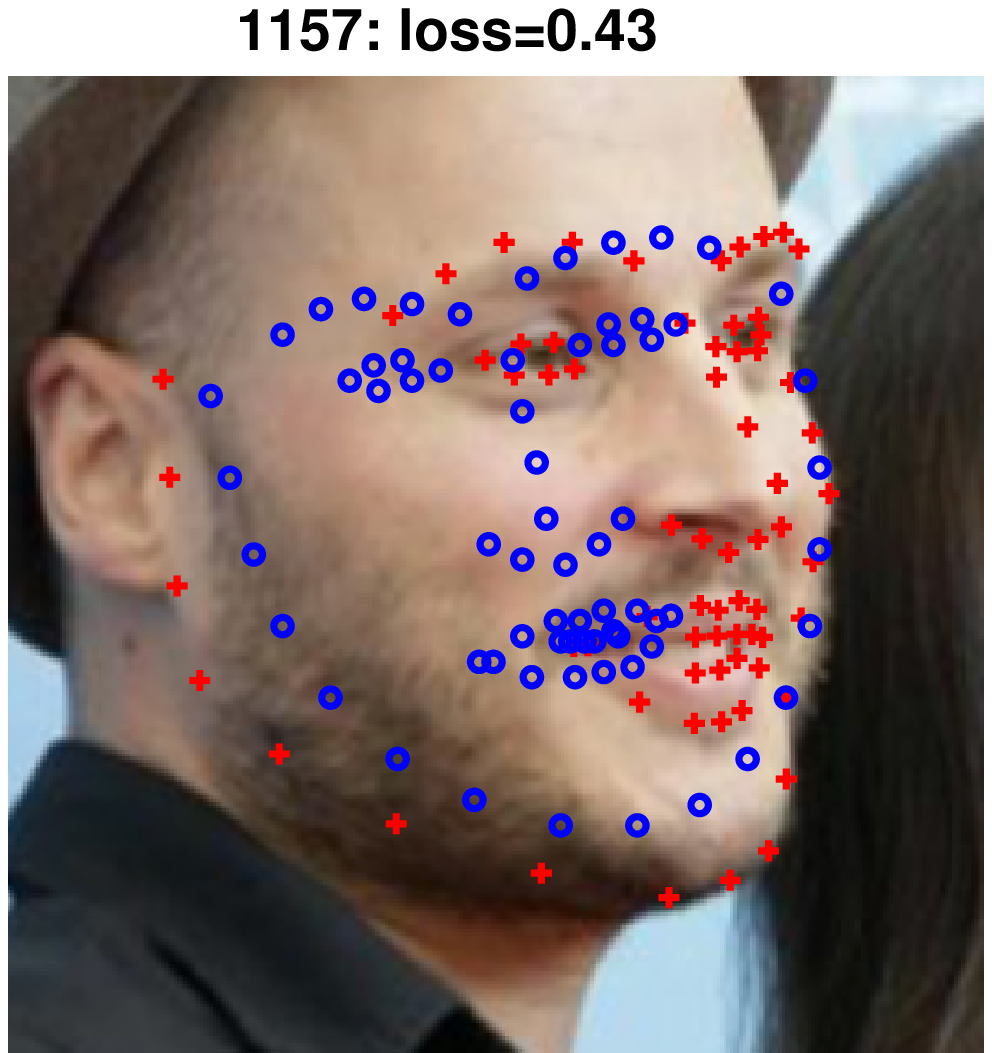}&
        \includegraphics[width=0.22\columnwidth]{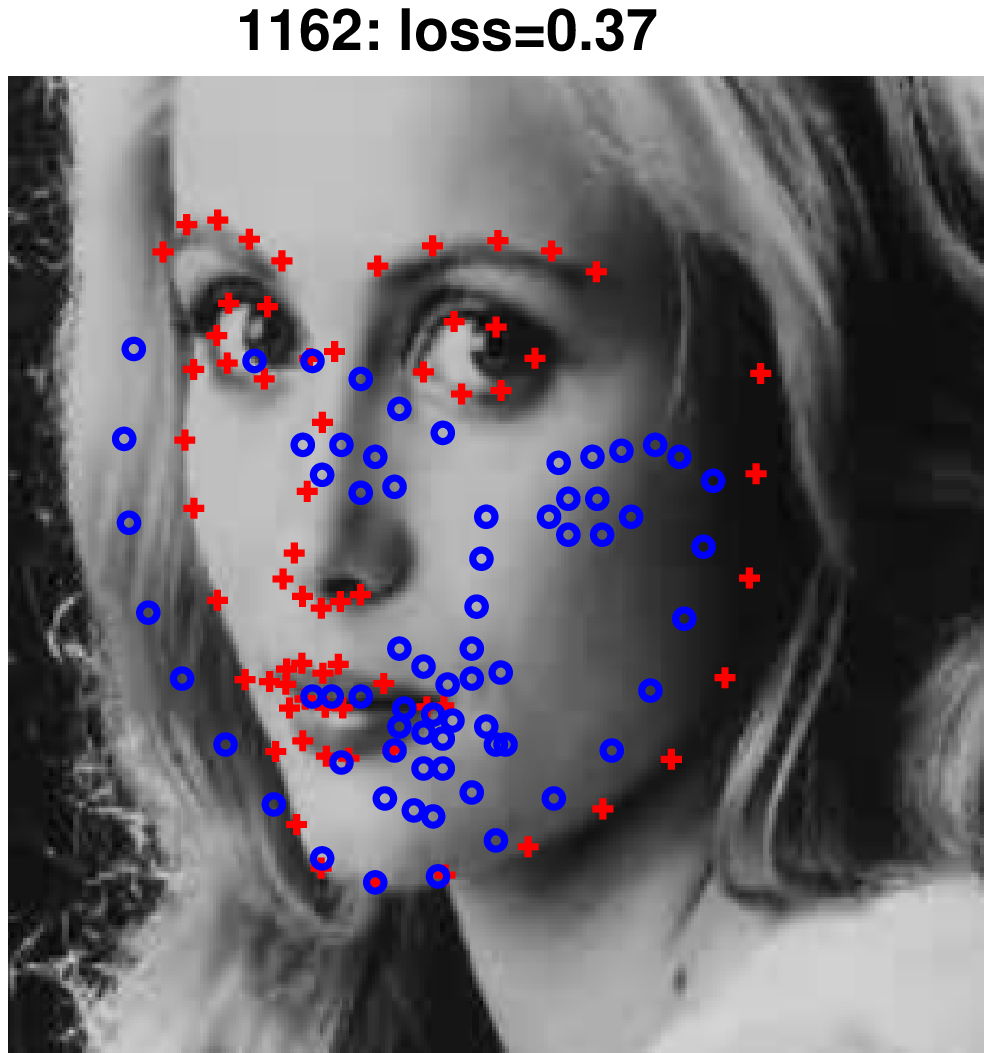}\\
    \end{tabular}
    \caption{Figure shows examples of 10 test faces from 300-W database with the lowest and 10 faces with the highest value of the SELE uncertainty score. The ground-truth landmark positions (red) and the landmark positions predicted by DLIB detector (blue) are superimposed into the image. The image title shows the rank induced by the SELE score and the normalized localization error which is used as the classification loss in this application. }
    \label{fig:DlibExamples}
\end{figure}
\section{Conclusions}\label{sec:conclusions}

The standard cost-based rejection model introduced by~\cite{Chow-RejectOpt-TIT1970} requires explicit definition of the rejection cost which is difficult in applications when the reject cost and the label loss have different nature or physical units. \cite{Pietraszek-AbstainROC-ICML2005} proposed the bounded-improvement model which avoids the problem by defining an optimal prediction strategy in terms of the coverage and the selective risk. We have coined a symmetric definition, the bounded-coverage model, which is useful when defining the target coverage is easier than defining the target selective risk. Our main result is a formal proof that despite their different objectives the three rejection models are equivalent in the sense that they lead to the same prediction strategy: the Bayes classifier and the randomized Bayes selection function. Thanks to the common optimal solution it is possible to convert between parameters of different rejection models. For example, for any target risk defining the bounded-improvement model there exists a corresponding reject cost so that both models have the same optimal strategy. 

The explicit characterization of the optimal strategies provides a recipe to construct plug-in rules solving the bounded-improvement and the bounded-coverage models. Any method estimating the class posterior probabilities can be thus turned into an algorithm for learning the selective classifier that solves the bounded-improvement and the bounded-coverage model. 

We have defined a notion of a proper uncertainty score which is sufficient to construct the randomized Bayes selection function. We proposed two algorithms to learn a proper uncertainty score from examples for a given classifier.
We have shown that both algorithms provide a Fisher consistent estimate of the proper uncertainty score. 
As a proof of concept we evaluated the proposed algorithms on different types of prediction problems. We have shown that the proposed algorithm based on minimization of the SELE loss outperforms existing approaches tailored for a particular prediction model and it works on par with the recently published state-of-the-art TCP score~\citep{Corbiere-Failure-NeurIPS2019}. Unlike the TCP score which requires the classifier to output the class posterior probabilities the proposed algorithms are applicable for an arbitrary black-box classifier. 

We have drawn a connection between the proposed bounded-coverage model and the RC curve. Namely, the RC curve represents quality of all admissible solutions of the bounded-coverage model that can be constructed from a pair of  classifier and uncertainty score. The AuRC is then the expected quality of the selective classifier constructed from the pair when the target coverage is selected uniformly at random. This connection sheds light on many published methods which do not explicitly define the target objective but use the RC curve and the AuRC as evaluation metrics.


Finally let us mentioned some topics for future work. Firstly, the proposed algorithms consider two-stage scenario when the classifier and the uncertainty score are learned separately from independent training sets. Although the scenario is useful in practice, an algorithm learning the classifier and the uncertainty score simultaneously from a single training set constitute an interesting topic to be solved. Secondly, we have shown how to learn the proper uncertainty score but have not discussed how to set up the decision threshold and the acceptance probability that are also needed to construct the selective classifier. It is straightforward to tune these parameters on empirical data using the RC curve. 
Analysis of the generalization error of this empirical approach is an open issue which has been solved only for the decision threshold of the bounded-improvement model by~\cite{Geifman-SelectClass-NIPS2017}. Thirdly, the empirical evaluation is limited to uncertainty scores linear in the parameters to be learned. Efficient implementations of the algorithms applicable to non-linear models, like e.g. the neural networks, is an another topic left for future.

\acks{We would like to acknowledge support for this work from Czech Science Foundation project GACR GA19-21198S. }


\newpage

\clearpage
\appendix

\section{Proofs of theorems from Section~\ref{sec:rejectOption}}

\subsection{Proof of Theorem~\ref{thm:optimalPredictor} }

The Bayes classifier reads
\begin{equation}
   h_B(x) \in \argmin_{\hat{y}\in\SY} \sum_{y\in\SY} p(y\mid x) \,
   \ell(y,\hat{y})
   \tag{\ref{equ:bayesPredictor}}
\end{equation}
\boundedImprovement*

\optimalPredictor*

\begin{proof}
It is sufficient to show that $(h_B,c)$ is feasible to~(\ref{equ:selectClassifTask}), i.e., that $R_S(h_B,c)\le \lambda$.
Then $(h_B,c)$ attains the same maximum objective value $\phi(c)$ as $(h,c)$.
Derive
\begin{align*}
R_S(h_B,c)&=\frac{1}{\phi(c)}\int\limits_{\SX}\sum\limits_{y\in\SY}p(x,y)\,\ell(y,h_B(x))\,c(x)\,dx \\
&= \frac{1}{\phi(c)}\int\limits_{\SX}p(x)c(x)\left(\sum\limits_{y\in\SY} p(y\,|\,x)\,\ell(y,h_B(x))\right)\,dx\\
&\stackrel{(\ref{equ:bayesPredictor})}{\le} \frac{1}{\phi(c)}\int\limits_{\SX}p(x)c(x)\left(\sum\limits_{y\in\SY} p(y\,|\,x)\,\ell(y,h(x))\right)\,dx\\
&=\frac{1}{\phi(c)}\int\limits_{\SX}\sum\limits_{y\in\SY}p(x,y)\,\ell(y,h(x))\,c(x)\,dx\\
&=R_S(h,c) \le \lambda.
\end{align*}
\end{proof}

\subsection{Proof of Theorem~\ref{thm:selectContinuousTask}}

The presented proof of the theorem uses Lemmas~\ref{lemma:nonzero-int-2} and~\ref{lemma:nonzero-int-3}, both derived based on Lemma~\ref{lemma:nonzero-int-1} bellow.

\begin{lemma}\label{lemma:nonzero-int-1}
For a set $\SX$, let $f:\SX \to \Re_{+}$\footnote{We use $\Re$, $\Re_{+}$ and $\Naturals_{+}$ to denote the set of real numbers, non-negative real numbers and positive integers, respectively.} and $g:\SX \to \Re$ be measurable functions such that $\int_{\SX} f(x)dx > 0$ and $g(x) > 0$ for all $x\in\SX$. Then it holds $\int_{\SX}g(x)f(x)dx > 0$.
\end{lemma}

\begin{proof}
    For $n\in \Naturals_{+}$, define functions
    \begin{equation*}
        f_n(x)=\left\{
        \begin{array}{cl}
             f(x) & \text{if } g(x)\ge \frac{1}{n}, \\
             0 & \text{otherwise.}
        \end{array}
        \right.
    \end{equation*}
    The sequence $\{f_n\}_{n=1}^{\infty}$ is monotone and converges to $f$. Using the monotone convergence theorem~\citep{Stein-RealAnalysis-2009}, derive
    \begin{equation*}
        0 < \int\limits_{\SX} f(x)dx = \int\limits_{\SX} \lim_{n \to \infty} f_n(x)dx = \lim_{n \to \infty} \int\limits_{\SX} f_n(x)dx\,.
    \end{equation*}
    This means that there is a $k \in \Naturals_{+}$ such that $\int_{\SX} f_k(x)dx > 0$, hence we conclude
    \begin{equation*}
        \int\limits_{\SX}g(x)f(x)dx \ge \int\limits_{\SX}g(x)f_k(x)dx \ge \int\limits_{\SX}\frac{1}{k} f_k(x)dx > 0.
    \end{equation*}
    
\end{proof}

\begin{lemma}\label{lemma:nonzero-int-2}
For a set $\SX$, let $f:\SX \to \Re_{+}$ and $g:\SX \to \Re$ be measurable functions such that $\int_{\SX} f(x)dx > 0$ and $g(x) > b$ for all $x\in\SX$ and some $b\in \Re$. Then it holds $\int_{\SX}g(x)f(x)dx > b \int_{\SX} f(x)dx$.
\end{lemma}

\begin{proof}
    By Lemma~\ref{lemma:nonzero-int-1}, we have
    \begin{equation*}
        \int\limits_{\SX}(g(x)-b)f(x)dx > 0,
    \end{equation*}
    thus
    \begin{equation*}
        \int\limits_{\SX}g(x)f(x)dx= \int\limits_{\SX}(g(x)-b)f(x)dx+ \int\limits_{\SX}b f(x)dx > b \int\limits_{\SX}f(x)dx.
    \end{equation*}
    
\end{proof}

\begin{lemma}\label{lemma:nonzero-int-3}
For a set $\SX$, let $f:\SX \to \Re_{+}$ and $g:\SX \to \Re$ be measurable functions such that $\int_{\SX} g(x)f(x)dx > 0$ and $g(x) < 1$ for all $x\in \SX$. Then it holds $\int_{\SX} f(x)dx > \int_{\SX} g(x)f(x)dx$.
\end{lemma}

\begin{proof}
    $\int_{\SX} g(x)f(x)dx > 0$ implies $\int_{\SX} f(x)dx > 0$. Since it holds $\forall x\in \SX: (1-g(x))>0$, Lemma~\ref{lemma:nonzero-int-1} yields
    \begin{equation*}
        0 < \int\limits_{\SX}(1-g(x))f(x)dx = \int\limits_{\SX}f(x)dx- \int\limits_{\SX}g(x)f(x)dx,
    \end{equation*}
   and $\int_{\SX} f(x)dx > \int_{\SX} g(x)f(x)dx$ is obtained as a direct consequence.
\end{proof}

\boundedImprovSelection*
\boundedImprovSolution*

\begin{proof} Observe that $b\ge 0$, because $\rho(\SX_{\bar{r}(x)\le 0})\le 0$. Next, observe that Problem~\ref{task:boundedImprovSelection} can be rewritten into the form
\begin{equation}
   \label{equ:selectContinuousTask}
  \max_{c\in [0,1]^\SX} \int_{\SX} p(x)c(x)dx  \,\, \text{ s.t.}\,\,\int_{\SX} p(x)c(x)\bar{r}(x)dx \le 0\,
\end{equation}
since
\begin{align}
  R_S(h,c) - \lambda &= \frac{\int\limits_{\SX}\sum\limits_{y\in\SY}p(x,y)\,\ell(y,h(x))\,c(x)\,dx- \lambda \phi(c)}{\phi(c)} \\
  &= \frac{\int\limits_{\SX}p(x)c(x)r(x)\,dx- \lambda \int\limits_{\SX}p(x)c(x)}{\phi(c)} = \frac{\int\limits_{\SX}p(x)c(x)\bar{r}(x)\,dx}{\phi(c)} \,.
\end{align}
Let $F(c)=\phi(c)=\int_{\SX}p(x)c(x)dx$ denote the objective function of~(\ref{equ:selectContinuousTask}).

\begin{case} $b>0$. \newline
    \textbf{\textit{Claim I}} Each $c^*:\SX\to [0,1]$ which fulfils~(\ref{equ:cond-1}), (\ref{equ:cond-2}) and (\ref{equ:cond-3}) is feasible to~(\ref{equ:selectContinuousTask}) and
    \begin{equation}\label{equ:objective_value}
    F(c^*)=\int\limits_{\SX_{\bar{r}(x)<b}}p(x)dx-\frac{1}{b}\rho(\SX_{\bar{r}(x)<b}).
    \end{equation}
    Proof of Claim I.
    
    Equality~(\ref{equ:objective_value}) is simply obtained by summing LHS and RHS of~(\ref{equ:cond-1}), (\ref{equ:cond-2}) and (\ref{equ:cond-3}). To verify the constraint of~(\ref{equ:selectContinuousTask}), observe that, since $\bar{r}$ is a bounded function and
    \begin{equation}
    \int_{\SX_{\bar{r}(x)<b}}p(x)(c^*(x)-1)dx \stackrel{(\ref{equ:cond-1})}= 0 \,,
    \end{equation}
    it holds that
    \begin{equation}
    \int_{\SX_{\bar{r}(x)<b}}p(x)(c^*(x)-1)\bar{r}(x)dx = 0 \,,
    \end{equation}
    which implies
    \begin{equation}\label{equ:proof-0}
         \int\limits_{\SX_{\bar{r}(x)<b}}p(x)c^*(x)\bar{r}(x)dx = \int\limits_{\SX_{\bar{r}(x)<b}}p(x)\bar{r}(x)dx \stackrel{(\ref{equ:rho_def})}= \rho(\SX_{\bar{r}(x)<b}).
    \end{equation}
    If $b<\infty$, then
    \begin{align}
        \int\limits_{\SX}p(x)c^*(x)\bar{r}(x)dx &\stackrel{(\ref{equ:cond-3})}= \int\limits_{\SX_{\bar{r}(x)<b}}p(x)c^*(x)\bar{r}(x)dx + \int\limits_{\SX_{\bar{r}(x)=b}}p(x)c^*(x)\bar{r}(x)dx\nonumber \\
        & \stackrel{(\ref{equ:proof-0})}= \int\limits_{\SX_{\bar{r}(x)<b}}p(x)\bar{r}(x)dx + b\int\limits_{\SX_{\bar{r}(x)=b}}p(x)c^*(x)dx \\ & \stackrel{(\ref{equ:cond-2}), (\ref{equ:rho_def}), (\ref{equ:proof-0})} = \rho(\SX_{\bar{r}(x)<b})-\rho(\SX_{\bar{r}(x)<b})=0.\label{equ:proof-1}
    \end{align}
    If $b=\infty$, then
    \begin{align*}
        & \int\limits_{\SX}p(x)c^*(x)\bar{r}(x)dx = \int\limits_{\SX_{\bar{r}(x)<b}}p(x)c^*(x)\bar{r}(x)dx \stackrel{(\ref{equ:proof-0})}= \rho(\SX_{\bar{r}(x)<b}) \le 0.
    \end{align*}
    \textbf{\textit{Claim II}}
    Let $c:\SX\to [0,1]$ be a feasible solution to~(\ref{equ:selectContinuousTask}) that violates at least one of the constraints~(\ref{equ:cond-1}), (\ref{equ:cond-2}) and (\ref{equ:cond-3}). Then, $F(c) < F(c^*)$, where $c^*:\SX\to [0,1]$ is a confidence function satisfying~(\ref{equ:cond-1}), (\ref{equ:cond-2}), (\ref{equ:cond-3}), and, without loss of generality,
    \begin{equation}\label{equ:c-star-assumption}
        \forall x\in \SX_{\bar{r}(x)<b}: c^*(x)=1\, .
    \end{equation}
    Proof of Claim II.
    
    Distinguish three cases.
    \begin{subcase}
    Condition~(\ref{equ:cond-3}) is violated (note that this is possible only if $b<\infty$), i.e.
    \begin{equation}\label{equ:cond-3-violated}
        \int\limits_{\SX_{\bar{r}(x)>b}}p(x)c(x)dx >0.
    \end{equation}
    Inequality~(\ref{equ:cond-3-violated}) and Lemma~\ref{lemma:nonzero-int-2} (applied to $f(x)=p(x)c(x)$ and $g(x)=\bar{r}(x)$) yield
    
    \begin{equation*}
        \int\limits_{\SX_{\bar{r}(x)>b}} p(x)c(x)\bar{r}(x)dx > b \int\limits_{\SX_{\bar{r}(x)>b}} p(x)c(x)dx.
    \end{equation*}
    Therefore, we can write
    \begin{equation}\label{equ:constr-multiple}
        \int\limits_{\SX_{\bar{r}(x)>b}} p(x)c(x)\bar{r}(x)dx=b' \int\limits_{\SX_{\bar{r}(x)>b}} p(x)c(x)dx
    \end{equation}
    for a suitable $b'\in \Re_{+}$ such that
    \begin{equation}\label{equ:multiple}
        b'>b>0.
    \end{equation}
    
    Based on the constraint of~(\ref{equ:selectContinuousTask}), derive
    \begin{align}
        \int\limits_{\SX}p(x)c(x)\bar{r}(x)dx &\stackrel{(\ref{equ:constr-multiple})}= \int\limits_{\SX_{\bar{r}(x)<b}} p(x)c(x)\bar{r}(x)dx + b\int\limits_{\SX_{\bar{r}(x)=b}} p(x)c(x)dx + b'\int\limits_{\SX_{\bar{r}(x)>b}} p(x)c(x)dx\nonumber \\
        & \stackrel{(\ref{equ:selectContinuousTask})}\le 0 \stackrel{(\ref{equ:proof-1})}= \int\limits_{\SX_{\bar{r}(x)<b}} p(x)c^*(x)\bar{r}(x)dx + b\int\limits_{\SX_{\bar{r}(x)=b}}p(x)c^*(x)dx.\label{equ:proof-10}
    \end{align}
    Let $\sigma(x)=\frac{1}{b}\bar{r}(x)$. Inequality~(\ref{equ:proof-10}) can be rearranged and upper bounded as
    \begin{align}\label{equ:proof-11}
        \int\limits_{\SX_{\bar{r}(x)=b}} &p(x)c(x)dx - \int\limits_{\SX_{\bar{r}(x)=b}}p(x)c^*(x)dx + \frac{b'}{b}\int\limits_{\SX_{\bar{r}(x)>b}} p(x)c(x)dx \\ &\stackrel{(\ref{equ:proof-10})}\le
        \int\limits_{\SX_{\bar{r}(x)<b}} p(x)(c^*(x)-c(x))\sigma(x)dx \nonumber \le \int\limits_{\SX_{\bar{r}(x)<b}}p(x)c^*(x)dx- \int\limits_{\SX_{\bar{r}(x)<b}}p(x)c(x)dx
    \end{align}
    where the second inequality follows from $\forall x\in \SX_{\bar{r}(x)<b}: \sigma(x)\le 1$. From this we get
    \begin{align}
        \int\limits_{\SX_{\bar{r}(x)\le b}} p(x)c(x)dx \stackrel{(\ref{equ:proof-11})} \le \int\limits_{\SX_{\bar{r}(x)\le b}}p(x)c^*(x)dx - \frac{b'}{b}\int\limits_{\SX_{\bar{r}(x)>b}} p(x)c(x)dx \,.\label{equ:proof-11b}
    \end{align}
    Now, derive
    \begin{align*}
        F(c) &=\int\limits_{\SX_{\bar{r}}(x)\le b} p(x)c(x)dx + \int\limits_{\SX_{\bar{r}}(x)> b} p(x)c(x)dx \\ &\stackrel{(\ref{equ:proof-11b})}\le \int\limits_{\SX_{\bar{r}(x)\le b}} p(x)c^*(x)dx - \left( \frac{b'}{b} - 1\right) \int\limits_{\SX_{\bar{r}(x)>b}} p(x)c(x)dx \\
        &\stackrel{(\ref{equ:cond-3-violated}), (\ref{equ:multiple})}< \int\limits_{\SX_{\bar{r}(x)\le b}}p(x)c^*(x)dx = F(c^*).
    \end{align*}
    \end{subcase}
    
    \begin{subcase}
    Condition~(\ref{equ:cond-3}) holds, condition~(\ref{equ:cond-2}) is violated.
    
    If $\int_{\SX_{\bar{r}(x)=b}}p(x)c(x)dx < -\frac{\rho(\SX_{\bar{r}(x)<b})}{b}$, then obviously $F(c) < F(c^*)$.
    Hence, assume 
    \begin{equation}\label{equ:cond-2-violated}
        \int\limits_{\SX_{\bar{r}(x)=b}}p(x)c(x)dx > -\frac{\rho(\SX_{\bar{r}(x)<b})}{b}.
    \end{equation}
    Analogically to~(\ref{equ:proof-10}), derive
    \begin{align}
        \int\limits_{\SX_{\bar{r}(x)<b}} &p(x)c(x)\bar{r}(x)dx + b\int\limits_{\SX_{\bar{r}(x)=b}}p(x)c(x)dx \stackrel{(\ref{equ:selectContinuousTask})}
        \le 0 \\ &\stackrel{(\ref{equ:proof-1})}= \int\limits_{\SX_{\bar{r}(x)<b}}p(x)c^*(x)\bar{r}(x)dx + b\int\limits_{\SX_{\bar{r}(x)=b}}p(x)c^*(x)dx,
    \end{align}
    and
    \begin{equation}\label{equ:proof-2}
        \int\limits_{\SX_{\bar{r}(x)<b}}p(x)c(x)\sigma(x)dx + \int\limits_{\SX_{\bar{r}(x)=b}}p(x)c(x)dx \le \int\limits_{\SX_{\bar{r}(x)<b}}p(x)c^*(x)\sigma(x)dx + \int\limits_{\SX_{\bar{r}(x)=b}}p(x)c^*(x)dx
    \end{equation}
    where $\sigma(x)=\frac{1}{b}\bar{r}(x)<1$ for all $x\in \SX_{\bar{r}(x)<b}$.
    
    Denote and derive
    \begin{equation}\label{equ:proof-3}
        \Delta = \int\limits_{\SX_{\bar{r}(x)=b}}p(x)c(x)dx - \int\limits_{\SX_{\bar{r}(x)=b}}p(x)c^*(x)dx \stackrel{(\ref{equ:cond-2})} = \int\limits_{\SX_{\bar{r}(x)=b}}p(x)c(x)dx +\frac{\rho(\SX_{\bar{r}(x)<b})}{b} \stackrel{(\ref{equ:cond-2-violated})}> 0.
    \end{equation}
    Then, (\ref{equ:proof-2}) can be rewritten as
    \begin{equation}\label{equ:delta-ineq}
        \int\limits_{\SX_{\bar{r}(x)<b}}p(x)(c^*(x)-c(x))\sigma(x)dx \ge \Delta \stackrel{(\ref{equ:proof-3})} > 0.
    \end{equation}
    Inequality~(\ref{equ:delta-ineq}) and Lemma~\ref{lemma:nonzero-int-3} (applied to $g(x)=\sigma(x)<1$ and $f(x)=p(x)(c^*(x)-c(x)) \stackrel{(\ref{equ:c-star-assumption})} \ge 0$ over $\SX_{\bar{r}(x)<b}$) yield
    \begin{equation}\label{equ:proof-4}
         \int\limits_{\SX_{\bar{r}(x)<b}}p(x)(c^*(x)-c(x))dx > \Delta.
    \end{equation}
    Now, combine and rearrange~(\ref{equ:proof-3}) and~(\ref{equ:proof-4}) to obtain
    \begin{align}
        F(c^*)&=\int\limits_{\SX_{\bar{r}(x)<b}}p(x)c^*(x)dx+\int\limits_{\SX_{\bar{r}(x)=b}}p(x)c^*(x)dx \stackrel{(\ref{equ:proof-4})} > \Delta \\
        & \stackrel{(\ref{equ:proof-3})}= \int\limits\limits_{\SX_{\bar{r}(x)<b}}p(x)c(x)dx+\int\limits_{\SX_{\bar{r}(x)=b}}p(x)c(x)dx=F(c).
    \end{align}
    \end{subcase}
    
    \begin{subcase}
    Conditions~(\ref{equ:cond-2}) and~(\ref{equ:cond-3}) hold, condition~(\ref{equ:cond-1}) is violated, i.e.
    \begin{equation}\label{equ:cond-1-violated}
        \int\limits_{\SX_{\bar{r}(x)<b}}p(x)c(x)dx <\int\limits_{\SX_{\bar{r}(x)<b}}p(x)dx.
    \end{equation}
    Then,
    \begin{equation*}
        F(c^*)=\int\limits_{\SX_{\bar{r}(x)<b}}p(x)c^*(x)dx -\frac{\rho(\SX_{\bar{r}(x)<b})}{b} \stackrel{(\ref{equ:cond-1-violated})}> \int\limits\limits_{\SX_{\bar{r}(x)<b}}p(x)c(x)dx -\frac{\rho(\SX_{\bar{r}(x)<b})}{b}=F(c).
    \end{equation*}
    \end{subcase}
    
\end{case}
\begin{case} $b=0$.

This occurs only if $\int_{\SX_{\bar{r}(x)<0}}p(x)\bar{r}(x)dx = 0$. The constraint of~(\ref{equ:selectContinuousTask}) implies
\begin{equation*}
\int\limits_{\SX_{\bar{r}(x)>0}}p(x)c(x)\bar{r}(x)dx = 0\,,
\end{equation*}
thus
\begin{equation*}
\int\limits_{\SX_{\bar{r}(x)>0}}p(x)c(x)dx = 0\,,    
\end{equation*}
which confirms condition~(\ref{equ:cond-3}).

Finally, the obvious equations 
\begin{align*}
\max_{c:\SX\to [0,1]}\int\limits_{\SX_{\bar{r}(x)<0}}p(x)c(x)dx&=\int\limits_{\SX_{\bar{r}(x)<0}}p(x)dx\,, \text{ and} \\
\max_{c:\SX\to [0,1]}\int\limits_{\SX_{\bar{r}(x)=0}}p(x)c(x)dx&=\int\limits_{\SX_{\bar{r}(x)=0}}p(x)dx
\end{align*}
confirm condition~(\ref{equ:cond-1}) and~(\ref{equ:cond-2}), respectively.

\end{case}
\end{proof}

\subsection{Proof of Theorem~\ref{thm:optSelectingFun}}

\ThmOptSelectingFce*

\begin{proof}
The optimality conditions~\equ{equ:cond-1} and ~\equ{equ:cond-3} given in Theorem~\ref{thm:selectContinuousTask} are equivalent to a probabilistic statement $\Prob_{x\sim p(x)}[ c^*(x)=0 \wedge \bar{r}(x)< b]=0$ and $\Prob_{x\sim p(x)}[ c^*(x)=1 \wedge \bar{r}(x)> b]=0$, respectively. Hence the two conditions are satisfied by a selection function which predicts, $c^*(x)=1$, whenever $\bar{r}(x)<b$ and rejects, $c^*(x)=0$, whenever $\bar{r}(x)>b$. Or equivalently, using the identity $\bar{r}(x)=r(x)-\lambda$ and a threshold $\gamma=b+\lambda$, 
by $c^*(x)=1$ when $r(x)<\gamma$  and $c^*(x)=0$ when $r(x)> \gamma$.
%
Finally, if we opt for a selection function that is constant $c^*(x)=\tau$ inside the boundary region $\SX_{\bar{r}(x)=b}$, then the condition~\equ{equ:cond-2} implies $\tau=-\frac{\rho(\SX_{r(x)<\gamma})}{b\cdot \rho_0}$ if $b>0$, where $\rho_0=\int_{\SX_{\bar{r}(x)=b}}p(x)\,dx$, and $\tau=1$ if $b=0$. Using $\SX_{\bar{r}(x)<b}=\SX_{r(x)<\gamma}$ and $b\cdot \rho_0=\rho(\SX_{\bar{r}(x)=b})=\rho(\SX_{r(x)=\gamma})$, we derive~\equ{equ:rejectProbab}.
\end{proof}

\subsection{Proof of Theorem~\ref{thm:optClsForBoundedCovModel}}

\optimalClassifierForBndCov*

\begin{proof}
The theorem follows from the fact that $R_S(h_B,c) \le R_S(h,c)$ for any $(h,c)$ feasible to~\equ{equ:boundedCoverageModel}, which is derived as follows:
\begin{align*}
R_S(h_B,c)&=\frac{1}{\phi(c)}\int\limits_{\SX}\sum\limits_{y\in\SY}p(x,y)\,\ell(y,h_B(x))\,c(x)\,dx \\
&= \frac{1}{\phi(c)}\int\limits_{\SX}p(x)c(x)\left(\sum\limits_{y\in\SY} p(y\,|\,x)\,\ell(y,h_B(x))\right)\,dx\\
&{\stackrel{\equ{equ:bayesPredictor}}\le} \frac{1}{\phi(c)}\int\limits_{\SX}p(x)c(x)\left(\sum\limits_{y\in\SY} p(y\,|\,x)\,\ell(y,h(x))\right)\,dx\\
&=\frac{1}{\phi(c)}\int\limits_{\SX}\sum\limits_{y\in\SY}p(x,y)\,\ell(y,h(x))\,c(x)\,dx\\
&=R_S(h,c).
\end{align*}

\end{proof}

\subsection{Proof of Theorem~\ref{thm:selectContinuousTask2}}

\boundedCoverageKnownCls*
\boundedCoverageSolution*

\begin{proof}
By substituting the definitions of $R_S(h,c)$ and $\phi(c)$ to~(\ref{equ:boundedCoverageModelKnownCls}), we rewrite the problem into the form
\begin{equation}
\label{equ:boundedCovModelKnownClsSubstituted}
   \min_{c\in [0,1]^\SX} \frac{\int\limits_{\SX}p(x)c(x)r(x)dx}{\int\limits_{\SX} p(x)c(x)dx} \,\quad \mbox{s.t.}\quad
    \int\limits_{\SX} p(x)c(x)dx \geq \omega  \,.
\end{equation}

Let $F(c)=\frac{\int_{\SX}p(x)c(x)r(x)dx}{\int_{\SX}p(x)c(x)dx}$ denote the objective function of~(\ref{equ:boundedCovModelKnownClsSubstituted}). Whenever $c^*:\SX\to [0,1]$ fulfils~(\ref{equ:task2-cond-1}), (\ref{equ:task2-cond-2}) and (\ref{equ:task2-cond-3}), it is feasible to~(\ref{equ:boundedCovModelKnownClsSubstituted}) and
\begin{equation}
F(c^*)=\beta + \frac{1}{\omega} \int\limits_{\SX_{{r}(x)<\beta}}p(x)r(x)dx-\frac{\beta}{\omega} \int\limits_{\SX_{{r}(x)=\beta}}p(x)dx \,,
\end{equation}
which is a value independent of $c^*$.

We will prove the theorem by showing that any
$c:\SX\to [0,1]$ feasible to~(\ref{equ:boundedCovModelKnownClsSubstituted}) that violates at least one of conditions (\ref{equ:task2-cond-1}), (\ref{equ:task2-cond-2}), (\ref{equ:task2-cond-3}) is not an optimal solution.
Three cases will be examined.

\noindent
\textbf{Case 1} Condition~(\ref{equ:task2-cond-1}) is violated, i.e.,
\begin{equation}\label{equ:task2-cond-3-violated}
    \int\limits_{\SX_{{r}(x)<\beta}}p(x)c(x)dx < \int\limits_{\SX_{{r}(x)<\beta}}p(x)dx.
\end{equation}
This means that there is a subset $X\subseteq \SX$ such that 
\begin{equation}\label{equ:task2-prop1}
    \forall x\in X \,:\, r(x) \ge \beta
\end{equation}
and 
\begin{equation}\label{equ:task2-prop2}
    \int\limits_{X}p(x)c(x)dx = \int\limits_{\SX_{{r}(x)<\beta}}p(x)dx - \int\limits_{\SX_{{r}(x)<\beta}}p(x)c(x)dx \stackrel{(\ref{equ:task2-cond-3-violated})}> 0 \,.
\end{equation}
Define $c':\SX\to [0,1]$ as follows.
\begin{equation}
   c'(x) = \left \{ 
      \begin{array}{ccl}
         1 & \mbox{if} & r(x) < \beta \,,\\
         0 & \mbox{if} & x \in X\,, \\
         c(x) & \mbox{otherwise} \,.\\
      \end{array}
   \right .
\end{equation}
$c'$ is feasible to~(\ref{equ:boundedCovModelKnownClsSubstituted}) as $\phi(c')=\phi(c)$.
Derive
\begin{align}
    \phi(c) \left(F(c)-F(c')\right) &= \int\limits_{X}p(x)c(x)r(x)dx - \int\limits_{\SX_{r(x)<\beta}}p(x)r(x)dx + \int\limits_{\SX_{r(x)<\beta}}p(x)c(x)r(x)dx \\
    &\stackrel{(\ref{equ:task2-prop1}), (\ref{equ:task2-prop2})}\ge
    \int\limits_{\SX_{{r}(x)<\beta}}\beta \cdot p(x)(1-c(x))dx - \int\limits_{\SX_{{r}(x)<\beta}}p(x)(1-c(x))r(x)dx \\
    &= \int\limits_{\SX_{{r}(x)<\beta}}p(x)(1-c(x))(\beta-r(x))dx > 0 \label{equ:task2-last-inequ}
\end{align}
where the inequality in (\ref{equ:task2-last-inequ}) is obtained from Lemma~\ref{lemma:nonzero-int-1} applied to $f(x)=p(x)(1-c(x))$, $g(x)=\beta-r(x)$, and the set $\SX_{{r}(x)<\beta}$. This shows that $c$ is not an optimal solution. 

\noindent
\textbf{Case 2} Condition~(\ref{equ:task2-cond-1}) is satisfied, condition~(\ref{equ:task2-cond-2}) is violated and
\begin{equation}\label{equ:task2-case2}
     \int_{\SX_{{r}(x)=\beta}}p(x)c(x)dx < \omega - \int_{\SX_{{r}(x)<\beta}}p(x)dx > 0\,.
\end{equation}

In this case, there is a $c':\SX\to [0,1]$ such that
\begin{alignat}{2}
    c'(x) &= c(x) &\quad \mbox{if } \,\, r(x) < \beta\,,\\
    c'(x) &= 0 &\quad \mbox{if } \,\, r(x) > \beta\,,
\end{alignat}
and
\begin{equation}\label{equ:task2-prop3}
    \int_{\SX_{{r}(x)=\beta}}p(x)c'(x)dx = \int_{\SX_{{r}(x)=\beta}}p(x)c(x)dx + \int_{\SX_{{r}(x)>\beta}}p(x)c(x)dx \,.
\end{equation}

If Lemma~\ref{lemma:nonzero-int-2} is applied to $f(x)=p(x)c(x)$, $g(x)=r(x)$, and the set $\SX_{r(x)>\beta}$, we get
\begin{equation}\label{equ:task2-prop4}
    \int_{\SX_{{r}(x)>\beta}}p(x)c(x)r(x)dx > \beta \int_{\SX_{{r}(x)>\beta}}p(x)c(x)dx \,.
\end{equation}

It also holds that $\phi(c')=\phi(c)$. Now, deriving
\begin{align}
    \phi(c) \left(F(c)-F(c')\right) &\stackrel{(\ref{equ:task2-prop3})}= \int\limits_{\SX_{r(x) > \beta}}p(x)c(x)r(x)dx - \beta \int\limits_{\SX_{r(x)>\beta}}p(x)c(x)dx \\
    &\stackrel{(\ref{equ:task2-prop4})}>
    \beta \int\limits_{\SX_{r(x) > \beta}}p(x)c(x) - \beta \int\limits_{\SX_{r(x)>\beta}}p(x)x(x)dx = 0
\end{align}
shows that $c$ is not an optimal solution.

\noindent
\textbf{Case 3} $\phi(c)>\omega$, which occurs if
\begin{equation}
     \int_{\SX_{{r}(x)=\beta}}p(x)c(x)dx > \omega - \int_{\SX_{{r}(x)<\beta}}p(x)dx > 0
\end{equation}
(implying that condition~(\ref{equ:task2-cond-2}) is violated), or if condition~(\ref{equ:task2-cond-3}) is violated.

Observe that $F(c)=F(\alpha \cdot c)$ for any $a\in \Re_+$. Let $c'=\frac{\omega}{\phi(c)} \cdot c$. Since $\phi(c')=\omega$, the selection function $c'$ is feasible to~(\ref{equ:boundedCovModelKnownClsSubstituted}). Because
\begin{equation}
     \int_{\SX_{{r}(x)<\beta}}p(x)c'(x)dx = \frac{\omega}{\phi(c)} \int_{\SX_{{r}(x)<\beta}}p(x)c(x)dx < \int_{\SX_{{r}(x)<\beta}}p(x)dx \,,
\end{equation}
$c'$ violates condition~(\ref{equ:task2-cond-1}) and is therefore not an optimal solution (see Case 1). This implies that $c$ is not an optimal solution too.
\end{proof}

\subsection{Proof of Theorem~\ref{thm:optSelectingFun2} }

\ThmOptSelectingFceTwo*

\begin{proof}
    It is easy to see that $c^*$ satisfies conditions~(\ref{equ:task2-cond-1}) and~(\ref{equ:task2-cond-3}).
    The validity of condition~(\ref{equ:task2-cond-2}) is proved as follows. If $\int_{\SX_{{r}(x)=\beta}}p(x)dx = 0$, then $\int_{\SX_{{r}(x)=\beta}}p(x)dx = \omega$, and condition~(\ref{equ:task2-cond-2}) is met. If $\int_{\SX_{{r}(x)=\beta}}p(x)dx > 0$, we derive
    \begin{align}
        \int_{\SX_{{r}(x)=\beta}}p(x)c^*(x)dx = \frac{\omega - \int_{\SX_{{r}(x)<\beta}}p(x)dx} {\int_{\SX_{{r}(x)=\beta}}p(x)dx} \int_{\SX_{{r}(x)=\beta}}p(x)dx = \omega - \int_{\SX_{{r}(x)<\beta}}p(x)dx \,.
    \end{align}
\end{proof}

\section{Proofs of theorems from Section~\ref{sec:UncertaintyLearning}}

\subsection{Proof of Theorem~\ref{thm:regEstimSolution} }

The expectation of the squared loss deviation reads
\begin{equation}
  E_{\rm reg}(s)  = \int\limits_{\SX} \sum_{y\in\SY} p(x,y)\,\Big (\ell(y,h(x))-s(x)\Big)^2 dx \:.
\end{equation}

\regEstimatorSolution*

\begin{proof} We can rewrite $E_{\rm reg}(s)$ as
\[
   E_{\rm reg}(s) = \int_{\SX} p(x) \sum_{y\in\SY} p(y\mid x) \Big (
   \ell(y,h(x))^2- 2\,\ell(y,h(x))\,s(x) + s(x)^2 \Big ) dx = \int_{\SX} p(x) f(s(x)) dx \:.
\]
Due to additivity we can solve $\min_{s\colon\SX\rightarrow\Re} E_{\rm reg}(s)$ for each $x\in\SX$ separately by setting derivative of $f(s)$ to zero and solving for $s$ which yields
\[
      f'(s) = -2\sum_{y\in\SY} p(y\mid x) \ell( y,h(x)) + 2\sum_{y\in\SY} p(y\mid x) s(x) = 0 \Rightarrow s^*(x) = \sum_{y\in\SY} p(y\mid x) \ell( y, h(x)) \:.
\]

\end{proof}

\subsection{Proof of Theorem~\ref{theorem:SeleUpperBoundsAuRC}}

\begin{lemma}\label{lemma:bound}
For an integer $n \ge 1$, let $\{a_i\}_{i=1}^n$ and $\{b_i\}_{i=1}^n$ be sequences of positive real numbers such that $\{\frac{a_i}{b_i}\}_{i=1}^n$ is a non-increasing sequence. For each non-decreasing sequence $\{\ell_i\}_{i=1}^n$ of non-negative real numbers with a positive sum it holds that
$$\frac{\sum_{i=1}^n a_i \ell_i}{\sum_{i=1}^n b_i \ell_i} \le \frac{\sum_{i=1}^n a_i}{\sum_{i=1}^n b_i}\:.$$
\end{lemma}

\begin{proof}
By induction on $n$. Base case: If $n=1$, then $\frac{a_1\ell_1}{b_1\ell_1}=\frac{a_1}{b_1}$.

Induction step: Let $n>1$. The fact that $\{\frac{a_i}{b_i}\}_{i=1}^{n}$ is non-increasing implies $a_i\le \frac{a_1}{b_1}\cdot b_i$ for all $i=1,\ldots,n$, hence
\begin{equation}
\frac{\sum_{i=2}^{n}a_i}{\sum_{i=2}^{n}b_i} \le \frac{\sum_{i=2}^{n}\frac{a_1}{b_1}\cdot b_i}{\sum_{i=2}^{n}b_i} = \frac{a_1}{b_1} \:.
\end{equation}

The lemma is obviously satisfied if $0< \ell_1=\ell_2= \ldots = \ell_n$. Assume that $\ell_n > \ell_1$. Then, the sequence $\{\ell_{i}-\ell_n\}_{i=2}^n$ is non-decreasing with positive sum, hence the induction hypothesis yields that
\begin{equation}\label{eq:sequences-1}
\frac{\sum_{i=2}^n a_i(\ell_i-\ell_1)}{\sum_{i=2}^n b_i(\ell_i-\ell_1)} \le \frac{\sum_{i=2}^{n}a_i}{\sum_{i=2}^{n}b_i} \le \frac{a_1}{b_1} \:.
\end{equation}
The induction hypotheses also ensures that
\begin{equation}\label{eq:sequences-2}
    \frac{\sum_{i=2}^n a_i\ell_i}{\sum_{i=2}^n b_i\ell_i} \le \frac{\sum_{i=2}^{n}a_i}{\sum_{i=2}^{n}b_i} \:.
\end{equation}
We can thus derive the following sequence of equivalent inequalities:
\begin{align*}
    b_1 \sum_{i=2}^n a_i(\ell_i-\ell_1) & \stackrel{(\ref{eq:sequences-1})}\le a_1 \sum_{i=2}^n b_i(\ell_i-\ell_1) \\
    b_1 \sum_{i=2}^n a_i\ell_i + a_1\ell_1 \sum_{i=2}^n b_i &\le a_1 \sum_{i=2}^n b_i\ell_i + b_1\ell_1 \sum_{i=2}^n a_i \\
    a_1\ell_1 b_1\! +\! b_1\! \sum_{i=2}^n a_i\ell_i \!+\! a_1\ell_1\! \sum_{i=2}^n b_i \!+\! \sum_{i=2}^n b_i\! \sum_{i=2}^n a_i\ell_i\! &\le\! a_1\ell_1 b_1 \!+\! a_1\! \sum_{i=2}^n b_i\ell_i \!+\! b_1\ell_1\! \sum_{i=2}^n a_i \!+\! \sum_{i=2}^n b_i\! \sum_{i=2}^n a_i\ell_i \\
    a_1\ell_1 b_1 \!+\! b_1\! \sum_{i=2}^n a_i\ell_i \!+\! a_1\ell_1\! \sum_{i=2}^n b_i \!+\! \sum_{i=2}^n b_i\! \sum_{i=2}^n a_i\ell_i & \!\!\stackrel{(\ref{eq:sequences-2})}\le\!\! a_1\ell_1 b_1\! + \!a_1\! \sum_{i=2}^n b_i\ell_i \!+\! b_1\ell_1\! \sum_{i=2}^n a_i \!+\! \sum_{i=2}^n a_i\! \sum_{i=2}^n b_i\ell_i \\
    a_1\ell_1\sum_{i=1}^n b_i + \sum_{i=1}^n b_i \sum_{i=2}^n a_i\ell_i &\le b_1\ell_1 \sum_{i=1}^n a_i + \sum_{i=1}^n a_i \sum_{i=2}^n b_i\ell_i \\
    \sum_{i=1}^n a_i\ell_i\sum_{i=1}^n b_i &\le \sum_{i=1}^n b_i\ell_i\sum_{i=1}^n a_i \\
    \frac{\sum_{i=1}^n a_i \ell_i}{\sum_{i=1}^n b_i \ell_i} &\le \frac{\sum_{i=1}^n a_i}{\sum_{i=1}^n b_i}\:.
\end{align*}

\end{proof}

For $i=1,\ldots,n$ and a permutation $\pi$ on $\{1,\ldots,n\}$, let $a_{\pi_i}=\sum_{j=i}^n \frac{n}{j}$ and $b_{\pi_i}=n-i+1$. We can write
$$\SELE(s, \ST_n) = \frac{1}{n^2} \sum_{i=1}^n b_{\pi_i} \ell_{\pi_i}$$
and
$$\AUC(s, \ST_n) = \frac{1}{n^2} \sum_{i=1}^n a_{\pi_i} \ell_{\pi_i} \,,$$
where $\ell_{\pi_i}=\ell(y_{\pi_i}, h(x_{\pi_i}))$.

Without loss of generality, assume that $\ell_1\le \ell_2 \le \ldots \le \ell_n$ and $\pi_i=i$ for all $i=1,\ldots,n$.

Let $H_k=\sum_{i=1}^k \frac{1}{i}$ denote the $k$-th harmonic number.
It fulfils
\begin{equation}\label{eq:harmonic-bounds}
\ln(k) + \gamma + \frac{1}{2k+1} \le H_k \le \ln(k) + \gamma + \frac{1}{2k-1}
\end{equation}
where $\gamma\approx 0.5772156649$ is the Euler–Mascheroni constant. Moreover, let us define $H_0=0$.

\begin{lemma}\label{lemma:monotonic}
Let $n\ge 2$ be an integer. For each $i=1,\ldots,n$, let $a_i=\sum_{j=i}^n \frac{n}{j}$ and $b_i=n-i+1$.
Then, $\frac{a_i}{b_i} \ge \frac{a_{i+1}}{b_{i+1}}$ holds for all $i=1,\ldots,n-1$.
\end{lemma}

\begin{proof}
Since
\begin{align*}
    \frac{a_i}{b_i} - \frac{a_{i+1}}{b_{i+1}} = \frac{n\cdot (H_n-H_{i-1})}{n-i+1} - \frac{n\cdot (H_n-H_i)}{n-i}
\end{align*}
it suffices to show that
$$(n-i)(H_n-H_{i-1}) \ge (n-i+1)(H_n-H_i)$$
and this is equivalent to showing that 
$$(n-i+1)H_i - (n-i)H_{i-1} \ge H_n\,.$$
If we substitute $H_{i-1}=H_i - \frac{1}{i}$, the inequality further reduces to
\begin{equation}\label{eq:monotonic}
    H_n - H_i \le \frac{n-i}{i} = \frac{n}{i}-1\,.
\end{equation}
Now we derive
$$H_n - H_i \stackrel{(\ref{eq:harmonic-bounds})}\le \ln\left(\frac{n}{i}\right) + \frac{1}{2n-1}-\frac{1}{2i+1} < \frac{n}{i}-1$$
where the second inequality follows from the fact that $\ln(x) < x - 1$ for all $x>1$. This confirms inequality~(\ref{eq:monotonic}).

\end{proof}

\begin{lemma}\label{lemma:harm_expr}
$\sum_{i=1}^n (H_n-H_{i-1})=n$.
\end{lemma}

\begin{proof}
By induction on $n$. The lemma trivially holds for $n=1$.
Let $n>1$. Then, using the induction hypothesis, we derive
\begin{align*}
\sum_{i=1}^n (H_n-H_{i-1}) &=H_n-H_{n-1}+\sum_{i=1}^{n-1} \left(\frac{1}{n}+H_{n-1}-H_{i-1}\right)\\
&=\frac{1}{n}+\frac{n-1}{n} + \sum_{i=1}^{n-1} (H_{n-1}-H_{i-1}) = n\:.
\end{align*}

\end{proof}

\SeleUpperBoundsAuRC*

\begin{proof}
The theorem trivially holds if $\sum_{i=1}^n \ell_i = 0$. If $\sum_{i=1}^n \ell_i > 0$, we apply Lemmas~\ref{lemma:monotonic}, \ref{lemma:bound} and~\ref{lemma:harm_expr} to derive
$$\frac{\AUC(s, \ST_n)}{\SELE(s, \ST_n)} \le \frac{\sum_{i=1}^n a_i}{\sum_{i=1}^n b_i}= \frac{\sum_{i=1}^n n\cdot (H_n-H_{i-1})}{\sum_{i=1}^n (n-i+1)}= \frac{n^2}{\frac{n}{2}(n+1)} = \frac{2n}{n+1} < 2\:.$$
\end{proof}

\subsection{Proof of Theorem~\ref{thm:scoringFunction}}


\begin{remark}
For the sake of simplicity, for predicates $\varphi_1(x,z), \ldots, \varphi_k(x,z)$ and a function $f:\SX\times\SX\to \Re$, we write
\[
\int\limits_{\SX}\int\limits_{\substack{\varphi_1(x,z)  \\ \vdots\\ \varphi_k(x,z)}}f(x,z)dz\, dx
\]
to represent
\[
\int\limits_{\SX}\int\limits_{\SX}f(x,z)\leftbb \varphi_1(x,z) \land \ldots \land \varphi_k(x,z) \rightbb dz\,dx \,.
\]
\end{remark}

\minSeleSolution*

\begin{proof}
    We first present four equalities to be used later. We assume that $s:\SX \to \Re$ is any measurable function. The validity of the equalities can be easily verified.
    \begin{align}
        &\int\limits_{\SX} r(x) p(x)\! \int\limits_{\substack{ r(z)> r(x) \\ s(z) < s(x) }} \!\! p(z) dz\, dx = \int\limits_{\SX} p(z) \! \int\limits_{\substack{ r(x)< r(z) \\ s(x) > s(z) }} \!\! r(x)p(x) dx\, dz = \int\limits_{\SX} p(x) \! \int\limits_{\substack{ r(z)< r(x) \\ s(z) > s(x) }} \!\!r(z)p(z) dz\,dx\,,\label{equ:aux-1}\\
        &\int\limits_{\SX} \int\limits_{\substack{ r(z)< r(x) \\ s(z) = s(x) }} r(x) p(x) p(z) dz\,dx = \frac{1}{2} \int\limits_{\SX} \int\limits_{\substack{ z \neq x \\ s(z) = s(x) }} \max\{r(x),r(z)\} p(x) p(z) dz\,dx \nonumber\\
        &\qquad\qquad\qquad\qquad\qquad\qquad\qquad - \frac{1}{2} \int\limits_{\SX} \int\limits_{\substack{ z \neq x \\ r(z) = r(x) \\ s(z) = s(x) }} \max\{r(x),r(z)\} p(x) p(z) dz\,dx \,, \label{equ:aux-2}\\
        &\int\limits_{\SX} \int\limits_{\substack{ r(z)= r(x) \\ s(z) < s(x) }} \!\!\!\!\!\! r(x) p(x) p(z) dz\,dx = \frac{1}{2} \int\limits_{\SX} \int\limits_{r(z)= r(x)} \!\!\!\!\!\! r(x) p(x) p(z) dz\,dx - \frac{1}{2} \int\limits_{\SX} \int\limits_{\substack{ r(z)= r(x) \\ s(z) = s(x) }} \!\!\!\!\!\! r(x) p(x) p(z) dz\,dx \,, \label{equ:aux-3}\\
        &\int\limits_{\SX} \int\limits_{\substack{ r(z)= r(x) \\ s(z) = s(x) }} r(x) p(x) p(z) dz\,dx = 2 \int\limits_{\SX} \int\limits_{\substack{ z > x \\ r(z)= r(x) \\ s(z) = s(x) }} r(x) p(x) p(z) dz\,dx + \int\limits_{\SX} \int\limits_{z=x} r(x) p(x) p(z) dz\,dx \,. \label{equ:aux-4}
    \end{align}
    Since $\argmin_{s:\SX \to \Re} E(s) = \argmin_{s:\SX \to \Re} \left( E(s)-E(r) \right)$, it suffices to analyze minimizers of $E(s)-E(r)$ instead of $E(s)$. Derive
    \begin{align*}
        E(s)-E(r)&=
        \int\limits_{\SX} \int\limits_{s(z)\ge s(x)} r(x)p(x)p(z) dz\,dx - \int\limits_{\SX} \int\limits_{{r(z)\ge r(x)}} r(x)p(x)p(z) dz\,dx \\
        &=\int\limits_{\SX} \int\limits_{\substack{ r(z) < r(x) \\ s(z)\ge s(x) }} r(x) p(x) p(z) dz\,dx - \int\limits_{\SX} \int\limits_{\substack{ r(z)\ge r(x) \\ s(z) < s(x) }} r(x) p(x) p(z) dz \,dx \\
        &=\int\limits_{\SX} \int\limits_{\substack{ r(z) < r(x) \\ s(z)> s(x) }} r(x)p(x)p(z) dz\,dx -  \int\limits_{\SX} \int\limits_{\substack{ r(z)> r(x) \\ s(z) < s(x) }} r(x) p(x) p(z) dz\, dx \\
        &+\int\limits_{\SX} \int\limits_{\substack{ r(z) < r(x) \\ s(z)= s(x) }} r(x)p(x)p(z) dz\,dx - \int\limits_{\SX} \int\limits_{\substack{ r(z)= r(x) \\ s(z) < s(x) }} r(x) p(x) p(z) dz\,dx \\
        &= F_1(s) + F_2(s)
    \end{align*}
    where
    \begin{align*}
        F_1(s)&=\int\limits_{\SX} \int\limits_{\substack{ r(z) < r(x) \\ s(z)> s(x) }} r(x)p(x)p(z) dz\,dx -  \int\limits_{\SX} \int\limits_{\substack{ r(z)> r(x) \\ s(z) < s(x) }} r(x) p(x) p(z) dz\, dx \\
        &\stackrel{(\ref{equ:aux-1})}= \int\limits_{\SX} \int\limits_{\substack{ r(z) < r(x) \\ s(z)> s(x) }} r(x)p(x)p(z) dz\,dx -  \int\limits_{\SX} \int\limits_{\substack{ r(z)< r(x) \\ s(z) > s(x) }} r(z) p(x) p(z) dz\, dx \\
        &=\int\limits_{\SX}\int\limits_{\substack{ r(z)< r(x) \\ s(z)> s(x)}} \left(r(x)-r(z)\right)p(x)p(z)dz\,dx
    \end{align*}
    and
    \begin{align*}
        F_2(s)&=\int\limits_{\SX} \int\limits_{\substack{ r(z) < r(x) \\ s(z)= s(x) }} r(x)p(x)p(z) dz\,dx - \int\limits_{\SX} \int\limits_{\substack{ r(z)= r(x) \\ s(z) < s(x) }} r(x) p(x) p(z) dz\,dx \\ &\stackrel{(\ref{equ:aux-2},\ref{equ:aux-3},\ref{equ:aux-4})}{=} \frac{1}{2} \int\limits_{\SX} \int\limits_{\substack{ z \neq x \\ s(z)= s(x) }} \max\{r(x),r(z)\}p(x)p(z) dz\,dx + \frac{1}{2} \int\limits_{\SX} \int\limits_{z=x} r(x) p(x) p(z) dz\,dx \\ &- \frac{1}{2} \int\limits_{\SX} \int\limits_{r(z)=r(x)} r(x) p(x) p(z) dz\,dx \,.
    \end{align*}
    Observe that
    \begin{equation*}
        \min_{s:\SX \to \Re} F_1(s) = 0,
    \end{equation*}
    \begin{equation*}
        \min_{s:\SX \to \Re} F_2(s) = \frac{1}{2} \int\limits_{\SX} \int\limits_{z=x} r(x) p(x) p(z) dz\,dx - \frac{1}{2} \int\limits_{\SX} \int\limits_{r(z)=r(x)} r(x) p(x) p(z) dz\,dx\,,  
    \end{equation*}
    and both minima are attained by a scoring function $s^*:\SX \to \Re$ if and only if conditions (\ref{equ:scoringFunction-cond1}) and (\ref{equ:scoringFunction-cond2}) hold for $s^*$. Also note that the conditions can be fulfilled, e.g. by any $s^*$ such that
    \begin{eqnarray*}
        \left( \forall x,z\in \SX\right) \left( x\neq z \Rightarrow \,s^*(x)\neq s^*(z) \,\, \land \,\, r(x) < r(z) \Rightarrow s^*(x) < s^*(z) \right).
    \end{eqnarray*}
\end{proof}

\subsection{Proof of Theorem~\ref{thm:seleProxySolution}}

The expectation of $\SELEcvx$ reads
\[
  E_{\rm proxy}(s)\!  =\! \frac{n^2-n}{n^2}\!\! \int_{\SX} \!p(x)r(x)\! \left( \int_{\SX} \!p(z) \log\big( 1 \!+\! \exp(s(z)\!-\!s(x)) \big ) dz\! \right) \!dx  + \frac{\log(2)}{n}\!\! \int_{\SX} p(x) r(x) dx \:.
\]

\seleProxySolution*

\begin{proof}
  For every $a \in \mathcal{X}$,  $E_{\rm proxy}(s)$ can be seen as a function of one variable $s(a)$, where the others $s(b)$, $b\in\SX\setminus\{a\}$ are fixed. Hence if $s$ is a minimizer of $E_{\rm proxy}(s)$, then for every $s(a), a\in\mathcal{X}$, the partial derivative w.r.t. $s(a)$ must be zero, i.e.,
  \begin{align*}
    0 &= \frac{\partial}{\partial s(a)}   \int_{\SX} p(x)\, r(x) \bigg ( \int_{\SX} p(z) \log\left(1+\exp( s(z)-s(x))\right) dz \bigg ) dx  \\
    & =p(a)r(a)\int_{\SX} p(z)\frac{-\exp(s(z)-s(a))}{1+\exp(s(z)-s(a))} dz 
     +p(a)\int_{\SX} p(x)r(x)\frac{\exp(s(a)-s(x))}{1+\exp(s(a)-s(x))} dx\\
    &= p(a)\int_{\SX} \frac{r(a)\,p(x)+r(x)\,p(x)}{1+\exp(s(a)-s(x))} dx - p(a)\,\int_{\SX} p(x)\,r(x)\, dx\\
    & = f(r(a),s(a)) - C\:.
\end{align*}
%
It shows that $f(r(a),s(a))=C$ for any $a \in \mathcal{X}$ in order to guarantee that $s$ is a minimizer of $E_{\rm proxy}(s)$. We prove by contradiction that the condition $r(a) < r(b) \Rightarrow s(a) < s(b)$ is satisfied up to a set $\{(a,b)| (a,b) \in \mathcal{X}^2\}$ of zero measure. Assume $s$ is optimal and the condition is violated, i.e. $r(a) < r(b) \wedge s(a) \geq s(b)$ holds for a pair $(a,b)\in\SX^2$. Since $s$ is optimal then $f(r(a),s(a))=C$. Since $r(b)>r(a)$ and $s(b) \leq s(a)$ then $f(r(b),s(b))>f(r(a),s(b))$ 
because the function $f(u,v)$ is strictly increasing in $u$ and strictly decreasing in $v$. Combined it implies that $f(r(b),s(b))>0$ which leads to a contradiction because an optimal $s$ requires $f(r(b),s(b))=C$.

\end{proof}


\bibliography{main}

\begin{thebibliography}{37}
\providecommand{\natexlab}[1]{#1}
\providecommand{\url}[1]{\texttt{#1}}
\expandafter\ifx\csname urlstyle\endcsname\relax
  \providecommand{\doi}[1]{doi: #1}\else
  \providecommand{\doi}{doi: \begingroup \urlstyle{rm}\Url}\fi

\bibitem[Bartlett and Wegkamp(2008)]{Bartlett-RejectHinge-JMLR2008}
P.~L. Bartlett and M.~H. Wegkamp.
\newblock Classification with a reject option using a hinge loss.
\newblock \emph{Journal of Machine Learning Research}, 9\penalty0
  (59):\penalty0 1823--1840, 2008.

\bibitem[Chang and C.J.Lin(2011)]{Chang-libSVM-2011}
C.C. Chang and C.J.Lin.
\newblock {LIBSVM}: A library for support vector machines.
\newblock \emph{ACM Transactions on Intelligent Systems and Technology},
  2:\penalty0 27:1--27:27, 2011.
\newblock URL \url{http://www.csie.ntu.edu.tw/~cjlin/libsvm}.

\bibitem[Chow(1970)]{Chow-RejectOpt-TIT1970}
C.~Chow.
\newblock On optimum recognition error and reject tradeoff.
\newblock \emph{IEEE Transactions on Information Theory}, 16\penalty0
  (1):\penalty0 41–46, 1970.

\bibitem[Chu and Keerthi(2005)]{Chu-SVOR-05}
W.~Chu and S.~S. Keerthi.
\newblock New approaches to support vector ordinal regression.
\newblock In \emph{Proceedings of the International Conference on Machine
  Learning}, pages 145--152, 2005.

\bibitem[Corbiere et~al.(2019)Corbiere, Thome, Bar-Hen, Cord, and
  Perez]{Corbiere-Failure-NeurIPS2019}
C.~Corbiere, N.~Thome, A.~Bar-Hen, M.~Cord, and P.~Perez.
\newblock Addressing failure prediction by learning model confidence.
\newblock In \emph{Advances in Neural Information Processing Systems},
  volume~32, pages 2902--2913, 2019.

\bibitem[Cortes et~al.(2016)Cortes, DeSalvo, and
  Mohri]{Cortes-BoostWitAbst-NIPS2016}
C.~Cortes, G.~DeSalvo, and M.~Mohri.
\newblock Boosting with abstention.
\newblock In \emph{Advances in Neural Information Processing Systems},
  volume~29, pages 1660--1668, 2016.

\bibitem[Dalal and Triggs(2005)]{Dalal-HOG-2005}
N.~Dalal and B.~Triggs.
\newblock Histograms of oriented gradients for human detection.
\newblock In \emph{Proceedings of Conference on Computer Vision and Patter
  Recognition}, volume~1, pages 886--893, 2005.

\bibitem[Dem{\v{s}}ar(2006)]{Demsar-JMLR06}
J.~Dem{\v{s}}ar.
\newblock Statistical comparisons of classifiers over multiple data sets.
\newblock \emph{Journal of Machine Learning Research}, 7\penalty0 (1):\penalty0
  1--30, 2006.

\bibitem[Dua and Taniskidou(2017)]{Dua-UCI-2017}
D.~Dua and E.~Karra Taniskidou.
\newblock {UCI} machine learning repository, 2017.
\newblock URL \url{http://archive.ics.uci.edu/ml}.

\bibitem[El-Yaniv and Wiener(2010)]{ElYaniv-SelectClass-JMLR10}
R.~El-Yaniv and Y.~Wiener.
\newblock On the foundations of noise-free selective classification.
\newblock \emph{Journal of Machine Learning Research}, 11\penalty0
  (53):\penalty0 1605--1641, 2010.

\bibitem[Fischer et~al.(2016)Fischer, Hammer, and
  Wersing]{Fisher-LocalRejectClassif-NC2016}
L.~Fischer, B.~Hammer, and H.~Wersing.
\newblock Optimal local rejection for classifiers.
\newblock \emph{Neurocomputing}, 214:\penalty0 445--457, 2016.

\bibitem[Fisher et~al.(2015)Fisher, Hammer, and
  Wersing]{Fisher-EfficientReject-Neuro2015}
L.~Fisher, B.~Hammer, and H.~Wersing.
\newblock Efficient rejection strategies for prototype-based classification.
\newblock \emph{Neurocomputing}, 169:\penalty0 334 -- 342, 2015.

\bibitem[Franc and Prusa(2019)]{Franc-SELE-ICML2019}
V.~Franc and D.~Prusa.
\newblock On discriminative learning of prediction uncertainty.
\newblock In \emph{Proceedings of the 36th International Conference on Machine
  Learning}, volume~97, pages 1963--1971, 2019.

\bibitem[Fumera and Roli(2002)]{Fumera-SvmRejectOpt-2002}
G.~Fumera and F.~Roli.
\newblock Support vector machines with embedded reject option.
\newblock In \emph{Pattern Recognition with Support Vector Machines, Lecture
  Notes in Computer Science}, volume 2388. Springer, 2002.

\bibitem[Fumera et~al.(2000)Fumera, Roli, and
  Giacinto]{Fumera-MultiReject-2000}
G.~Fumera, F.~Roli, and G.~Giacinto.
\newblock Multiple reject thresholds for improving classification reliability.
\newblock In \emph{Advances in Pattern Recognition}, pages 863--871, 2000.

\bibitem[Geifman and El-Yaniv(2017)]{Geifman-SelectClass-NIPS2017}
Y.~Geifman and R.~El-Yaniv.
\newblock Selective classification for deep neural networks.
\newblock In \emph{Advances in Neural Information Processing Systems 30}, pages
  4878--4887, 2017.

\bibitem[Grandvalet et~al.(2008)Grandvalet, Rakotomamonjy, Keshet, and
  Canu]{Grandvalet-SVMwithRejectOpt-NIPS2008}
Y.~Grandvalet, A.~Rakotomamonjy, J.~Keshet, and S.~Canu.
\newblock Support vector machines with a reject option.
\newblock In \emph{Advances in Neural Information Processing Systems},
  volume~21, pages 537--544, 2008.

\bibitem[Hanczar and Dougherty(2008)]{Hanczar-ClsReject-BIO08}
B.~Hanczar and E.~R. Dougherty.
\newblock Classification with reject option in gene expression data.
\newblock \emph{Bioinformatics}, 24:\penalty0 1889--1895, 2008.

\bibitem[Hastie et~al.(2009)Hastie, Tibshirani, and
  Friedman]{Hastie-Elements-09}
T.~Hastie, R.~Tibshirani, and J.~Friedman.
\newblock \emph{The elements of statistical learning: data mining, inference
  and prediction}.
\newblock Springer, 2009.

\bibitem[Herbei and Wegkamp(2006)]{Herbei-ClassRejectOpt-CJS2006}
R.~Herbei and M.H. Wegkamp.
\newblock Classification with reject option.
\newblock \emph{The Canadian Journal of Statistics / La Revue Canadienne de
  Statistique}, 34\penalty0 (4):\penalty0 709--721, 2006.

\bibitem[Jiang et~al.(2018)Jiang, Kim, Guan, and
  Gupta]{Jiang-TrustOrNot-NIPS2018}
H.~Jiang, B.~Kim, M.~Y. Guan, and M.~Gupta.
\newblock To trust or not to trust a classifier.
\newblock In \emph{Proceedings of the 32nd International Conference on Neural
  Information Processing Systems}, page 5546–5557, 2018.

\bibitem[Kazemi and Sullivan(2014)]{Kazemi-FaceAlign-CVPR14}
V.~Kazemi and J.~Sullivan.
\newblock One millisecond face alignment with an ensemble of regression trees.
\newblock In \emph{IEEE Conference on Computer Vision and Pattern Recognition},
  pages 1867--1874, 2014.

\bibitem[King(2009)]{dlib09}
D.~E. King.
\newblock Dlib-ml: A machine learning toolkit.
\newblock \emph{Journal of Machine Learning Research}, 10:\penalty0 1755--1758,
  2009.

\bibitem[Kummert et~al.(2016)Kummert, Paassen, Jensen, G{\" o}pfert, and
  Hammer]{Kummert-LocalReject-ICANN2016}
J.~Kummert, B.~Paassen, J.~Jensen, C.~G{\" o}pfert, and B.~Hammer.
\newblock Local reject option for deterministic multi-class {SVM}.
\newblock In \emph{Artificial Neural Networks and Machine Learning – ICANN,
  Lecture Notes in Computer Science}, volume 9887. Springer, 2016.

\bibitem[Lakshminarayanan et~al.(2017)Lakshminarayanan, Pritzel, and
  Blundell]{Laksh-UncertDeep-NIPS2017}
B.~Lakshminarayanan, A.~Pritzel, and C.~Blundell.
\newblock Simple and scalable predictive uncertainty estimation using deep
  ensembles.
\newblock In \emph{Advances in Neural Information Processing Systems},
  volume~30, pages 6402--6413, 2017.

\bibitem[LeCun et~al.(1990)LeCun, Boser, Denker, Henderson, Howard, Hubbard,
  and Jakel]{Lecun-OcrNN-NIPS90}
Y.~LeCun, B.~Boser, J.S. Denker, D.~Henderson, R.E. Howard, W.~Hubbard, and
  L.D. Jakel.
\newblock Handwritten digit recognition with a back-propagation networks.
\newblock In \emph{Advances in Neural Information Processing Systems},
  volume~2, pages 396--404, 1990.

\bibitem[Lei(2014)]{Lei-ClassConf-Biom14}
J.~Lei.
\newblock Classification with confidence.
\newblock \emph{Bimetrika}, 101:\penalty0 755--769, 2014.

\bibitem[Pietraszek(2005)]{Pietraszek-AbstainROC-ICML2005}
T.~Pietraszek.
\newblock Optimizing abstaining classifiers using {ROC} analysis.
\newblock In \emph{Proceedings of the 22nd International Conference on Machine
  Learning}, page 665–672, 2005.

\bibitem[Sagonas et~al.(2016)Sagonas, Antonakos, Tzimiropoulos, Zafeiriou, and
  Pantic]{Sagonas-300W-ICCVW2013}
C.~Sagonas, E.~Antonakos, G.~Tzimiropoulos, S.~Zafeiriou, and M.~Pantic.
\newblock 300 faces in-the-wild challenge: database and results.
\newblock \emph{Image and Vision Computing}, 47:\penalty0 3 -- 18, 2016.

\bibitem[Schlesinger and Hlav{\'a}{\v c}(2002)]{SchlesingerHlavac-10Lectures02}
M.I. Schlesinger and V.~Hlav{\'a}{\v c}.
\newblock \emph{Ten lectures on statistical and structural pattern
  recognition}.
\newblock Kluwer Academic Publishers, 2002.

\bibitem[Stein and Shakarchi(2009)]{Stein-RealAnalysis-2009}
E.M. Stein and R.~Shakarchi.
\newblock \emph{Real Analysis: Measure Theory, Integration, and Hilbert
  Spaces}.
\newblock Princeton University Press, 2009.

\bibitem[Teo et~al.(2010)Teo, Vishwanthan, Smola, and Le]{Teo-BMRM-JMLR0}
C.~H. Teo, S.V.N. Vishwanthan, A.~J. Smola, and Q.~V. Le.
\newblock Bundle methods for regularized risk minimization.
\newblock \emph{Journal of Machine Learning Research}, 11\penalty0
  (10):\penalty0 311--365, 2010.

\bibitem[Tortorella(2000)]{Torella-OptRejectRule-2000}
F.~Tortorella.
\newblock An optimal reject rule for binary classifiers.
\newblock In \emph{Advances in Pattern Recognition, Lecture Notes in Computer
  Science}, volume 1876. Springer, 2000.

\bibitem[Vapnik(1998)]{Vapnik-StatLearning-98}
V.N. Vapnik.
\newblock \emph{Statistical Learning Theory}.
\newblock John Wiley \& Sons, Inc., 1998.

\bibitem[Villman et~al.(2016)Villman, Kaden, Bohnsack, Villman, Drogies,
  Saralajew, and Hammer]{Villmann-RejectProto-AISC2016}
T.~Villman, M.~Kaden, A.~Bohnsack, J.~M. Villman, T.~Drogies, S.~Saralajew, and
  B.~Hammer.
\newblock Self-adjusting reject options in prototype based classification.
\newblock In \emph{Advances in Intelligent Systems and Computing}, volume 428.
  Springer, 2016.

\bibitem[Yuan and Wegkamp(2010)]{Yuan-ClassifRejectOpt-JMLR2010}
M.~Yuan and M.~Wegkamp.
\newblock Classification methods with reject option based on convex risk
  minimization.
\newblock \emph{Journal of Machine Learning Research}, 11\penalty0
  (5):\penalty0 111--130, 2010.

\bibitem[Zaragoza and d'Alche Buc(1998)]{Zaragoza-ConfInNN-IPMU98}
H.~Zaragoza and F.~d'Alche Buc.
\newblock Confidence measures for neural network classifiers.
\newblock In \emph{7th Conference on Information Processing and Management of
  Uncertainty in Knowledge-Based Systems}, 1998.

\end{thebibliography}

\end{document}